%% file: main.tex
\documentclass{article}

\usepackage{PRIMEarxiv}

\usepackage[utf8]{inputenc} 
\usepackage[T1]{fontenc}    
\usepackage{url}            
\usepackage{booktabs,longtable,array,ragged2e,xcolor,colortbl}
\usepackage{amsfonts}       
\usepackage{nicefrac}       
\usepackage{microtype}      
\usepackage{lipsum}
\usepackage{amsmath}
\usepackage{amsthm}
\usepackage{amssymb}
\usepackage{empheq}
\usepackage{graphicx}
\usepackage{wrapfig}
\usepackage{subcaption} 
\usepackage{algorithm}
\usepackage{algpseudocode}

\usepackage{siunitx }
\usepackage{bigints}
\usepackage{bm}
\usepackage{diagbox}
\usepackage{makecell}
\usepackage{multirow}
\usepackage{caption}
\usepackage{subcaption}
\usepackage{graphicx}

\include{pythonlisting}
\usepackage{xcolor}
\usepackage{mathtools}
\usepackage{enumitem}
\usepackage[most]{tcolorbox}
\usepackage{hyperref}
\hypersetup{
    colorlinks=true,
    citecolor=black,
    linkcolor=black,
    filecolor=black,
    urlcolor=black,
}

\usepackage{pifont}

\newtheorem{theorem}{Theorem}[section]

\newtheorem{remark}[theorem]{Remark}

\numberwithin{equation}{section}

\title{When PINNs Go Wrong: Pseudo-Time Stepping Against Spurious Solutions
}

\author{
  Sifan Wang$^{1}$ \quad
  Shawn Koohy$^{2}$ \quad
  Yiping Lu$^{3}$ \quad
  Paris Perdikaris$^{2}$ \\[0.5em]
  $^{1}$Institute for Foundations of Data Science, Yale University \\
  [0.3em ]
  $^{2}$Department of Mechanical Engineering and Applied Mechanics, University of Pennsylvania \\[0.3em]
    $^{3}$Industrial Engineering \& Management Science, Northwestern University \\
  \texttt{sifan.wang@yale.edu} \quad
\texttt{yiping.lu@northwestern.edu} \quad
  \texttt{\{skoohy, pgp\}@seas.upenn.edu} 
}

\begin{document}
\maketitle

\begin{abstract}
Physics-informed neural networks (PINNs) provide a promising machine learning framework for solving partial differential equations, but their training often breaks down on challenging problems, sometimes converging to physically incorrect solutions despite achieving small residual losses. This failure, we argue, is not merely an optimization difficulty. Rather, it reflects a fundamental weakness of the empirical PDE residual loss, which can admit trivial or spurious solutions during training. From this perspective, we revisit pseudo-time stepping, a technique that has recently shown strong empirical success in PINNs. We show that its main benefit is not simply to ease optimization; instead, when combined with collocation-point resampling, it helps reveal and avoid spurious solutions. At the same time, we find that the effectiveness of pseudo-time stepping depends critically on the choice of step size, which cannot be tuned reliably from the training loss alone. To overcome this limitation, we propose an adaptive pseudo-time stepping strategy that selects the step size from a finite-difference surrogate of the local residual Jacobian, yielding the largest step permitted by local stability without per-problem tuning. Across a diverse set of PDE benchmarks, the proposed method consistently improves both accuracy and robustness. Together, these findings provide a clearer understanding of why PINNs fail and suggest a practical pathway toward more reliable physics-informed learning.  All code and data accompanying this manuscript are available at \url{https://github.com/sifanexisted/jaxpi2}.
\end{abstract}


\section{Introduction}


Physics-informed neural networks (PINNs) \cite{raissi2019physics,karniadakis2021physics} have emerged as a flexible framework for solving partial differential equations by embedding the governing equations and auxiliary conditions directly into the training objective. Their mesh-free formulation, seamless use of automatic differentiation, and ability to combine physical constraints with observational data make them especially attractive for scientific computing. As a result, PINNs have been used not only for forward simulation of complex dynamical systems \cite{wang2025simulating}, but also for inverse problems such as parameter and coefficient identification \cite{haghighat2021physics,chen2020physics}, data assimilation from sparse and noisy observations \cite{he2020physics}, uncertainty-aware modeling \cite{yang2021b,linka2022bayesian}, and data-driven discovery of governing equations from partial measurements \cite{raissi2018deep,raissi2020hidden,wang2025deep}. Representative applications span fluid mechanics \cite{rao2020physics,cai2021physics}, solid mechanics \cite{zhang2022analyses,diao2023solving}, heat transfer \cite{cai2021physics,xu2023physics,li2025physics}, reaction--diffusion systems \cite{colace2024physics,serebrennikova2024physics}, wave propagation \cite{wang2023physics}, subsurface flow \cite{tartakovsky2020physics,fraces2021physics,haruzi2023modeling}, and biomedical modeling \cite{kissas2020machine,sel2023physics,mukhmetov2023physics}.

Despite their promise, PINNs often remain difficult to train reliably on challenging problems, especially when the target solution exhibits sharp transitions, long-time evolution, multiscale structure, or strong nonlinear and stiff dynamics. To address these limitations, a broad range of improvements has been proposed  recently, including more expressive   architectures \cite{jagtap2020extended,wang2024piratenets,moseley2021finite,zhao2023pinnsformer}, loss weighting strategies \cite{mcclenny2020self,wang2022and,bischof2025multi},  optimization algorithms \cite{rathore2024challenges, muller2023achieving, wang2025gradient}, and more effective collocation-point sampling schemes \cite{wu2023comprehensive,gao2023failure,mao2023physics,daw2022mitigating}. Beyond improving empirical performance, these developments have also led to a clearer picture of why standard PINNs fail in the first place. In particular, recent studies have identified several recurring training pathologies, including spectral bias toward low-frequency solution components \cite{wang2021eigenvector,mustajab2024physics},  gradient conflicts across different loss terms \cite{wang2021understanding,anagnostopoulos2024residual, liu2024config}, optimization difficulties in stiff or multiscale regimes \cite{rathore2024challenges,krishnapriyan2021characterizing}, and failures caused by violating the intrinsic causal structure of time-dependent problems \cite{penwarden2023unified,wang2024respecting,roy2024exact}.

Yet these advances do not fully resolve a more fundamental issue: training can appear stable from the perspective of optimization while the resulting prediction remains physically incorrect. One particularly important manifestation of this failure is the tendency of PINNs to converge to trivial or spurious solutions. Among the many strategies proposed to improve PINN training, pseudo-time stepping has shown especially strong empirical benefits in mitigating this issue. Adapted from pseudo-transient continuation in numerical analysis, it replaces the original PDE residual constraints with a relaxed formulation that can guide training more robustly \cite{kelley1998convergence,cao2023tsonn,zhang2025pseudo}. Existing explanations often attribute its success to improved conditioning or to a better-behaved loss landscape, but here we show empirically that this explanation is insufficient.

In this work, we revisit pseudo-time stepping from the perspective of spurious-solution formation in PINNs. Our goal is to better understand why PINNs can fail even when the empirical residual loss is small, why pseudo-time stepping is often effective in such cases, and how this understanding can be translated into a more practical and robust training strategy. Our main contributions are as follows:
\begin{itemize}[leftmargin=*,itemsep=0.3em,topsep=0.3em]

    \item We identify a fundamental failure mode of the empirical PINN loss -- convergence to spurious solutions that satisfy all training constraints at any finite collocation set -- and prove its existence in a representative homogeneous setting.

    \item We demonstrate, contrary to the prevailing optimization-centric view, that the empirical effectiveness of pseudo-time stepping cannot be explained by improved conditioning alone. 

    \item We show that fixed pseudo-time stepping is highly sensitive to the choice of step size, and that this sensitivity cannot be diagnosed from the training loss -- making a fixed step size unreliable in practical settings where reference solutions are unavailable.
        
    \item We derive an adaptive pseudo-time stepping strategy directly from the theory: a Barzilai–Borwein-style \cite{barzilai1988two} finite-difference surrogate for the inverse local Jacobian magnitude.
    
    \item We validate the proposed method against strong baselines across 10 challenging benchmarks spanning shock formation, chaotic dynamics, reaction-diffusion, and high-Reynolds-number flows, with consistent gains over both the baseline and best-tuned fixed pseudo-time stepping.

\end{itemize}

The rest of the paper is organized as follows. In Section~\ref{sec:method}, we first review the standard PINN formulation in Section~\ref{sec: pinns} and then analyze the spurious-solution failure mode induced by the empirical residual loss in Section~\ref{sec: failure_mode}. We next revisit pseudo-time stepping in PINNs in Section~\ref{sec:pseduo_time}, where we highlight a key paradox: although the method is effective when collocation points are resampled, it can fail on a fixed collocation set even while attaining an apparently smaller training loss. This observation challenges the conventional view that pseudo-time stepping helps mainly by improving conditioning. We then resolve this paradox by developing a new mechanistic interpretation showing that pseudo-time stepping helps expose hidden residual defects in spurious solutions. In Section~\ref{sec:adaptive_pseduo_time}, we further show that conventional pseudo-time stepping is highly sensitive to the choice of step size, which motivates the adaptive pseudo-time stepping strategy introduced there. Section~\ref{sec:results} presents empirical results on a diverse set of PDE benchmarks, and Section~\ref{sec:discussion} concludes with a discussion of the implications and future directions.

\section{Method}
\label{sec:method}

\subsection{Physics-informed neural networks (PINNs)}
\label{sec: pinns}

In this section, we briefly review the standard formulation of physics-informed neural networks (PINNs), following the original framework of Raissi {\em et al.}~\cite{raissi2019physics}. 
Without loss of generality, we consider an abstract parabolic PDE of the form
\begin{equation}
\label{eq:problem}
\left\{
\begin{aligned}
\mathbf{u}_t + \mathcal{D}[\mathbf{u}] &= \mathbf{f}, 
&& (t,\mathbf{x}) \in [0,T]\times\Omega,\\
\mathcal{B}[\mathbf{u}] &= 0, 
&& (t,\mathbf{x}) \in [0,T]\times\partial\Omega, \\
\mathbf{u}(0,\mathbf{x}) &= \mathbf{g}(\mathbf{x}), 
&& \mathbf{x}\in\Omega,
\end{aligned}
\right.
\end{equation}
where \(\Omega\subset\mathbb{R}^d\) is a bounded domain with sufficiently regular boundary \(\partial\Omega\), \(\mathcal{D}[\cdot]\) is a linear or nonlinear differential operator acting on the spatial variables \(\mathbf{x}\), and \(\mathbf{u}(t,\mathbf{x})\) denotes the unknown solution. 
Here, \(\mathbf{f}(t,\mathbf{x})\) on \([0,T]\times\Omega\) and \(\mathbf{g}(\mathbf{x})\) on \(\Omega\) are given functions with sufficient regularity such that the problem is well posed and admits a smooth solution. The operator \(\mathcal{B}[\cdot]\) denotes an abstract boundary operator, which may represent Dirichlet, Neumann, Robin, periodic, or other boundary conditions.

We approximate the solution \(\mathbf{u}(t,\mathbf{x})\) by a neural network \(\mathbf{u}_{\theta}(t,\mathbf{x})\), where \(\theta\) denotes the trainable parameters. 
With sufficiently smooth activation functions, \(\mathbf{u}_{\theta}\) defines a differentiable approximation that can be evaluated at any query point \((t,\mathbf{x})\). 
Moreover, the required derivatives with respect to the input variables and the parameters \(\theta\) can be computed efficiently using automatic differentiation~\cite{griewank2008evaluating}. 
This allows us to define the PDE residual
\begin{align}
\label{eq:pde_residual}
\mathcal{R}_{\mathrm{int}}[\mathbf{u}_{\theta}](t,\mathbf{x})
:=
\partial_t \mathbf{u}_{\theta}(t,\mathbf{x})
+\mathcal{D}[\mathbf{u}_{\theta}](t,\mathbf{x})
-\mathbf{f}(t,\mathbf{x}),
\qquad
(t,\mathbf{x})\in [0,T]\times\Omega,
\end{align}
the boundary residual
\begin{align}
\label{eq:bcc}
\mathcal{R}_{\mathrm{bc}}[\mathbf{u}_{\theta}](t,\mathbf{x})
:=
\mathcal{B}[\mathbf{u}_{\theta}](t,\mathbf{x}),
\qquad
(t,\mathbf{x})\in [0,T]\times\partial\Omega,
\end{align}
and the initial residual
\begin{align}
\label{eq:icc}
\mathcal{R}_{\mathrm{ic}}[\mathbf{u}_{\theta}](\mathbf{x})
:=
\mathbf{u}_{\theta}(0,\mathbf{x})-\mathbf{g}(\mathbf{x}),
\qquad
\mathbf{x}\in\Omega.
\end{align}

Let
\[
X_{\mathrm{int}}=\{(t^i_{\mathrm{int}},\mathbf{x}^i_{\mathrm{int}})\}_{i=1}^{N_{\mathrm{int}}}
\subset [0,T]\times\Omega,
\quad
X_{\mathrm{bc}}=\{(t^i_{\mathrm{bc}},\mathbf{x}^i_{\mathrm{bc}})\}_{i=1}^{N_{\mathrm{bc}}}
\subset [0,T]\times\partial\Omega, \quad
X_{\mathrm{ic}}=\{\mathbf{x}^i_{\mathrm{ic}}\}_{i=1}^{N_{\mathrm{ic}}}
\subset \Omega,
\]
denote the interior, boundary, and initial collocation sets, respectively. 
The corresponding empirical loss terms are defined by
\begin{align}
\mathcal{L}_{\mathrm{int}}(\theta;X_{\mathrm{int}})
&:=
\frac{1}{N_{\mathrm{int}}}
\sum_{i=1}^{N_{\mathrm{int}}}
\left|
\mathcal{R}_{\mathrm{int}}[\mathbf{u}_{\theta}]
(t^i_{\mathrm{int}},\mathbf{x}^i_{\mathrm{int}})
\right|^2,\\
\mathcal{L}_{\mathrm{bc}}(\theta;X_{\mathrm{bc}})
&:=
\frac{1}{N_{\mathrm{bc}}}
\sum_{i=1}^{N_{\mathrm{bc}}}
\left|
\mathcal{R}_{\mathrm{bc}}[\mathbf{u}_{\theta}]
(t^i_{\mathrm{bc}},\mathbf{x}^i_{\mathrm{bc}})
\right|^2,\\
\mathcal{L}_{\mathrm{ic}}(\theta;X_{\mathrm{ic}})
&:=
\frac{1}{N_{\mathrm{ic}}}
\sum_{i=1}^{N_{\mathrm{ic}}}
\left|
\mathcal{R}_{\mathrm{ic}}[\mathbf{u}_{\theta}]
(\mathbf{x}^i_{\mathrm{ic}})
\right|^2.
\end{align}
The standard PINN objective is then
\begin{align}
\label{eq:pinn_loss}
\mathcal{L}(\theta;X_{\mathrm{int}},X_{\mathrm{bc}},X_{\mathrm{ic}})
:=
\lambda_{\mathrm{int}} \mathcal{L}_{\mathrm{int}}(\theta;X_{\mathrm{int}})
+\lambda_{\mathrm{bc}}\,\mathcal{L}_{\mathrm{bc}}(\theta;X_{\mathrm{bc}})
+\lambda_{\mathrm{ic}}\,\mathcal{L}_{\mathrm{ic}}(\theta;X_{\mathrm{ic}}),
\end{align}
where \(\lambda_{\mathrm{int}} , \lambda_{\mathrm{bc}}>0\) and \(\lambda_{\mathrm{ic}}>0\) are weighting hyperparameters.

This loss encourages the neural-network approximation \(\mathbf{u}_{\theta}\) to satisfy the governing PDE together with the initial and boundary conditions in~\eqref{eq:problem}. 
The collocation sets \(X_{\mathrm{int}}, X_{\mathrm{bc}}, X_{\mathrm{ic}}\) may either be fixed throughout training or resampled at each iteration of a gradient-based optimizer. 
Although we focus on parabolic equations for exposition, the same formulation extends naturally to more general elliptic and hyperbolic PDEs, whether linear or nonlinear.

\subsection{A Failure mode of PINNs: Tendency to learn trivial solutions}
\label{sec: failure_mode}

Recent advances in physics-informed neural networks (PINNs) have increasingly focused on understanding and mitigating their training pathologies, such as spectral bias \cite{wang2021eigenvector}, gradient conflicts \cite{wang2021understanding, wang2022and, wang2025gradient,liu2024config}, stiffness of training dynamics \cite{rathore2024challenges}, causality violation \cite{wang2024respecting,es2024modifications}  and optimization failures \cite{krishnapriyan2021characterizing,gao2023failure,xu2025fp64}.

\begin{wrapfigure}{r}{0.5\textwidth}
    \vspace{-2mm}
    \centering
    \begin{subfigure}{\linewidth}
        \centering
        \includegraphics[width=\linewidth]{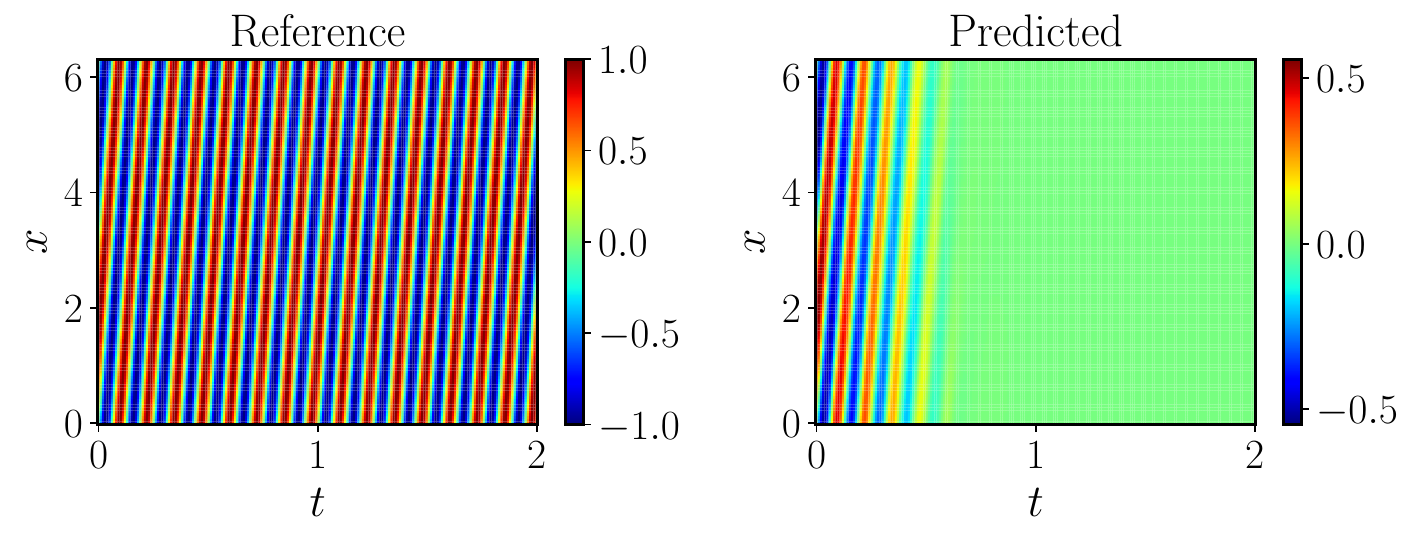}
    \end{subfigure}


    \begin{subfigure}{\linewidth}
        \centering
        \includegraphics[width=\linewidth]{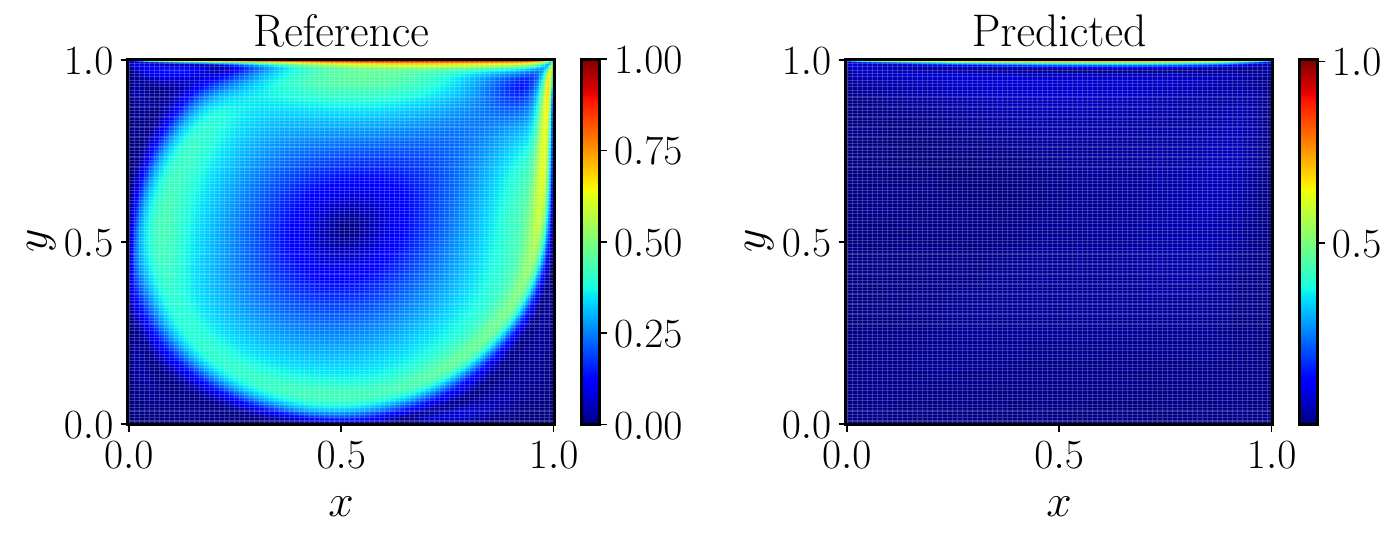}
    \end{subfigure}
  \caption{Failure predictions of PINN models for linear advection (top) and lid-driven cavity flow (bottom).}
    \label{fig:failure_cases}
    \vspace{-12mm}
\end{wrapfigure}
Among these pathologies, an especially important yet still insufficiently understood failure mode is the tendency of PINNs to converge to trivial or spurious solutions. Here, we aim to systematically study this long-observed phenomenon, in which training is attracted to trivial solutions rather than the desired physical ones.
To demonstrate this phenomenon, we consider two representative benchmark problems that are widely used in the PINN literature, covering both time-dependent and boundary-value settings.

\paragraph{Linear advection.}
We consider the one-dimensional linear advection equation
\begin{align}
    u_t + c\, u_x &= 0, \quad (t, x) \in [0, 2] \times [0, 2\pi], \\
    u(0, x)       &= \sin(x), \quad x \in [0, 2\pi],
\end{align}
with periodic boundary conditions, where $u(t,x)$ is the scalar solution
and $c = 50$ is the prescribed advection speed. This benchmark is a
representative time-dependent transport problem commonly used to assess
the ability of PINNs to capture propagating solution structures.

\paragraph{Lid-driven cavity flow.}
As a representative boundary-value problem, we consider the two-dimensional steady incompressible Navier--Stokes equations for lid-driven cavity flow:
\begin{align}
    (\mathbf{u} \cdot \nabla)\,\mathbf{u} + \nabla p - \frac{1}{\text{Re}}\Delta \mathbf{u} &= \mathbf{0}, \quad \mathbf{x} \in (0,1)^2, \\
    \nabla \cdot \mathbf{u} &= 0, \quad \mathbf{x} \in (0,1)^2,
\end{align}
where $\mathbf{u}=(u,v)$ is the velocity field, $p$ is the pressure, and $\text{Re} = 5{,}000$ is the Reynolds number. The boundary conditions are given by no-slip walls on the left, right, and bottom boundaries $\mathbf{u} = (0,0)$,
and a moving lid on the top boundary $\mathbf{u} = (U_{\rm lid},0).$
This problem is a standard benchmark for PINNs in fluid mechanics and provides a prototypical example of a nonlinear boundary-value problem.

The loss formulation for the advection problem follows directly from the problem setup in Section \ref{sec: pinns} and is therefore omitted for brevity. Here, we briefly outline the loss formulation for the lid-driven cavity flow.
We approximate the solution by a neural network  $\mathbf{u}_\theta = (u_\theta(\mathbf{x}),\, v_\theta(\mathbf{x}),\, p_\theta(\mathbf{x}))$, where $\theta$ denotes the trainable parameters. The corresponding PDE residuals are defined as
\begin{align}
    \mathcal{R}_{u}[\mathbf{u}_\theta](\mathbf{x})
    &:=
    u_\theta \, \partial_x u_\theta
    + v_\theta \, \partial_y u_\theta
    + \partial_x p_\theta
    - \frac{1}{\mathrm{Re}}
    \left(\partial_{xx}u_\theta + \partial_{yy}u_\theta\right), \\
    \mathcal{R}_{v}[\mathbf{u}_\theta](\mathbf{x})
    &:=
    u_\theta \, \partial_x v_\theta
    + v_\theta \, \partial_y v_\theta
    + \partial_y p_\theta
    - \frac{1}{\mathrm{Re}}
    \left(\partial_{xx}v_\theta + \partial_{yy}v_\theta\right), \\
    \mathcal{R}_{c}[\mathbf{u}_\theta](\mathbf{x})
    &:=
    \partial_x u_\theta + \partial_y v_\theta .
\end{align}
Let $X_{\mathrm{int}}=\{\mathbf{x}^i_{\mathrm{int}}\}_{i=1}^{N_{\mathrm{int}}}\subset (0,1)^2$
denote the set of interior collocation points. We define the three interior loss terms separately as
\begin{align}
    \mathcal{L}_{ru}(\theta;X_{\mathrm{int}})
    &:=
    \frac{1}{N_{\mathrm{int}}}
    \sum_{i=1}^{N_{\mathrm{int}}}
    \left|
    \mathcal{R}_{u}[\theta](\mathbf{x}^i_{\mathrm{int}})
    \right|^2,\\
    \mathcal{L}_{rv}(\theta;X_{\mathrm{int}})
    &:=
    \frac{1}{N_{\mathrm{int}}}
    \sum_{i=1}^{N_{\mathrm{int}}}
    \left|
    \mathcal{R}_{v}[\theta](\mathbf{x}^i_{\mathrm{int}})
    \right|^2,\\
    \mathcal{L}_{rc}(\theta;X_{\mathrm{int}})
    &:=
    \frac{1}{N_{\mathrm{int}}}
    \sum_{i=1}^{N_{\mathrm{int}}}
    \left|
    \mathcal{R}_{c}[\theta](\mathbf{x}^i_{\mathrm{int}})
    \right|^2.
\end{align}
Similarly, let $X_{\mathrm{wall}}=\{\mathbf{x}^i_{\mathrm{wall}}\}_{i=1}^{N_{\mathrm{wall}}}\subset \partial\Omega$
denote the set of boundary collocation points. We define the boundary loss terms by
\begin{align}
    \mathcal{L}_{u,\mathrm{bc}}(\theta;X_{\mathrm{wall}})
    &:=
    \frac{1}{N_{\mathrm{wall}}}
    \sum_{i=1}^{N_{\mathrm{wall}}}
    \left|
    u_\theta(\mathbf{x}^i_{\mathrm{wall}})
    -
    u_b(\mathbf{x}^i_{\mathrm{wall}})
    \right|^2,\\
    \mathcal{L}_{v,\mathrm{bc}}(\theta;X_{\mathrm{wall}})
    &:=
    \frac{1}{N_{\mathrm{wall}}}
    \sum_{i=1}^{N_{\mathrm{wall}}}
    \left|
    v_\theta(\mathbf{x}^i_{\mathrm{wall}})
    -
    v_b(\mathbf{x}^i_{\mathrm{wall}})
    \right|^2,
\end{align}
where \((u_b(\mathbf{x}),v_b(\mathbf{x}))\) denotes the prescribed boundary velocity, given by
\[
(u_b(\mathbf{x}),v_b(\mathbf{x}))=
\begin{cases}
(0,0), & \mathbf{x}\in \partial\Omega\setminus\{(x,1):x\in[0,1]\},\\
(U_{\mathrm{lid}},0), & \mathbf{x}\in \{(x,1):x\in[0,1]\}.
\end{cases}
\]
The total loss is then given by 
\begin{align}
\mathcal{L}(\theta)
:=
\lambda_{ru}\mathcal{L}_{ru}
+\lambda_{rv}\mathcal{L}_{rv}
+\lambda_{rc}\mathcal{L}_{rc}
+\lambda_{u,\mathrm{bc}}\mathcal{L}_{u,\mathrm{bc}}
+\lambda_{v,\mathrm{bc}}\mathcal{L}_{v,\mathrm{bc}}.
\end{align}

Here, we train PINN models for the respective problems by minimizing the corresponding loss functions using the experimental settings detailed in Section \ref{sec:results}, and the corresponding results are shown in Fig.~\ref{fig:failure_cases}. In both cases, the PINNs completely fail to capture the true solutions.
A key observation is that the PINN often converges to a constant solution over a large interior region away from the boundaries, regardless of whether the problem is an initial-boundary value problem or a boundary value problem.
We further emphasize that the phenomenon observed here appears to be broadly universal and relatively orthogonal to other previously identified PINN training pathologies, in the sense that it cannot be fully resolved by existing techniques. Even when improved network architectures \cite{wang2024piratenets}, loss-weighting strategies \cite{wang2021understanding,mcclenny2023self}, or advanced optimizers \cite{muller2023achieving,wang2025gradient,kiyani2025optimizing} are employed, very similar behavior can still arise for sufficiently challenging problems or parameter regimes, such as flows at high Reynolds numbers.

Motivated by this phenomenon, we will show that the issue is not merely an optimization difficulty, but reflects a more fundamental ill-posedness of the empirical PDE residual loss: for any given set of collocation points, there exists a spurious solution $\mathbf{u}^\dagger$ that exactly satisfies the PDE at those finite collocation points. 
Although the two examples above involve periodic or nonhomogeneous boundary conditions, the same mechanism can already be seen in a simpler homogeneous Dirichlet setting. We therefore state the following theorem in this simplified setting to isolate the core issue.

\begin{theorem}[Existence of spurious solutions for the empirical PINN loss]
\label{thm1}
Consider the homogeneous problem in \eqref{eq:problem} with \(\mathbf{f}=0\) and homogeneous Dirichlet boundary conditions $\mathbf{u}|_{\partial \Omega}=0,$
and let \(\mathbf{u}^\ast\) be its classical solution. Let $X_{\mathrm{int}}=\{(t^i_{\mathrm{int}},\mathbf{x}^i_{\mathrm{int}})\}_{i=1}^{N_{\mathrm{int}}}
\subset [0,T]\times\Omega$
be any finite set of interior collocation points, and define the empirical interior residual loss of \(\mathbf{u}\) by
\begin{align}
\mathcal{L}_{\mathrm{int}}(\mathbf{u};X_{\mathrm{int}})
=
\frac{1}{N_{\mathrm{int}}}
\sum_{i=1}^{N_{\mathrm{int}}}
\left|
\mathcal{R}_{\mathrm{int}}[\mathbf{u}]
(t^i_{\mathrm{int}},\mathbf{x}^i_{\mathrm{int}})
\right|^2 .
\end{align}
Then for any \(t_0>0\), there exists a smooth function \(\mathbf{u}^\dagger\in C^\infty([0,T]\times\Omega)\) such that
\begin{enumerate}
    \item \(\mathbf{u}^\dagger\) satisfies the initial and Dirichlet boundary conditions, namely,
    \begin{align}
         \mathbf{u}^\dagger|_{\partial\Omega} \equiv 0,
    \qquad
    \mathbf{u}^\dagger(0,\mathbf{x})=\mathbf{g}(\mathbf{x}), \quad \forall \mathbf{x} \in \Omega ;
    \end{align}
    \item \(\mathbf{u}^\dagger\) becomes trivial after time \(t_0\), in the sense that
    \begin{align}
         \mathbf{u}^\dagger(t,\mathbf{x})=0,
    \qquad
    \forall\, t\ge t_0,\ \mathbf{x}\in\Omega;
    \end{align}
    \item its empirical interior residual vanishes exactly:
    \begin{align}
        \mathcal{L}_{\mathrm{int}}(\mathbf{u}^\dagger;X_{\mathrm{int}})=0.
    \end{align}

\end{enumerate}
Consequently, for any \(t_0>0\), the spurious solution set
\begin{align}
    \mathcal{S}_{t_0}
:=
\left\{
\mathbf{u}\in C^\infty([0,T]\times\Omega)\; \middle| \;
\begin{array}{l}
\mathbf{u}|_{\partial \Omega} \equiv 0, \\
\mathbf{u}(0,\mathbf{x})=\mathbf{g}(\mathbf{x}), \quad \forall \mathbf{x} \in \Omega, \\
\mathbf{u}(t,\mathbf{x})=0,\quad \forall\, t\ge t_0,\ \mathbf{x}\in\Omega,\\
\mathcal{L}_{\mathrm{int}}(\mathbf{u};X_{\mathrm{int}}) = 0
\end{array}
\right\}
\end{align}
is nonempty.
\end{theorem}

The proof is provided in Appendix~\ref{app:proof1} and is based on a smooth mollifier construction. The main idea is to construct a smooth function \(\mathbf{u}^\dagger\) that agrees with the true solution \(\mathbf{u}^\ast\) near all collocation points before \(t_0\), but becomes identically zero for \(t \ge t_0\). More precisely, we choose a smooth time cutoff \(\alpha(t)\) such that
\begin{align}
    \alpha(t)=
\begin{cases}
1, & \text{near } t_{\text{int}}^i,\quad t_{\text{int}}^i<t_0,\\
0, & \text{near } t_{\text{int}}^i,\quad t_{\text{int}}^i\ge t_0,\\
0, & t\ge t_0.
\end{cases}
\end{align}
Then we define
\begin{align}
\mathbf{u}^\dagger(t,\mathbf{x})=\alpha(t)\mathbf{u}^\ast(t,\mathbf{x}).
\end{align}
The resulting function \(\mathbf{u}^\dagger\) remains smooth, satisfies the same initial and boundary conditions, and becomes trivial after \(t_0\). Moreover, its PDE residual still vanishes at every collocation point: before \(t_0\), \(\mathbf{u}^\dagger\) agrees locally with \(\mathbf{u}^\ast\), while after \(t_0\), it agrees locally with the zero solution, which also satisfies the homogeneous PDE. 
Similar constructions can be adapted to boundary-value problems with inhomogeneous data, and we observe that the same pathology arises more broadly. A systematic treatment of such general settings is beyond the scope of this work and is left for future study.

The theorem reveals a fundamental limitation of the conventional empirical PINN loss: when the collocation points are fixed throughout training, the optimization landscape may contain many poor global minima associated with trivial or spurious solutions. Minimizing this objective alone is therefore insufficient to prevent convergence to such undesirable states. Moreover, due to spectral bias, PINNs may preferentially converge to these spurious solutions. Our experiments provide strong evidence for this undesirable behavior.

\begin{remark}[Beyond fixed collocation points]
Theorem~\ref{thm1} is an existence result for a \emph{fixed} finite set of collocation points. In particular, the construction of $\mathbf{u}^\dagger$ exploits the fact that the empirical residual is evaluated only on a prescribed set $X_{\mathrm{int}}$, and therefore does not by itself imply the same exact zero-residual property under random resampling or full space--time integration.

Nevertheless, the same pathology may still persist in practice under random sampling or small-batch training. Indeed, if a spurious solution transitions from the true solution branch to the trivial state within a narrow transition layer of width $h$, then its residual is concentrated in that layer. For a standard mean-squared PDE residual loss, the contribution of such a transition layer is typically of order $O(h^{-1})$ due to the specific construction of the spurious solutions. Therefore, unless the optimization is able to reduce the PDE residual loss below this transition-layer scale, such spurious solutions may not be effectively excluded, even when collocation points are randomly resampled during training.
\end{remark}

\subsection{Pseudo-time stepping} 
\label{sec:pseduo_time}

In this section, we show that the pseudo-time stepping can effectively mitigate the training pathology identified in the previous section. Pseudo-time stepping, or pseudo-transient continuation, has a long history in numerical analysis, where it is widely used to stabilize iterative solution methods \cite{kelley1998convergence,coffey2003pseudotransient}. More recently, related ideas have been adopted in the PINN literature and have shown consistent empirical gains \cite{cao2023tsonn,cao2024surrogate,zhang2025pseudo,chiu2026scale,wei2026bridging}. These benefits are typically motivated from a classical numerical-analysis perspective, which attributes them to improved conditioning of the training dynamics and a better-behaved loss landscape.

\begin{figure}[t]
    \centering
    \includegraphics[width=1.0\linewidth]{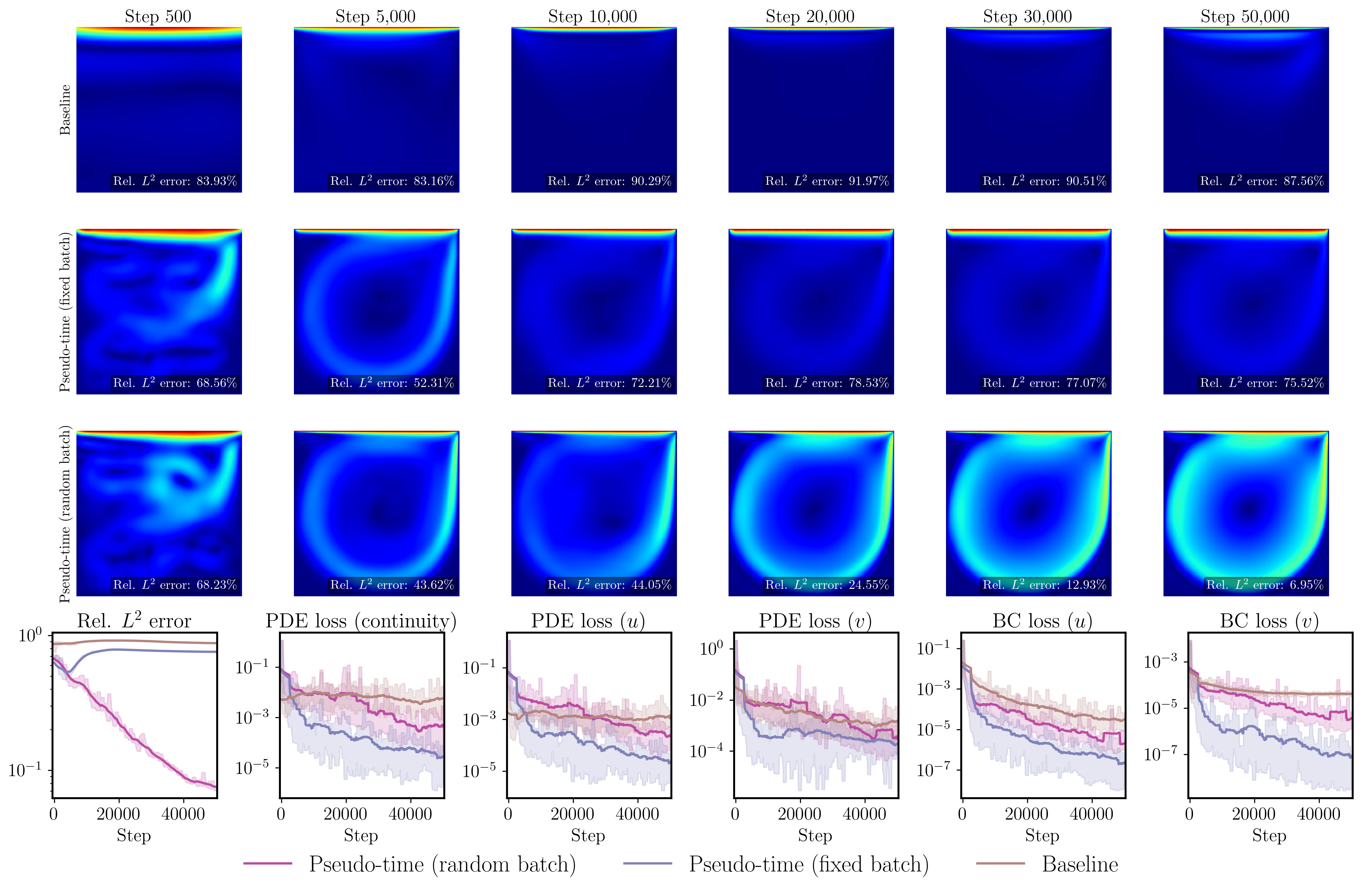}
\caption{{\em Lid-driven cavity.} Predicted velocity fields $U = \sqrt{u^2 + v^2}$ at different steps of gradient descent during training. From top to bottom, the rows correspond to: the baseline PINN model; the baseline PINN model with pseudo-time stepping trained using a fixed-batch of collocation points; and the baseline model with pseudo-time stepping trained using randomly resampled collocation points at each iteration. The last row shows the convergence histories of the relative $L^2$ error and the loss for these three methods.}
    \label{fig:ldc_train}
\end{figure}

\begin{figure}[h]
    \centering
    \includegraphics[width=1.0\linewidth]{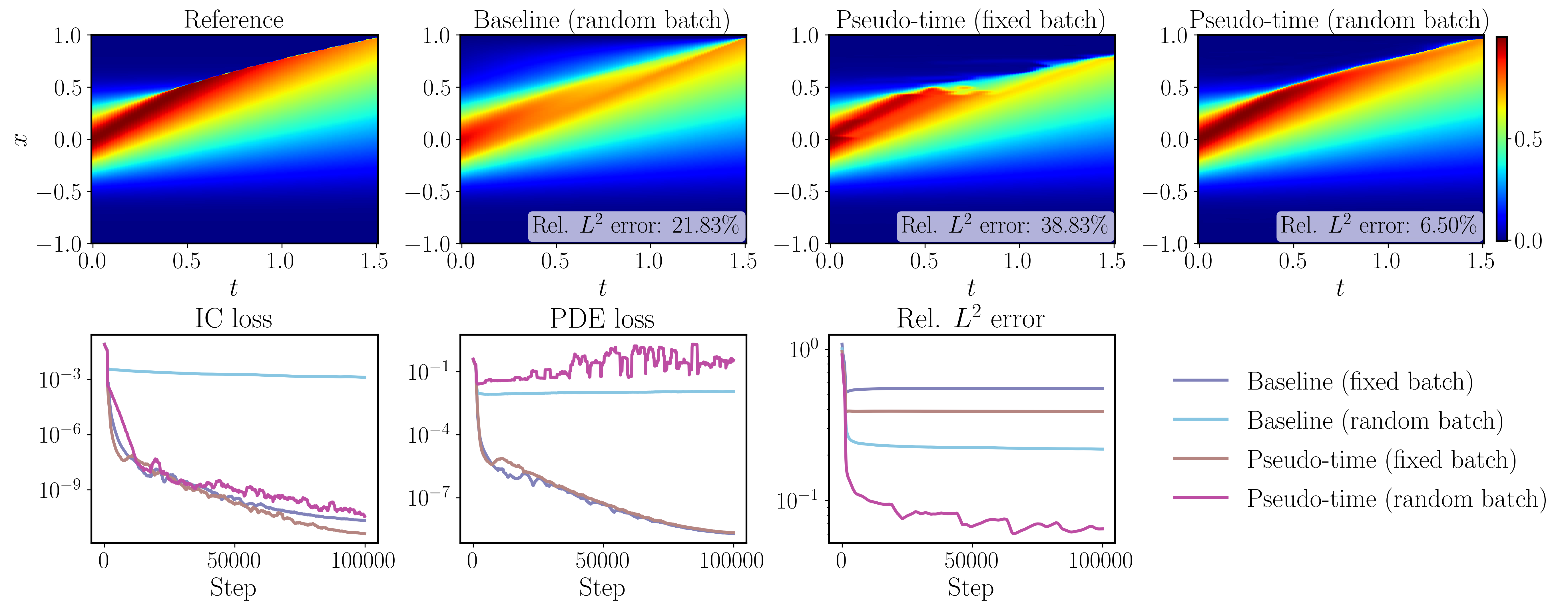}
 \caption{{\em Inviscid Burgers' equation.} Top: Comparison of the predicted solutions obtained by the baseline PINN model and by pseudo-time stepping, with and without random collocation-point resampling at each iteration. Bottom: training loss and relative \(L^2\) error histories.}
    \label{fig:inviscid_burgers_sample_effect}
\end{figure}

Motivated by this observation, we replace the standard PDE residual loss with the pseudo-time-relaxed formulation in Eq.~\eqref{eq:pts_loss} and repeat the same experiment on the lid-driven cavity problem under identical settings. The results are shown in Fig.~\ref{fig:ldc_train}. Although pseudo-time stepping leads to a clear improvement, the underlying reason is not what one might initially expect. In particular, the fixed-batch variant (row 2) consistently attains a lower PDE residual loss than the random-resampling variant (row 3), yet it still converges to an incorrect solution.

To examine whether the same phenomenon also arises in time-dependent problems, we further consider the one-dimensional inviscid Burgers equation. The problem setup is provided in Appendix~\ref{appendix:benchmarks}, and the training protocol is described in Section~\ref{sec:results}. This problem is substantially more challenging than the viscous Burgers equation commonly studied in the PINN literature, because its solution develops discontinuous shocks. Figure~\ref{fig:inviscid_burgers_sample_effect} shows the predicted solutions together with the corresponding training-loss and relative \(L^2\) error histories. We observe the same qualitative behavior: pseudo-time stepping with a fixed collocation-point sets achieves excellent loss convergence, yet yields the worst predictive accuracy among the variants considered. In addition, the PDE residual loss obtained by pseudo-time stepping with randomly sampled collocation points is significantly higher. This behavior is expected for this problem when the PINN attempts to learn the correct physical solution, since the pointwise PDE residual can become arbitrarily large near the shock. 

We emphasize that this phenomenon appears to be quite universal, and it cannot be adequately explained from the traditional optimization perspective alone. If the primary role of pseudo-time stepping were simply to improve conditioning or make the loss landscape easier to optimize, then the fixed-batch variant should perform at least as well as the random-resampling variant, provided that the training loss is properly minimized. However, this is not what we observe. This discrepancy naturally leads to the following question:

\begin{tcolorbox}[
    colback=blue!5!white,
    colframe=blue!50!black,
    boxrule=0.8pt,
    arc=2mm,
    left=2mm,
    right=2mm,
    top=1mm,
    bottom=1mm
]
\centering
\emph{\textbf{Question:} If pseudo-time stepping does not improve PINN training primarily through better conditioning, then what is the mechanism behind its effectiveness?}
\end{tcolorbox}


As we show in the following analysis, the primary benefit of pseudo-time stepping is more fundamental: it modifies the training trajectory to discourage convergence to the spurious solutions identified above. In particular, under collocation-point resampling, pseudo-time stepping makes these spurious solutions unstable by inducing much larger residuals on newly sampled points, thereby driving the optimization away from them.

With this question in mind, we now return to how pseudo-time stepping is formulated for PINNs in boundary value problems. In this setting, instead of directly minimizing the interior PDE residual
\begin{align}
    \mathcal{R}_{\mathrm{int}}[\mathbf{u}_\theta](\mathbf{x}) = 0,
\end{align}
pseudo-time stepping introduces auxiliary \emph{relaxation dynamics} for the network prediction \(\mathbf{u}_\theta\) along an artificial variable \(s\). Here, \(s\) is not a physical variable of the original PDE, but an auxiliary variable used to parameterize the training dynamics of the network parameters under gradient flow,
\begin{align}
    \frac{d\theta(s)}{ds} = - \nabla_\theta \mathcal{L}(\theta(s)).
\end{align}
Motivated by this viewpoint, pseudo-time stepping may be formally interpreted through the relaxed relation
\begin{align}
    \partial_s u_{\theta(s)} +\mathcal{R}_{\mathrm{int}}[u_{\theta(s)}] = 0.
\end{align}
At convergence in the artificial time \(s\), one has \(\partial_s u_{\theta(s)} = 0\), and hence the steady state satisfies the target residual equation
\begin{align}
    \mathcal{R}_{\mathrm{int}}[u_{\theta(s)}] = 0.
\end{align}

Let \(\theta^k\) denote the network parameters at the \(k\)-th training step. Discretizing the artificial relaxation in \(s\) with step size \(\tau>0\) gives
\begin{align}
    \frac{\mathbf{u}^{k} - \mathbf{u}^{k-1}}{\tau} + \mathcal{R}_{\mathrm{int}}[\mathbf{u}^{k}] = 0.
\end{align}
In the PINN setting, this relation is interpreted at the level of the network prediction, namely
\begin{align}
    \mathbf{u}^k \approx \mathbf{u}_{\theta^k}, \qquad \mathbf{u}^{k-1} \approx\mathbf{u}_{\theta^{k-1}},
\end{align}
where \(\theta^k\) and \(\theta^{k-1}\) denote the network parameters at two consecutive training steps.
This motivates the following pseudo-time-relaxed residual loss
\begin{align}
\label{eq:pts_loss}
\mathcal{L}_{\mathrm{pts}}(\theta;\theta^{k-1})
=
\frac{1}{N_\text{int}}\sum_{i=1}^{N_\text{int}}
\left|
    \frac{\mathbf{u}_{\theta}(\mathbf{x}_{\text{int}}^i)- \mathbf{u}_{\theta^{k-1}}(\mathbf{x}_{\text{int}}^i)}{\tau}
+ \mathcal{R}_{\mathrm{int}}[\mathbf{u}_{\theta}](\mathbf{x}_{\text{int}}^i)
\right|^2.
\end{align}
where $\{\mathbf{x}_{\text{int}}^i\}_{i=1}^{N_{\text{int}}}$ is a set of collocation points sampled in the computational domain $\Omega$ of interest.
The next parameter iterate is obtained by optimizing the total PINN loss, in which the original PDE residual loss is replaced by its pseudo-time-relaxed counterpart. In this update, the optimization is performed with respect to the current parameters $\theta^k$, while the previous iterate $\theta^{k-1}$ is treated as fixed and the boundary condition loss is kept unchanged.

Intuitively, the new network prediction \(\mathbf{u}_{\theta^{k}}\) is encouraged not only to reduce the PDE residual, but also to remain close to the previous prediction \(\mathbf{u}_{\theta^{k-1}}\). This leads to a more stable training trajectory, in contrast to directly minimizing the residual in a single shot.

\begin{remark}[Implicit versus explicit pseudo-time stepping]
In this work, we focus on the implicit pseudo-time stepping formulation. Although an explicit pseudo-time stepping variant can also be derived, our experiments indicate that it is generally less robust. This is consistent with the observations reported in TSONN \cite{cao2023tsonn}.
\end{remark}

\begin{remark}[Extension to time-dependent PDEs]
For time-dependent PDEs, the artificial variable \(s\) should not be confused with the physical time variable \(t\). For a time-dependent PDE,
\begin{align}
    \mathbf{u}_t + \mathcal{D}[\mathbf{u}] = 0,
\end{align}
the interior residual becomes
\begin{align}
    \mathcal{R}_{\mathrm{int}}[\mathbf{u}_\theta](t,\mathbf{x})
    =
    \partial_t \mathbf{u}_\theta(t,\mathbf{x}) + \mathcal{D}[\mathbf{u}_\theta](t,\mathbf{x}).
\end{align}
Pseudo-time stepping is then applied exactly as in the boundary value setting, but now to this time-dependent residual, leading to
\begin{align}
\label{eq:time_pts_loss}
\mathcal{L}_{\mathrm{pts}}(\theta;\theta^{k-1})
=
\frac{1}{N_\mathrm{int}}\sum_{i=1}^{N_\mathrm{int}}
\left|
\frac{\mathbf{u}_{\theta}(t_{\mathrm{int}}^i,\mathbf{x}_{\mathrm{int}}^i)-\mathbf{u}_{\theta^{k-1}}(t_{\mathrm{int}}^i,\mathbf{x}_{\mathrm{int}}^i)}{\tau}
+ \mathcal{R}_{\mathrm{int}}[\mathbf{u}_{\theta}](t_{\mathrm{int}}^i,\mathbf{x}_{\mathrm{int}}^i)
\right|^2.
\end{align}
Thus, \(t\) remains the physical time variable in the PDE, while \(s\) only parameterizes the artificial relaxation used during training.
\end{remark}

For the lid-driven cavity problem and the result presented in Fig.~\ref{fig:ldc_train}, we apply pseudo-time stepping to relax the PDE residual loss. This requires access to consecutive network parameters during training. Specifically, at iteration 
$k$, we retain the solution from the previous iterate
$$
\mathbf{u}_{\theta^{k-1}}(\mathbf{x})=\left(u_{\theta^{k-1}}(\mathbf{x}), v_{\theta^{k-1}}(\mathbf{x}), p_{\theta^{k-1}}(\mathbf{x})\right),
$$
where $\theta^{k-1}$ denotes the network parameters from iteration $k-1$. The specific pseudo-time-relaxed interior loss terms are defined as
\begin{align}
    \label{eq:ldc_pts_loss}
    \mathcal{L}_{ru}^{\mathrm{pts}}(\theta;\theta^{k-1},X_{\mathrm{int}})
    &:=
    \frac{1}{N_{\mathrm{int}}}
    \sum_{i=1}^{N_{\mathrm{int}}}
    \left|
    \frac{
    u_{\theta}(\mathbf{x}^i_{\mathrm{int}})
    -
    u_{\theta^{k-1}}(\mathbf{x}^i_{\mathrm{int}})
    }{\tau_u}
    +
    \mathcal{R}_u[\theta](\mathbf{x}^i_{\mathrm{int}})
    \right|^2,\\
    \mathcal{L}_{rv}^{\mathrm{pts}}(\theta;\theta^{k-1},X_{\mathrm{int}})
    &:=
    \frac{1}{N_{\mathrm{int}}}
    \sum_{i=1}^{N_{\mathrm{int}}}
    \left|
    \frac{
    v_{\theta}(\mathbf{x}^i_{\mathrm{int}})
    -
    v_{\theta^{k-1}}(\mathbf{x}^i_{\mathrm{int}})
    }{\tau_v}
    +
    \mathcal{R}_v[\theta](\mathbf{x}^i_{\mathrm{int}})
    \right|^2,\\
    \mathcal{L}_{rc}^{\mathrm{pts}}(\theta;\theta^{k-1},X_{\mathrm{int}})
    &:=
    \frac{1}{N_{\mathrm{int}}}
    \sum_{i=1}^{N_{\mathrm{int}}}
    \left|
    \frac{
    p_{\theta}(\mathbf{x}^i_{\mathrm{int}})
    -
    p_{\theta^{k-1}}(\mathbf{x}^i_{\mathrm{int}})
    }{\tau_p}
    +
    \mathcal{R}_c[\theta](\mathbf{x}^i_{\mathrm{int}})
    \right|^2.
\end{align}
where \(\tau_u,\tau_v,\tau_p>0\) are pseudo-time step sizes. To obtain the results in Fig.~\ref{fig:ldc_train}, we follow Scale-PINNs \cite{chiu2026scale} and set \(\tau=\tau_u=\tau_v=\tau_p=0.1\). 
As will be shown in Fig.~\ref{fig:ldc_tau_ablation}, the performance of pseudo-time stepping is highly sensitive to the choice of these step sizes, and cannot be reliably tuned by loss values, which motivates the development of an adaptive pseudo-time stepping strategy that automatically identifies suitable values during training.

\begin{figure}[h]
    \centering
    \includegraphics[width=1.0\linewidth]{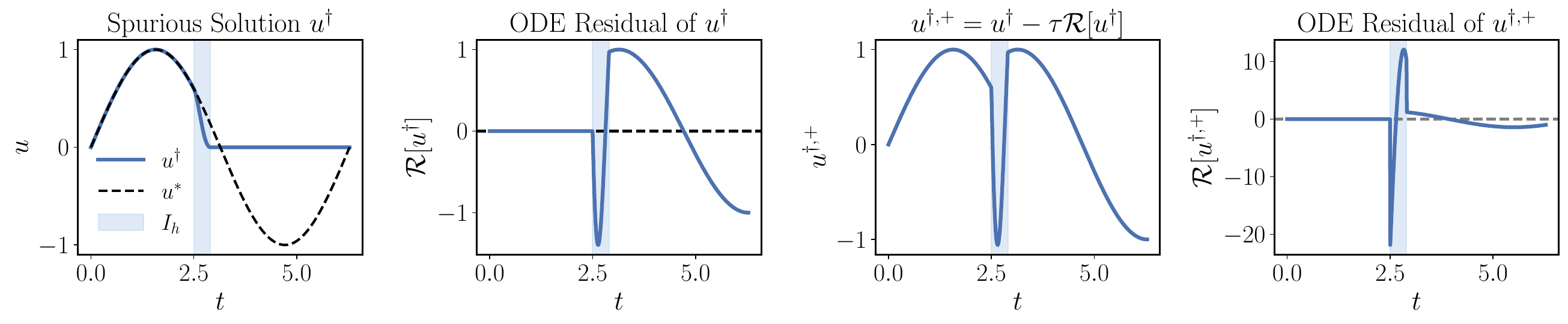}
\caption{{\em Illustration of the pseudo-time stepping mechanism in the PDE residual loss.} From left to right, when the model approaches a spurious solution \(u^\dagger\) with a transition layer centered at time \(t_0\) and width \(h\), a large residual error is produced in the transition region. The pseudo-time update
\(u^{\dagger,+}=u^\dagger - \tau \mathcal{R}[u^\dagger]\)
significantly increases the residual in this region, thus strengthening the loss penalty against such spurious solutions.}
    \label{fig:spurious_solution}
\end{figure}

We now turn to the question raised at the beginning of this section. To build intuition for the underlying mechanism of pseudo-time stepping, we begin with a simple ODE example
\[
u_t=\cos t,
\]
with zero initial condition. For this example, we set \(\tau=1\) and choose a transition-layer width \(h=0.4\). Suppose that at some stage of training, the PINN prediction approaches a spurious profile with a sharp transition layer. As shown in Fig.~\ref{fig:spurious_solution}, the PDE residual is large inside this transition layer, yet such a localized defect may not be adequately captured by a finite set of collocation points. After one pseudo-time stepping update, the solution profile is substantially altered, and the PDE residual of the updated state \(u^{\dagger,+}\) increases by several orders of magnitude within the transition region.

To formalize the intuition above, we prove the following theorem.
\begin{theorem}[Pseudo-time stepping destabilizes spurious solutions under resampling]
\label{thm:unstable}
Consider the initial boundary value problem in \eqref{eq:problem}, and assume that \(\mathcal D\) is a linear differential operator acting only on the spatial variable \(\mathbf x\). Let \(\mathbf{u}^\ast\) be a classical solution.

Fix \(t_0\in(0,T)\) and \(h>0\). We consider the constructed spurious solution
\begin{align}
    \mathbf{u}^\dagger(t,\mathbf x)=\alpha_h(t)\,\mathbf{u}^\ast(t,\mathbf x),
\end{align}
where \(\alpha_h\in C^\infty(\mathbb R)\) is a smooth cutoff satisfying
\begin{align}
\alpha_h(t)=
\begin{cases}
1, & t\le t_0,\\[4pt]
0, & t\ge t_0+h.
\end{cases}
\end{align}
We denote its transition layer by $I_h:=[t_0,t_0+h].$
Next, we define the exact output of one explicit pseudo-time stepping update applied to \(\mathbf{u}^\dagger\) by
\begin{align}
    \mathbf{u}^{\dagger,+}:=\mathbf{u}^\dagger - \tau\,\mathcal R_{\mathrm{int}}[\mathbf{u}^\dagger],
\qquad \tau>0.
\end{align}
Let $\widetilde X_{\mathrm{int}}
=
\{(\tilde t_j,\tilde{\mathbf x}_j)\}_{j=1}^M
\subset (0,T)\times\Omega$
be i.i.d. random interior collocation points on \((0,T) \times \Omega\). We define the empirical residual loss on these newly sampled collocation points by
\begin{align}
    \mathcal L_{\mathrm{int}}^{\mathrm{new}}(\mathbf{v})
:=
\frac{1}{M}\sum_{j=1}^M
\left|
\mathcal R_{\mathrm{int}}[\mathbf{v}](\tilde t_j,\tilde{\mathbf x}_j)
\right|^2.
\end{align}
Then, with expectation taken over the random sample \(\widetilde X_{\mathrm{int}}\),
\begin{align}
    \mathbb E\,\mathcal L_{\mathrm{int}}^{\mathrm{new}}(\mathbf{u}^\dagger)=O(h^{-1}),
\qquad
\mathbb E\,\mathcal L_{\mathrm{int}}^{\mathrm{new}}(\mathbf{u}^{\dagger,+})
=O\!\left(h^{-1}+\tau^2 h^{-3}\right).
\end{align}
Moreover, if \(\mathbf{u}^\ast(t_0,\cdot)\not\equiv 0\), then these orders are sharp:
\begin{align}
    \mathbb E\,\mathcal L_{\mathrm{int}}^{\mathrm{new}}(\mathbf{u}^\dagger)\asymp h^{-1},
\qquad
\mathbb E\,\mathcal L_{\mathrm{int}}^{\mathrm{new}}(\mathbf{u}^{\dagger,+})
\asymp h^{-1}+\tau^2 h^{-3}.
\end{align}

In particular, one pseudo-time stepping update amplifies the empirical residual on newly sampled collocation points from order \(h^{-1}\) to order \(\tau^2 h^{-3}\).
\end{theorem}

The proof of this theorem is provided in Appendix~\ref{app:proof2}. Although the implemented loss corresponds to an implicit pseudo-time formulation, the explicit residual update in the theorem is used only as an analytical proxy to illustrate the residual-amplification mechanism. The theorem shows that, when combined with collocation resampling, pseudo-time stepping can destabilize spurious solutions by exposing and amplifying their hidden residual defects. As a result, the pseudo-time-relaxed residual loss penalizes such solutions more strongly than the conventional residual loss, thereby helping steer training toward the physical solution when the relaxed loss is properly minimized.

It is worth noting that the choice of \(\tau\) is subtle. In the idealized analysis, a larger \(\tau\) leads to stronger amplification of hidden residual defects. In practice, however, we optimize the pseudo-time-relaxed loss in \eqref{eq:pts_loss}, and there is no guarantee that one iteration of gradient descent update will exactly reproduce the corresponding pseudo-time stepping update from the current iterate. Moreover, if \(\tau\) is too large, then the relaxed system
\begin{align}
    \frac{\mathbf{u}^k-\mathbf{u}^{k-1}}{\tau}+\mathcal{R}_{\mathrm{int}}\!\left[\mathbf{u}^k\right]=0
\end{align}
corresponds to a coarse pseudo-time discretization, which introduces a larger discretization error in approximating the underlying continuous pseudo-time dynamics, 
This, in turn, can make the resulting training objective harder to optimize stably.
Therefore, \(\tau\) must be chosen carefully to balance defect amplification and optimization stability. As we discuss in the next section, this tradeoff motivates the need for an adaptive strategy to choose \(\tau\) appropriately during training.

\paragraph{Related work.} TSONN \cite{cao2023tsonn} and Scale-PINN \cite{chiu2026scale} showed empirically that pseudo-time-style losses can substantially improve PINN training, with the pseudo-time step $\tau$ chosen through heuristic search. In contrast, our work provides, to the best of our knowledge, the first theoretical explanation for this phenomenon through the lens of spurious-solution formation, and further leads to a principled adaptive strategy for selecting $\tau$ directly from this analysis.

More broadly, many techniques proposed for PINNs can be viewed, at least intuitively, as ways to suppress spurious solutions. For example, adaptive loss-weighting schemes \cite{mcclenny2020self,wang2021understanding,wang2022and,anagnostopoulos2024residual} often assign larger weights to initial and boundary condition losses, or to regions with large PDE residuals, thereby discouraging convergence to trivial or nonphysical solutions. Likewise, adaptive sampling strategies \cite{daw2022mitigating,wu2023comprehensive} place more collocation points in high-residual regions, enabling the model to better resolve difficult parts of the domain and enforce the PDE constraints more effectively. Causal training \cite{wang2024respecting,penwarden2023unified} follows a similar principle in time-dependent problems by placing greater emphasis on reducing residuals at earlier times.

From this perspective, pseudo-time stepping provides a complementary mechanism. Rather than only reweighting losses or redistributing training points, it directly modifies the residual formulation itself, so that hidden defects associated with spurious solutions are revealed and amplified, making these undesirable solutions much more strongly penalized during training.

\subsection{Adaptive pseudo-time stepping}
\label{sec:adaptive_pseduo_time}

Although we have both theoretically and empirically demonstrated the effectiveness of pseudo-time stepping in PINNs, we also observe a fundamental practical limitation of the current method: its performance is highly sensitive to the choice of $\tau$. More importantly, the training losses obtained with different values of $\tau$ are often very similar, even when the final solution quality differs substantially.

For the illustrative lid-driven cavity example considered here, we conduct an ablation study of pseudo-time stepping under different choices of $\tau$, and the corresponding results are summarized in Fig.~\ref{fig:ldc_tau_ablation}. We observe that the error convergence can differ dramatically across different choices of $\tau$, even when the corresponding loss values remain highly comparable.
Our further experiments show that the optimal value of $\tau$ depends strongly on many aspects of the problem setup, including the PDE being solved, the optimizer, the network architecture, and the loss weighting strategy. As a result, it is difficult to tune $\tau$ by monitoring the training loss alone. This is especially problematic in PINNs for solving forward problems, as the loss is often the only directly accessible metric during training. In addition, extensive hyperparameter tuning over $\tau$ can introduce substantial additional computational cost in practice.

\begin{figure}[h]
    \centering
    \includegraphics[width=1.0\linewidth]{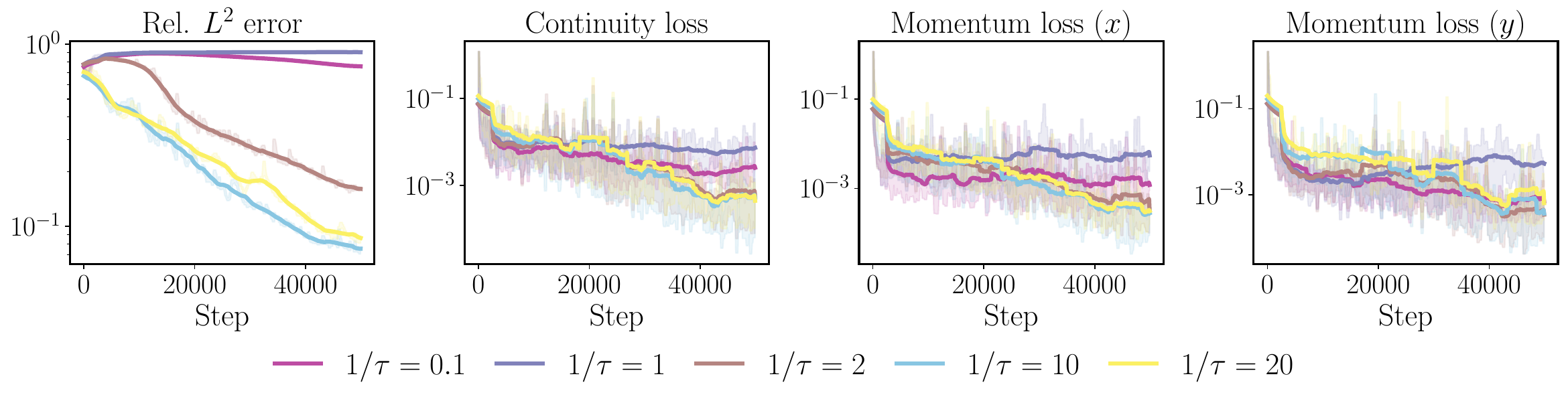}
\caption{{\em Lid-driven cavity.} Histories of the relative \(L^2\) error and PDE residual losses during PINN training with pseudo-time stepping for different choices of \(\tau\).}
    \label{fig:ldc_tau_ablation}
\end{figure}

To address the limitation discussed above, we seek to develop a self-adaptive strategy for selecting 
$\tau$. We begin by analyzing the local behavior of the pseudo-time stepping update
\begin{align}
    \frac{\mathbf{u}^{k}-\mathbf{u}^{k-1}}{\tau} + \mathcal{R}_{\mathrm{int}}[\mathbf{u}^{k}] = 0 .
    \label{eq:pts_update_appendix}
\end{align}
The key idea is that the stability and effectiveness of this iteration depend on the local Jacobian of the residual operator.

Let \(\mathbf{u}^\ast\) denote a solution of the interior residual equation,
\begin{align}
    \mathcal{R}_{\mathrm{int}}[\mathbf{u}^\ast]=0,
\end{align}
and define the error \(\mathbf{e}^k := \mathbf{u}^k-\mathbf{u}^\ast\). Linearizing \(\mathcal{R}_{\mathrm{int}}\) around \(\mathbf{u}^\ast\) gives
\begin{align}
    \mathcal{R}_{\mathrm{int}}[\mathbf{u}^{k}]
    \approx
    J_\ast (\mathbf{u}^{k}-\mathbf{u}^\ast)
    =
    J_\ast \mathbf{e}^{k},
    \qquad
    J_\ast := \frac{\partial \mathcal{R}_{\mathrm{int}}}{\partial u}(\mathbf{u}^\ast).
\end{align}
Substituting this approximation into \eqref{eq:pts_update_appendix} yields
\begin{align}
    \frac{\mathbf{e}^{k}-\mathbf{e}^{k-1}}{\tau} + J_\ast \mathbf{e}^{k} \approx 0,
\end{align}
or equivalently
\begin{align}
    \mathbf{e}^{k} \approx (I+\tau J_\ast)^{-1} \mathbf{e}^{k-1}.
\end{align}
Therefore, if \(\lambda_i\) is an eigenvalue of \(J_\ast\), for the iteration to be locally contractive, we require
\begin{align}
    \rho\!\left((I+\tau J_\ast)^{-1}\right)
    =
    \max_i \left|\frac{1}{1+\tau\lambda_i}\right|
    < 1,
\end{align}
which shows that the choice of \(\tau\) is governed by the local spectrum of \(J_\ast\). In practice, the useful range of \(\tau\) is determined by both the local residual Jacobian and the stability of the neural-network optimization used to enforce the pseudo-time update.

Theorem~\ref{thm:unstable} further clarifies how \(\tau\) should be selected. Since the amplification ratio scales like \(\tau^2 h^{-3}\), \(\tau\) should be chosen as large as possible while still preserving stable training. This observation naturally motivates an adaptive strategy that chooses \(\tau\) according to local Jacobian information.

In practice, however, explicitly estimating the spectrum of the Jacobian during training is prohibitively expensive. We therefore seek a cheap local surrogate that captures the same scaling behavior. To this end, we compare the change in the iterate with the corresponding change in the residual:
\begin{align}
    \tau^k
    :=
    \frac{\|\mathbf{u}^{k}-\mathbf{u}^{k-1}\|}
    {\|\mathcal{R}_{\mathrm{int}}[\mathbf{u}^{k}] - \mathcal{R}_{\mathrm{int}}[\mathbf{u}^{k-1}]\|}.
\end{align}
This quantity can be interpreted as an approximation to the inverse local Jacobian magnitude along the current update direction. Indeed, by a first-order Taylor expansion at \(\mathbf{u}^{k-1}\),
\begin{align}
    \mathcal{R}_{\mathrm{int}}[\mathbf{u}^{k}]
    -
    \mathcal{R}_{\mathrm{int}}[\mathbf{u}^{k-1}]
    \approx
    J_{k-1}(\mathbf{u}^{k}-\mathbf{u}^{k-1}),
    \qquad
    J_{k-1}
    :=
    \frac{\partial \mathcal{R}_{\mathrm{int}}}{\partial u}(\mathbf{u}^{k-1}).
\end{align}
Hence,
\begin{align}
\label{eq:bb}
    \frac{
    \|\mathcal{R}_{\mathrm{int}}[\mathbf{u}^{k}]
    -
    \mathcal{R}_{\mathrm{int}}[\mathbf{u}^{k-1}]\|
    }
    {\|\mathbf{u}^{k}-\mathbf{u}^{k-1}\|}
    \approx
    \|J_{k-1}\mathbf{v}_{k}\|,
    \qquad
    \mathbf{v}_{k}
    :=
    \frac{\mathbf{u}^{k}-\mathbf{u}^{k-1}}
    {\|\mathbf{u}^{k}-\mathbf{u}^{k-1}\|},
\end{align}
so that
\begin{align}
    \tau^k \approx \frac{1}{\|J_{k-1}\mathbf{v}_{k}\|}.
\end{align}
The estimator \eqref{eq:bb} is a finite-difference surrogate for the local Jacobian magnitude, analogous in spirit to the two-point step-size methods of Barzilai and Borwein \cite{barzilai1988two} in unconstrained optimization, but adapted here to operate on the residual operator in solution space rather than on the gradient in parameter space.
Although \(\tau^k\) does not recover the full spectrum of the Jacobian, it provides an inexpensive estimate of the inverse local stiffness in the direction most relevant to the current update. This makes it a practical surrogate for adapting \(\tau\) during training.

\begin{algorithm}[t]
\caption{Training PINNs with adaptive pseudo-time stepping}
\label{alg:adaptive_pts}
\begin{algorithmic}[1]
\Require initial network parameters \(\theta^{0}, \theta^{1}\), initial pseudo-time step \(\tau^{1}>0\), learning rate \(\eta>0\), total training steps \(K\), update frequency \(m \in \mathbb{N}\), smoothing factor \(\beta\in(0,1]\), small constant \(\varepsilon>0\)

\For{$k=1,2,\dots,K$}
    \State Sample random collocation sets \(X_{\mathrm{int}}^{k}\), \(X_{\mathrm{bc}}^{k}\), and \(X_{\mathrm{ic}}^{k}\)

  \State Update the pseudo-time step size every $m$ iterations by
\begin{empheq}[box=\fbox]{align}
    \Delta \mathbf{u}^{k}
    &:=
    \mathbf{u}_{\theta^{k}}(X_{\mathrm{int}}^{k})
    -
    \mathbf{u}_{\theta^{k-1}}(X_{\mathrm{int}}^{k}),
    \\
    \Delta \mathbf{r}^{k}
    &:=
    \mathcal{R}_{\mathrm{int}}[\mathbf{u}_{\theta^{k}}](X_{\mathrm{int}}^{k})
    -
    \mathcal{R}_{\mathrm{int}}[\mathbf{u}_{\theta^{k-1}}](X_{\mathrm{int}}^{k}),
    \\
    \gamma^{k} 
    &:=
    \mathrm{CosineDecay}\!\left(
    \mathcal{L}_{\mathrm{int}}(\theta^{k};X_{\mathrm{int}}^{k}),
    \mathcal{L}_{\mathrm{int}}(\theta^{0};X_{\mathrm{int}}^{0})
    \right),
    \\
    \widehat{\tau}^{k}
    &:=
    \gamma^{k}
    \frac{\|\Delta \mathbf{u}^{k}\|_{2}}{\|\Delta \mathbf{r}^{k}\|_{2}+\varepsilon},
    \\
    \tau^{k}
    &:=
    \operatorname{stopgrad}\!\left(
    \begin{cases}
        \displaystyle
        (1-\beta)\tau^{k-1} + \beta \widehat{\tau}^{k},
        & k \bmod m = 0, \\[2mm]
        \tau^{k-1},
        & \text{otherwise}.
    \end{cases}
    \right).
\end{empheq}

\State Define the pseudo-time stepping PINN loss at iteration \(k\) as a function of the current parameters \(\theta\), with the previous iterate \(\theta^{k-1}\) held fixed:
\begin{align}
    \mathcal{L}^{k}(\theta;\theta^{k-1})
    &:=
    \frac{1}{|X_{\mathrm{int}}^{k}|}
    \sum_{\mathbf{x}\in X_{\mathrm{int}}^{k}}
    \left|
    \frac{\mathbf{u}_{\theta}(\mathbf{x})-\mathbf{u}_{\theta^{k-1}}(\mathbf{x})}{\tau^{k}}
    + \mathcal{R}_{\mathrm{int}}[\mathbf{u}_{\theta}](\mathbf{x})
    \right|^{2}
    + \lambda_{\mathrm{bc}} \mathcal{L}_{\mathrm{bc}}(\theta;X_{\mathrm{bc}}^{k})
    + \lambda_{\mathrm{ic}} \mathcal{L}_{\mathrm{ic}}(\theta;X_{\mathrm{ic}}^{k}).
\end{align}

\State Compute the gradient with respect to the current parameters \(\theta\) evaluated at \(\theta=\theta^k\), and perform one gradient-descent step:
\begin{align}
    g^{k} &\gets \nabla_{\theta}\mathcal{L}^{k}(\theta;\theta^{k-1})\big|_{\theta=\theta^{k}},\\
    \theta^{k+1} &\gets \theta^{k} - \eta g^{k}.
\end{align}
\EndFor
\end{algorithmic}
\end{algorithm}

Based on the discussion above, we now present a practical training procedure for PINNs with adaptive pseudo-time stepping; the full procedure is summarized in Algorithm~\ref{alg:adaptive_pts}. In implementation, it is important to apply a stop-gradient operation when updating the pseudo-time step sizes. This is because the pseudo-time steps are computed from the network predictions and therefore depend on the model parameters $\theta$; however, during training they are intended to act as auxiliary weights rather than optimization variables.

\begin{wrapfigure}{r}{0.4\textwidth}
    \vspace{-4mm}
    \centering
    \includegraphics[width=\linewidth]{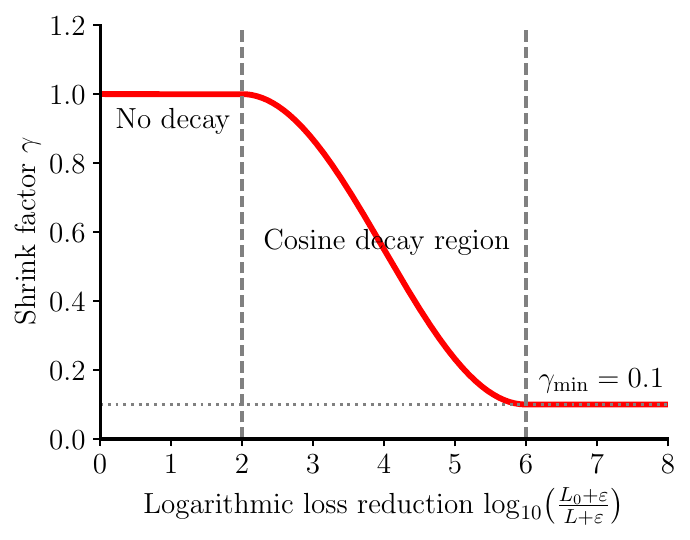}
    \vspace{-6mm}
 \caption{Illustration of the shrink factor used in adaptive pseudo-time stepping with $s_\text{start}=2, s_\text{end}=6$, $\gamma_{\text{min}}=0.1$.  }
 \label{fig:decay}
 \vspace{-10mm}
\end{wrapfigure}
In addition, one may note that we further introduce a shrink factor to modulate the adaptive pseudo-time step \(\tau^k\). Specifically, \(\gamma^k \in [\gamma_{\min},1]\) is defined through a cosine-decay schedule based on the relative reduction of the training PDE residual loss:
\begin{align}
    \gamma^k
    :=
    \gamma_{\min}
    +
    (1-\gamma_{\min})
    \frac{1+\cos(\pi p^k)}{2},
\end{align}
where \(p^k \in [0,1]\) denotes the normalized logarithmic reduction of the interior residual loss \(L_{\mathrm{int}}^k := \mathcal{L}_{\mathrm{int}}(\theta^k;X_{\mathrm{int}}^k)\) relative to its initial value \(L_{\mathrm{int}}^0\):
\begin{align}
    p^k
    :=
    \frac{
    \log_{10}\!\bigl((L_{\mathrm{int}}^0+\varepsilon)/(L_{\mathrm{int}}^k+\varepsilon)\bigr)
    -
    s_{\mathrm{start}}
    }{
    s_{\mathrm{end}}-s_{\mathrm{start}}
    }.
\end{align}
Here, $s_{\text {start }}$ and $s_{\text {end }}$ are two thresholds that control when the cosine decay starts and stops. They are defined in terms of the logarithmic reduction of the interior residual loss. Specifically, $s_{\text {start }}$ is the point at which the shrink factor begins to decrease from 1, and $s_{\text {end }}$ is the point at which it reaches its minimum value. Therefore, the interval $\left[s_{\text {start }}, s_{\text {end }}\right]$ defines the stage of training over which the decay is active, as illustrated in Fig.~\ref{fig:decay}.
This design is motivated by our empirical observation that, in the later stages of training, an excessively large \(\tau^k\) can destabilize the optimization process. As the learning rate becomes small and the training dynamics become increasingly stiff, a single gradient step may no longer be sufficient to accurately realize the desired pseudo-time stepping update. The shrink factor therefore serves as an additional stabilization mechanism by progressively reducing the effective pseudo-time step as training proceeds.

It is worth noting that conventional pseudo-time stepping in PINNs typically uses a single fixed step size for all PDE residual terms. However, when the loss function contains several such terms, it is natural to assign a separate pseudo-time step size to each relaxed equation. For example, in the lid-driven cavity problem, we introduce component-wise pseudo-time step sizes
\[
\boldsymbol{\tau}^k
=
\bigl(\tau_u^k,\tau_v^k,\tau_p^k\bigr),
\]
corresponding to the \(u\)-momentum, \(v\)-momentum, and continuity equations, respectively. Specifically, we define
\begin{align}
{\tau}_u^{k}
&:=
\frac{\left\|u_{\theta^{k}}-u_{\theta^{k-1}}\right\|_{2}}
{\left\|\mathcal{R}_{u}[\mathbf{u}_{\theta^{k}}]-\mathcal{R}_{u}[\mathbf{u}_{\theta^{k-1}}]\right\|_{2}+\varepsilon},\\
{\tau}_v^{k}
&:=
\frac{\left\|v_{\theta^{k}}-v_{\theta^{k-1}}\right\|_{2}}
{\left\|\mathcal{R}_{v}[\mathbf{u}_{\theta^{k}}]-\mathcal{R}_{v}[\mathbf{u}_{\theta^{k-1}}]\right\|_{2}+\varepsilon},\\
{\tau}_p^{k}
&:=
\frac{\left\|p_{\theta^{k}}-p_{\theta^{k-1}}\right\|_{2}}
{\left\|\mathcal{R}_{c}[\mathbf{u}_{\theta^{k}}]-\mathcal{R}_{c}[\mathbf{u}_{\theta^{k-1}}]\right\|_{2}+\varepsilon}.
\end{align}
This component-wise relaxation provides additional flexibility, as it allows each equation to adapt according to its own scale and local optimization dynamics. 


\begin{figure}[h]
    \centering
    \includegraphics[width=1.0\linewidth]{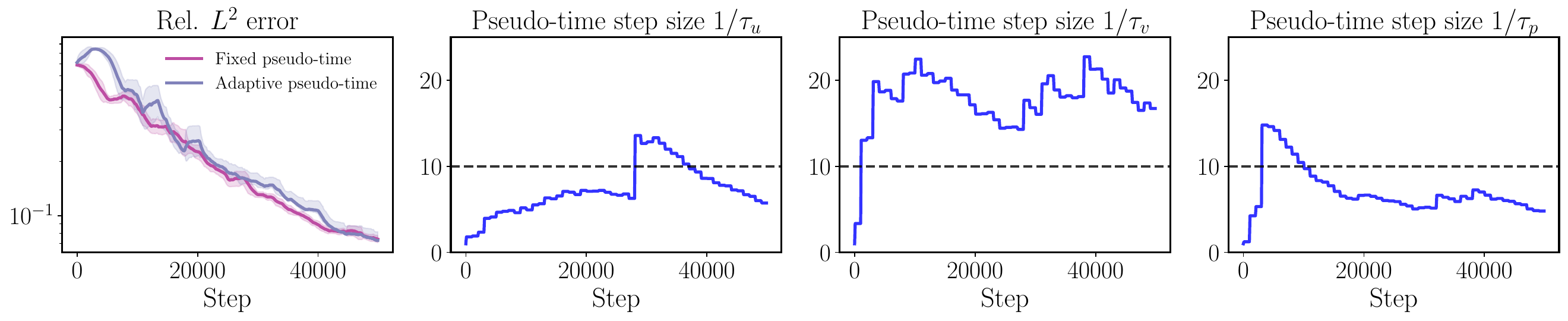}
\caption{{\em Lid-driven cavity}: Comparison of adaptive pseudo-time stepping with fixed pseudo-time stepping using the best tuned shared step size, \(1/\tau = 10\). From left to right: relative \(L^2\) error versus training iterations, and the evolution of the adaptive pseudo-time step sizes \(\tau_u\), \(\tau_v\), and \(\tau_p\). In the last three panels, the adaptive step sizes are shown together with the reference fixed-step value.}
    \label{fig:ldc_method_tau}
\end{figure}

In Figure~\ref{fig:ldc_method_tau}, we compare conventional pseudo-time stepping with a fixed shared step size, using the best-tuned value for this example ($1 / \tau=10$), against the proposed adaptive variant. The adaptive variant achieves error convergence comparable to that of the optimally tuned fixed-step method. In addition, starting from their initial value of 1, the adaptive pseudo-time step sizes remain roughly close to the reference value throughout training. In the Results section, we further assess the effectiveness of the adaptive strategy on a wider range of problems.

\section{Results}
\label{sec:results}

In this section, we aim to evaluate the effectiveness of the proposed adaptive pseudo-time stepping scheme in PINNs for 
solving forward PDE problems. To this end,  we consider a diverse set of 10 representative and challenging benchmarks arising from fundamental physical phenomena. These examples span shock formation, chaotic dynamics, reaction--diffusion systems, fluid mechanics, and heat transfer. Detailed descriptions of all problem setups, including PDE parameters, initial and boundary conditions, numerical implementations are provided in Appendix~\ref{appendix:experiments}.

Here, we view PINNs as an integrated computational framework whose performance is governed by the interaction of multiple design choices. In practice, training failures often arise from the coupling among these components, and no single modification is sufficient to resolve all difficulties. Therefore, to fairly assess the value of a particular method, it is more meaningful to compare it against a strong baseline that already incorporates established improvements. Following this principle, we construct a competitive baseline by integrating several techniques that have been widely studied and shown to improve PINN training. Specifically, our baseline combines improvements in architecture design, loss formulation, and optimization:
\begin{itemize}[leftmargin=1.2em,itemsep=0.3em,topsep=0.3em]
\item \textbf{PirateNet \cite{wang2024piratenets}:} a physics-informed residual adaptive network designed to improve the training of deep PINNs by mitigating spectral bias and better capturing sharp gradients and transition regions;

    \item \textbf{Causal training \cite{wang2024respecting}:} a PDE residual weighting strategy that respects temporal causality and improves robustness for time-dependent problems;
    
    \item \textbf{Self-adaptive loss weighting \cite{wang2022and,wang2021understanding}:} a global loss weighting scheme for balancing gradient contributions from loss terms with different scales;
    
    \item \textbf{SOAP \cite{vyas2024soap,wang2025gradient}:} a quasi-second-order optimization method that alleviates directional gradient conflicts among multiple loss terms.
\end{itemize}

We believe this baseline represents one of the most accurate and scalable PINN pipelines currently available. It has achieved state-of-the-art performance across a broad range of PDE benchmarks \cite{wang2023expert,wang2025gradient} and has also enabled successful simulations of three-dimensional turbulent flows \cite{wang2025simulating}. A comprehensive description of the techniques incorporated into this baseline  is provided in Appendices~\ref{appendix:arch}, \ref{appendix:optim},  and \ref{appendix:weighitng}, respectively.

Unless otherwise stated, we use the same training protocol and hyperparameter settings for all experiments. Our goal is to evaluate the methods under a robust and unified baseline while minimizing problem-specific tuning. This is particularly important for forward PDE problems, where the reference solution is unavailable during training and hyperparameter selection cannot rely on test-time accuracy. The resulting setup is given below.

We use PirateNet \cite{wang2024piratenets} as the backbone architecture, with three residual blocks (nine layers in total), a hidden width of 256, and \texttt{Tanh} activation functions. When applicable, we impose exact periodic boundary conditions following \cite{dong2021method}.
We train all models for $10^5$ iterations using mini-batch gradient descent, with 4,096 randomly sampled collocation points per iteration for 2D problems and 8,192 for 3D problems. The learning rate schedule begins with a linear warm-up over the first 2,000 iterations, increasing from 0 to $10^{-3}$, followed by exponential decay with a factor of 0.9 every 2,000 iterations. To further improve stability and convergence, we employ learning-rate annealing for loss balancing \cite{wang2021understanding,wang2023expert}, updating the loss weights every 1,000 iterations using a moving average. For time-dependent PDEs, we additionally apply causal training \cite{wang2024respecting,wang2023expert} with a causal tolerance of 1.0. Optimization is performed using the SOAP optimizer \cite{vyas2024soap,wang2025gradient} with $\beta_1=0.9$ and $\beta_2=0.999$, while the covariance matrix is updated every two iterations. The full set of hyperparameters is reported in Table~\ref{tab:hyper-parameters}, and largely follows the configurations established in \cite{wang2023expert,wang2024piratenets,wang2025gradient}.

To ensure a fair comparison, we evaluate both conventional pseudo-time stepping and the proposed adaptive pseudo-time stepping on top of the constructed baseline using identical hyperparameter settings in all experiments. For conventional pseudo-time stepping, we report the best result obtained over step sizes spanning several orders of magnitude, \(\tau \in \{0.01, 0.1, 1, 10, 100\}\). We emphasize, however, that this tuning is performed only for benchmarking purposes, where the reference solution is available. In practical settings, where the ground-truth solution is unknown, such tuning is generally not feasible. For the adaptive variant, we initialize all step sizes to 1 and  update the step size every 1{,}000 iterations  unless otherwise stated. Moreover, 
we enable step size shrinkage with \(s_{\mathrm{start}} = 2\), \(s_{\mathrm{end}} = 6\), and \(\gamma_{\min} = 0.1\). This choice is empirical and is based on our observation that, with the constructed baseline, the PINN loss typically decreases by about six orders of magnitude from its initial value during training.

\begin{table}[t]

    \vspace{-1mm}
    \centering
    \renewcommand{\arraystretch}{1.2}
    \caption{Comparison of the baseline PINN and the same baseline augmented with either fixed or adaptive pseudo-time stepping. The evaluation metric is the relative \(L^2\) error over the full space--time domain. For fixed pseudo-time stepping, we report the best result over \(\tau \in \{0.01, 0.1, 1, 10, 100\}\), together with the corresponding \(\tau_{\text{best}}\).}
    \label{tab:sota}
    \resizebox{\textwidth}{!}{
    \begin{tabular}{lcccc}
        \toprule
        \multirow{2}{*}{\textbf{Benchmark}}
        & \multirow{2}{*}{\textbf{Baseline}}
        & \multicolumn{2}{c}{\textbf{Fixed pseudo-time}}
        & \multirow{2}{*}{\textbf{Adaptive pseudo-time}} \\
        \cmidrule(lr){3-4}
        & & \textbf{Rel. \(L^2\) error} & \(\boldsymbol{\tau_{\textbf{best}}}\) & \\
        \midrule
        Allen--Cahn  
        & \(5.17 \times 10^{-6}\)
        & \(3.26 \times 10^{-6}\)
        & \({1}\)
        & \(\mathbf{3.05\times 10^{-6}}\) \\
        
        Korteweg–De Vries & \(3.55 \times 10^{-4}\)& \(2.92 \times 10^{-4}\)& 1& \(\mathbf{2.28\times 10^{-4}}\) \\
        
        Inviscid Burgers
        & \(2.13 \times 10^{-1}\)
        & \(1.01 \times 10^{-1}\)
        & \({10}\)
        & \(\mathbf{6.47\times 10^{-2}}\) \\
        Kuramoto--Sivashinsky
        & \(1.06 \times 10^{-1}\)
        & \(3.19 \times 10^{-2}\)
        & \({10}\)
        & \(\mathbf{1.11\times 10^{-2}}\) \\

        Gray--Scott
        & \(4.14 \times 10^{-1}\)
        & \(1.07 \times 10^{-1}\)
        & \({100}\)
        & \(\mathbf{1.52 \times 10^{-2}}\) \\

        Ginzburg--Landau
        & \(1.74 \times 10^{-1}\)
        & \(5.46 \times 10^{-2}\)
        & \({1}\)
        & \(\mathbf{7.75 \times 10^{-3}}\) \\

        Lid-driven cavity \((\mathrm{Re}=5 \times 10^3)\)
        & \(6.33 \times 10^{-1}\)
        & \(4.93 \times 10^{-2}\)
        & \({10}\)
        & \(\mathbf{4.13 \times 10^{-2}}\) \\

        Backward-facing step \((\mathrm{Re}=8\times 10^2)\)
        & \(3.02 \times 10^{-1}\)
        & \(1.34 \times 10^{-2}\)
        & \({0.1}\)
        & \(\mathbf{2.32\times 10^{-3}}\) \\

        Kolmogorov flow \((\mathrm{Re}=10^4)\)
        & \(7.66 \times 10^{-1}\)
        & \(9.12 \times 10^{-2}\)
        & \({10}\)
        & \(\mathbf{7.05\times 10^{-2}}\) \\

        Rayleigh--Taylor instability \((\mathrm{Ra}=10^6)\)
        & \(3.98 \times 10^{-1}\)
        & \(5.61 \times 10^{-3}\)
        & \({1}\)
        & \(\mathbf{3.82\times 10^{-3}}\) \\
        \bottomrule
    \end{tabular}
    }
\end{table}

\begin{figure}[t]
    \centering

    \begin{subfigure}[t]{1.0\linewidth}
        \centering
        \includegraphics[width=\linewidth]{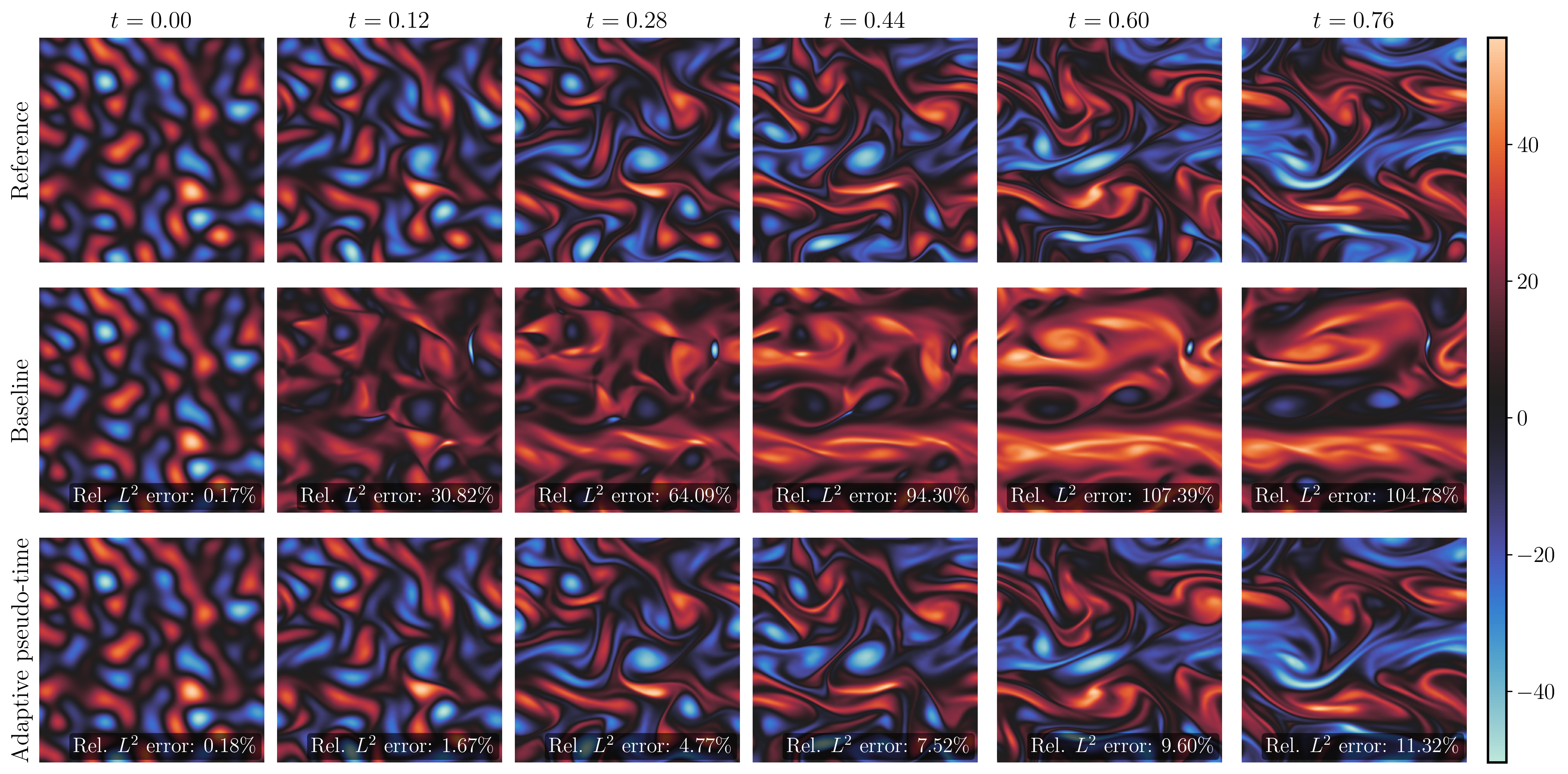}
    \end{subfigure}

    \vspace{0.1em}

    \begin{subfigure}[t]{1.0\linewidth}
        \centering
        \includegraphics[width=\linewidth]{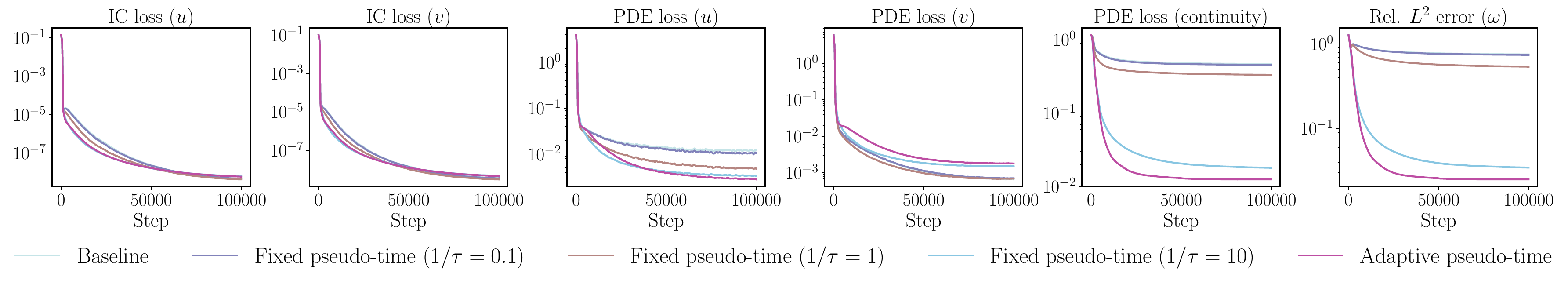}
    \end{subfigure}

    \caption{{\em Kolmogorov flow.}  Top: Comparison of the reference solution, the baseline PINN, and the baseline PINN with adaptive pseudo-time stepping. Bottom: training loss and relative \(L^2\) error histories.}
    \label{fig:kf_results}
\end{figure}

\paragraph{Main results.}
Table~\ref{tab:sota} reports the relative \(L^2\) error between the ground-truth PDE solutions and the corresponding PINN predictions across a diverse set of benchmark problems. Overall, our baseline PINN equipped with adaptive pseudo-time stepping achieves the best performance in every case, outperforming both the strong baseline and fixed pseudo-time stepping with carefully tuned \(\tau\). These quantitative gains are further supported by the visual results in Fig.~\ref{fig:allen_cahn_results} to Fig.~\ref{fig:rt_results}. We can observe that the predicted solutions by adaptive pseudo-time stepping are in good agreement with the reference numerical solutions and yield consistently better error convergence.

A key observation is that, for conventional pseudo-time stepping, the optimal choice of \(\tau\) varies substantially across problems, often spanning several orders of magnitude. This highlights the practical difficulty of selecting \(\tau\). We emphasize that the fixed \(\tau\) values reported in Table~\ref{tab:sota} are included only for benchmarking purposes. While better values may exist, identifying them would require more extensive hyperparameter tuning. More critically, they cannot be identified reliably from the training losses alone, making fixed pseudo-time stepping difficult to use in realistic applications. This point is further illustrated by the inviscid Burgers and Gray--Scott examples (Figs.\ref{fig:gs_results} and ~\ref{fig:inviscid_burgers_results}), where the best predictive performance is obtained with choices of \(\tau\) that are associated with noticeably larger training losses. Therefore, the training loss is not a reliable criterion for tuning \(\tau\), even when exhaustive benchmarking is feasible.

Another important observation is that the loss convergence provides further evidence for our interpretation of pseudo-time stepping. As we argued in Section~\ref{sec:pseduo_time}, its main benefit does not appear to come from improved conditioning alone, but rather from residual amplification under collocation-point resampling, which helps reveal and suppress spurious solutions. This point is illustrated by the inviscid Burgers, KS, and backward-facing step examples (Fig~\ref{fig:gs_results}, ~\ref{fig:inviscid_burgers_results},  and \ref{fig:bfs}), where the pseudo-time stepping produces significantly lower test errors even though its training losses are comparable to, and in some cases higher than, those of the competing methods. 

We also observe the previously identified failure mode in the Kolmogorov flow problem (Fig.~\ref{fig:kf_results}) and the Rayleigh--Taylor instability problem (Fig.~\ref{fig:rt_results}). In both cases, the baseline PINN performs poorly. More importantly, after some time, its prediction becomes nearly stationary and no longer evolves with time, suggesting that the model has converged to a nonphysical spurious solution rather than capturing the true dynamics. For the boundary-value problem of backward-facing step flow, the failure takes a different form: the baseline PINN learns a flow that is largely unidirectional and gradually decays toward the outlet. Taken together, these results provide further evidence that the PINN converges to the spurious solutions discussed in the  Section \ref{sec: failure_mode}.

Finally, we note that pseudo-time stepping is most beneficial for problems in which PINN models are susceptible to converging to spurious solutions, a phenomenon that typically arises in PDEs with sharp transition regions or long-time evolution. In contrast, when the PINN training already performs well, the improvement brought by pseudo-time stepping is naturally more limited. This is consistent with our observations on the Allen--Cahn and wave equation benchmarks, where the gains are less pronounced. We attribute this to the fact that these problems have been studied extensively and that our baseline already reaches accuracy close to the \texttt{float32} precision limit used in training. Nevertheless, Fig.~\ref{fig:allen_cahn_results} shows that pseudo-time stepping does not hurt training and still yields faster convergence during the early stages.

\paragraph{Effect of pseudo-time step size update frequency.}
We also note that the proposed adaptive pseudo-time stepping exhibits some sensitivity to the update frequency, and we report this effect for completeness. To investigate this, we consider four benchmark problems and train PINN models with adaptive pseudo-time stepping using different update frequencies, with each setting repeated over five random seeds. The resulting relative \(L^2\) errors are summarized in the left panel of Fig.~\ref{fig:ablation_freq_shrink}. We observe that more frequent updates generally lead to worse predictive accuracy. We attribute this behavior to the fact that overly frequent changes in the pseudo-time step sizes can destabilize the training process. In contrast, updating the step size at a moderate frequency provides a better balance between adaptivity and stability, leading to more reliable performance across different runs. Based on this observation, we update the pseudo-time step size \(\tau\) every 1{,}000 iterations. An additional advantage of this choice is that its computational overhead is negligible relative to the overall training cost.

\begin{figure}
    \centering
    \includegraphics[width=1.0\linewidth]{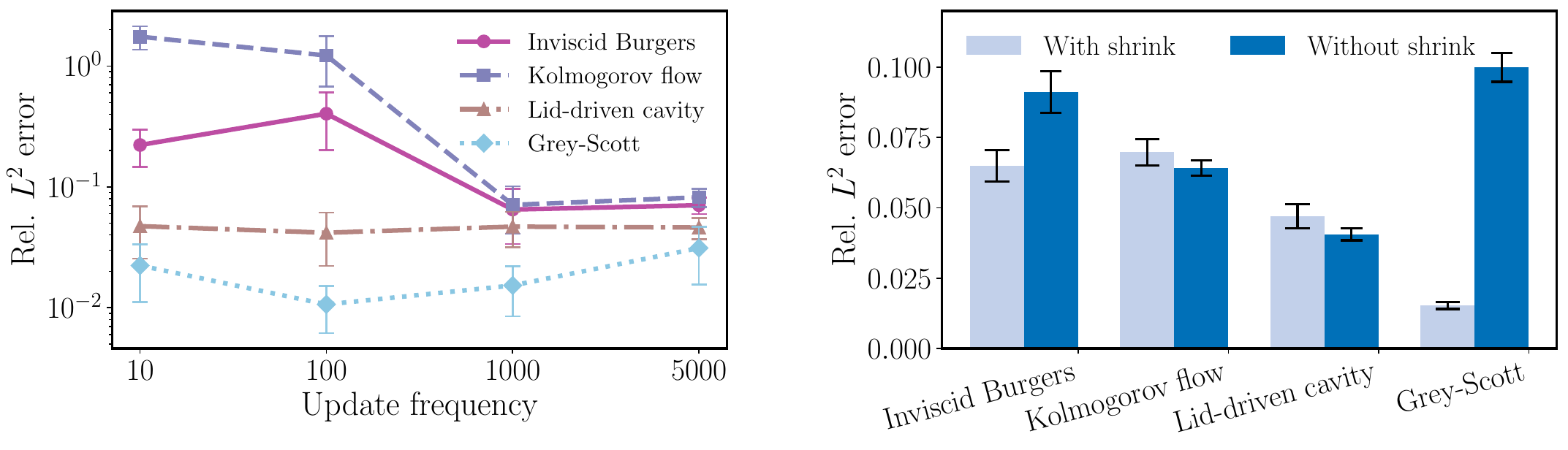}
    \caption{Ablation studies of the update frequency (left) and step size shrinkage (right) in adaptive pseudo-time stepping. The relative \(L^2\) errors are averaged over five random seeds for four benchmark problems. }
    \label{fig:ablation_freq_shrink}
\end{figure}

\paragraph{Effect of the shrink factor in adaptive pseudo-time stepping.}
Another important ablation investigates the proposed step size shrinkage during training. To assess its effect, we compare adaptive pseudo-time stepping with and without shrinkage on the same four benchmark problems, and report the relative \(L^2\) error averaged over five random seeds. The results are shown in the right panel of Fig.~\ref{fig:ablation_freq_shrink}. Overall, the effect is not entirely uniform across problems. However, disabling shrinkage can lead to substantially worse performance in some cases, most notably for the Gray--Scott equation. Since our goal is to develop an adaptive pseudo-time stepping strategy that is robust across a wide range of problems, we enable shrinkage in all experiments. This choice leads to more reliable training behavior overall, even if its benefit is not equally evident on every benchmark.

\begin{figure}
    \centering
       \includegraphics[width=0.48\linewidth]{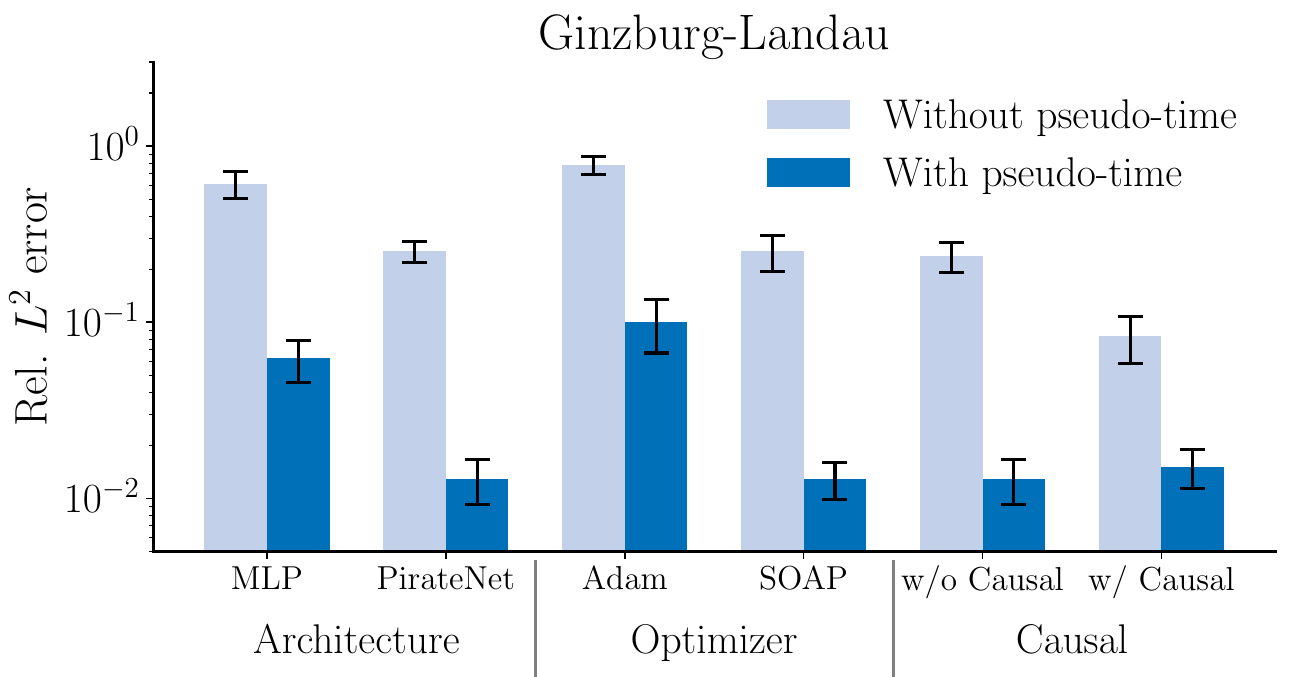}
    \includegraphics[width=0.48\linewidth]{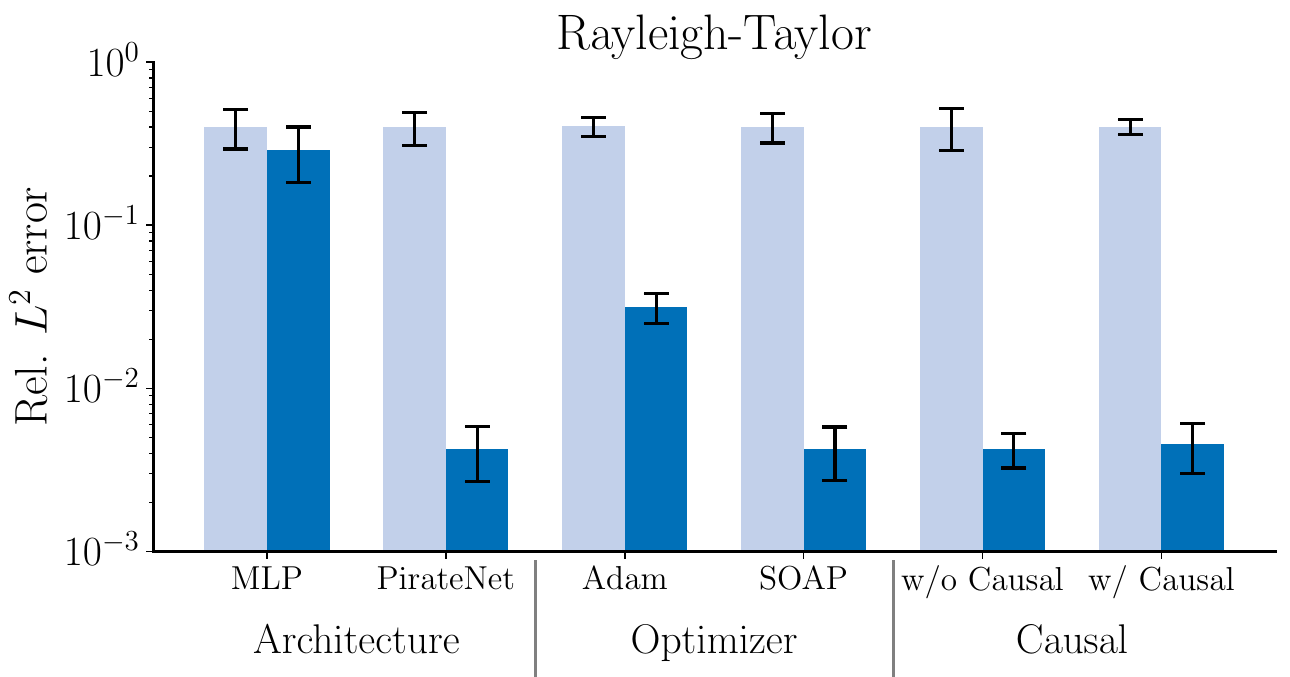}
 \caption{Ablation studies of the baseline and robustness of pseudo-time stepping on the Ginzburg--Landau equation (left) and the Rayleigh--Taylor instability (right). Relative \(L^2\) errors are reported for different baseline variants, each obtained by removing one key component from the full baseline while keeping the others unchanged. For each variant, results are shown both with and without pseudo-time stepping, and are averaged over five random seeds.}
    \label{fig:ablation}
\end{figure}

\paragraph{Robustness of adaptive pseudo-time stepping.}
As a final ablation study, we examine both the role of individual components in the baseline and the robustness of pseudo-time stepping when combined with different design choices. To this end, we consider two representative benchmark problems, namely the Ginzburg--Landau equation and the Rayleigh--Taylor instability. For each problem, we start from the full baseline and disable one key component at a time while keeping all other components unchanged. In particular, we consider ablations of the PirateNet architecture, the SOAP optimizer, and causal training, and for each resulting variant we evaluate the model both with and without pseudo-time stepping.  The relative \(L^2\) errors, averaged over five random seeds, are reported in Fig.~\ref{fig:ablation}. Overall, pseudo-time stepping improves performance consistently across all settings, regardless of the choice of architecture, optimizer, or causal training. In addition, replacing Adam with SOAP and MLP with PirateNet generally improves accuracy. However, without pseudo-time stepping, these challenging problems still yield errors above 
10 \%, even with the stronger baseline components.
By contrast, causal training does not significantly improve the final accuracy in these examples, although in our experience it mainly helps make training more robust. 
We also observe that the network architecture is particularly important for the Rayleigh--Taylor example: using an MLP, with or without pseudo-time stepping, leads to similarly poor performance.

\paragraph{Computational costs.}
We summarize the computational overhead of pseudo-time stepping in Table~\ref{tab:cost}, which reports the wall-clock training time per 100 iterations for the baseline PINN and the same baseline augmented with pseudo-time stepping on a single H200 GPU. Overall, the additional cost is modest across most benchmarks. This is because pseudo-time stepping only requires an extra evaluation of the network outputs using the parameters from the previous iteration, which are treated as constants and therefore do not contribute to backpropagation. These results suggest that pseudo-time stepping can substantially improve training behavior while introducing only a small computational overhead in most cases. For the proposed adaptive pseudo-time stepping, the cost of updating the step size is negligible relative to the total training cost, since it is performed only once every few thousand iterations.

\begin{table}[h]
    \centering
    \renewcommand{\arraystretch}{1.2}
    \caption{Computational cost of the baseline PINN and the same baseline augmented with fixed pseudo-time stepping. We report the wall-clock training time per 100 iterations (in seconds), measured on a single H200 GPU, as mean $\pm$ standard deviation.}
    \label{tab:cost}
    \begin{tabular}{lcc}
        \toprule
        \multirow{2}{*}{\textbf{Benchmark}} & \multicolumn{2}{c}{\textbf{Training time per 100 iterations (s)}} \\
        \cmidrule(lr){2-3}
        & \textbf{Baseline} & \textbf{Pseudo-time stepping} \\
        \midrule
        Allen--Cahn & $2.50 \pm 0.08$ & $2.63 \pm 0.07$ \\
        Korteweg--De Vries & $2.39 \pm 0.06$ & $2.42 \pm 0.08$ \\
        Inviscid Burgers & $1.07 \pm 0.03$ & $1.08 \pm 0.04$ \\
        Kuramoto--Sivashinsky & $3.03 \pm 0.11$ & $5.50 \pm 0.18$ \\
        Gray--Scott & $4.26 \pm 0.13$ & $4.56 \pm 0.17$ \\
        Ginzburg--Landau & $4.33 \pm 0.19$ & $4.67 \pm 0.16$ \\
        Lid-driven cavity \((\mathrm{Re}=5 \times 10^3)\) & $1.09 \pm 0.03$ & $1.13 \pm 0.04$ \\
        Backward-facing step \((\mathrm{Re}=8\times 10^2)\) & $1.39 \pm 0.05$ & $1.41 \pm 0.06$ \\
        Kolmogorov flow \((\mathrm{Re}=10^4)\) & $4.90 \pm 0.20$ & $5.02 \pm 0.17$ \\
        Rayleigh--Taylor instability \((\mathrm{Ra}=10^6)\) & $6.65 \pm 0.21$ & $6.81 \pm 0.27$ \\
        \bottomrule
    \end{tabular}
\end{table}

\section{Discussion}
\label{sec:discussion}


In this work, we identify a fundamental challenge in PINN training: the tendency to converge to spurious solutions. We argue that this failure mode originates from an inherent weakness in the loss formulation, namely that the PDE residual is enforced only at a finite set of collocation points, which can make the training objective effectively ill-posed. Through both theoretical analysis and empirical evidence, we show that this issue can be substantially mitigated by introducing pseudo-time stepping into the PDE residual loss. At the same time, we demonstrate that the effectiveness of pseudo-time stepping depends critically on the choice of the pseudo-time step $\tau$. To overcome this limitation, we propose an adaptive pseudo-time stepping strategy that automatically selects $\tau$ using a cheap estimate of the PDE operator’s local Jacobian stiffness. We validate the proposed method across a broad set of PDE benchmarks, and the results show that it consistently outperforms conventional pseudo-time stepping with manually tuned $\tau$, while also exhibiting strong robustness.

More broadly, we believe this work provides a clearer perspective on why PINNs fail and helps open new directions for improving their reliability. In particular, our analysis suggests that the spurious solutions are a central obstacle in physics-informed learning, and that understanding how to detect and avoid them is essential for building more robust methods. From a methodological standpoint, there is still considerable room to improve pseudo-time stepping, for example by incorporating longer optimization histories, introducing spatially or temporally local step sizes, or developing more principled adaptive strategies that better exploit the structure of the underlying PDE operator. Moreover, the proposed idea is not limited to standard PINNs, and can be extended naturally to the training of physics-informed neural operators, where related optimization difficulties may also arise. On the application side, our current study mainly validates the method on fluid-related problems; it would be equally interesting to investigate its effectiveness in computational solid mechanics, including elasticity, fracture, and other challenging settings. Overall, we hope this work serves as a useful step toward more robust and scalable physics-informed machine learning.

\section*{Acknowledgments}
This work was supported by the U.S. Department of Energy, Office of Science, Advanced Scientific Computing Research program, under Award No. DE-SC0024563. We also gratefully acknowledge support from the NVIDIA Academic Grant Program.

\bibliographystyle{unsrt}  
\bibliography{references}  

\appendix
\input{appendix}

\end{document}

%% file: appendix.tex
\clearpage
\section{Nomenclature}


\definecolor{tableblue}{RGB}{33,76,140}
\definecolor{tablebluebg}{RGB}{235,242,252}
\definecolor{headergray}{RGB}{245,247,250}

\newcolumntype{L}[1]{>{\RaggedRight\arraybackslash}p{#1}}
\newcommand{\notationgroup}[1]{%
\addlinespace[4pt]
\rowcolor{tablebluebg}
\multicolumn{3}{@{}l}{\textbf{\textcolor{tableblue}{#1}}}
\\[-2pt]
}

\begingroup
\setlength{\tabcolsep}{8pt}
\renewcommand{\arraystretch}{1.18}
\small

\begin{longtable}{@{} L{3.2cm} L{2.0cm} L{8.8cm} @{}}
\caption{{Notation used throughout the paper. }}
\label{tab:notation} \\

\toprule
\rowcolor{headergray}
\textbf{Symbol} & \textbf{Type} & \textbf{Description} \\
\midrule
\endfirsthead

\multicolumn{3}{@{}l}{\small\itshape Table \thetable\ continued from previous page.} \\
\toprule
\rowcolor{headergray}
\textbf{Symbol} & \textbf{Type} & \textbf{Description} \\
\midrule
\endhead

\midrule
\multicolumn{3}{r@{}}{\small\itshape Continued on next page.} \\
\endfoot

\bottomrule
\endlastfoot

\notationgroup{Domain and geometry}
$\Omega \subset \mathbb{R}^d$ & domain & Bounded spatial domain with regular boundary $\partial\Omega$ \\
$\partial\Omega$ & boundary & Boundary of the spatial domain \\
$[0,T]$ & time interval & Physical time horizon for parabolic or hyperbolic problems \\
$(t,\mathbf{x})$ & coordinates & Physical time and spatial coordinate \\
$d$ & integer & Spatial dimension \\

\notationgroup{PDE operators and data}
$\mathcal{D}[\,\cdot\,]$ & operator & Linear or nonlinear spatial differential operator \\
$\mathcal{B}[\,\cdot\,]$ & operator & Abstract boundary operator (e.g., Dirichlet, Neumann, Robin) \\
$\mathbf{f}(t,\mathbf{x})$ & forcing & Right-hand-side source term of the PDE \\
$\mathbf{g}(\mathbf{x})$ & initial data & Prescribed initial condition at $t=0$ \\
$\mathbf{u}^*(t,\mathbf{x})$ & solution & Exact solution of the PDE \\

\notationgroup{Neural network approximation and residuals}
$\mathbf{u}_\theta(t,\mathbf{x})$ & network output & Neural-network approximation of the solution parameterized by $\theta$ \\
$\theta$ & parameters & Trainable weights and biases of the network \\
$\mathcal{R}_{\mathrm{int}}[\mathbf{u}_\theta]$ & residual & Interior PDE residual, e.g.,
$\partial_t \mathbf{u}_\theta + \mathcal{D}[\mathbf{u}_\theta] - \mathbf{f}$ \\
$\mathcal{R}_{\mathrm{bc}}[\mathbf{u}_\theta]$ & residual & Boundary-condition residual on $\partial\Omega$ \\
$\mathcal{R}_{\mathrm{ic}}[\mathbf{u}_\theta]$ & residual & Initial-condition residual $\mathbf{u}_\theta(0,\mathbf{x}) - \mathbf{g}(\mathbf{x})$ \\

\notationgroup{Collocation sets and counts}
$X_{\mathrm{int}}$ & set & Interior collocation points in $[0,T]\times\Omega$, with cardinality $N_{\mathrm{int}}$ \\
$X_{\mathrm{bc}}$ & set & Boundary collocation points in $[0,T]\times\partial\Omega$, with cardinality $N_{\mathrm{bc}}$ \\
$X_{\mathrm{ic}}$ & set & Initial collocation points in $\Omega$, with cardinality $N_{\mathrm{ic}}$ \\
$N_{\mathrm{int}},N_{\mathrm{bc}},N_{\mathrm{ic}}$ & integers & Numbers of interior, boundary, and initial collocation points \\

\notationgroup{Loss functions}
$\mathcal{L}_{\mathrm{int}}(\theta;X_{\mathrm{int}})$ & scalar loss & Mean-squared PDE residual over interior collocation points \\
$\mathcal{L}_{\mathrm{bc}}(\theta;X_{\mathrm{bc}})$ & scalar loss & Mean-squared boundary-condition residual \\
$\mathcal{L}_{\mathrm{ic}}(\theta;X_{\mathrm{ic}})$ & scalar loss & Mean-squared initial-condition residual \\
$\mathcal{L}(\theta)$ & scalar loss & Total PINN loss:
$\lambda_{\mathrm{int}}\mathcal{L}_{\mathrm{int}}
+\lambda_{\mathrm{bc}}\mathcal{L}_{\mathrm{bc}}
+\lambda_{\mathrm{ic}}\mathcal{L}_{\mathrm{ic}}$ \\
$\lambda_{\mathrm{int}},\lambda_{\mathrm{bc}},\lambda_{\mathrm{ic}}>0$ & scalars & Positive weights assigned to each loss component \\

\notationgroup{Pseudo-time stepping}
$s$ & artificial time & Auxiliary variable parameterizing training dynamics \\
$\tau$ & step size &  Pseudo-time step size \\
$\mathcal{L}_{\mathrm{pts}}(\theta;\theta^{k-1})$ & scalar loss & Pseudo-time-relaxed residual loss evaluated at iterate $k$ \\
$\mathbf{u}^\dagger(t,\mathbf{x})$ & function & Spurious solution that satisfies the PDE at finite collocation points \\
$\mathbf{u}^{\dagger,+}$ & function & One pseudo-time update applied to $\mathbf{u}^\dagger$ \\
$\alpha_h(t)$ & cutoff & Smooth time cutoff with $\alpha_h(t)=1$ for $t\le t_0$ and $\alpha_h(t)=0$ for $t\ge t_0+h$ \\
$I_h=[t_0,t_0+h]$ & interval & Transition layer of width $h$ in the spurious-solution construction \\

\notationgroup{Adaptive pseudo-time stepping}
$\Delta\mathbf{u}^k$ & vector & Change in network output, $\mathbf{u}_{\theta^k}-\mathbf{u}_{\theta^{k-1}}$ \\
$\Delta\mathbf{r}^k$ & vector & Change in PDE residual between consecutive iterates \\
$\widehat{\tau}^k$ & scalar & Raw adaptive estimate:
$\gamma^k \|\Delta\mathbf{u}^k\|_2 / \bigl(\|\Delta\mathbf{r}^k\|_2 + \varepsilon\bigr)$ \\
$\gamma^k\in[\gamma_{\min},1]$ & shrink factor & Cosine-decay factor that reduces $\widehat{\tau}^k$ as training progresses \\
$\beta\in(0,1]$ & scalar & Exponential smoothing factor used to update $\tau^k$ \\
$\tau_{\min},\tau_{\max}$ & scalars & Lower and upper clipping bounds for the adaptive pseudo-time step \\
$s_{\mathrm{start}},s_{\mathrm{end}}$ & scalars & Decade thresholds controlling the cosine-decay schedule for $\gamma^k$ \\
$m\in\mathbb{N}$ & integer & Update frequency for $\tau^k$ during training \\
$J_k$ & Jacobian & Local Jacobian of $\mathcal{R}_{\mathrm{int}}$ evaluated at $\mathbf{u}_{\theta^k}$ \\

\notationgroup{Optimization}
$\eta>0$ & learning rate & Step size of the gradient-based optimizer \\
$g^k$ & gradient & Gradient $\nabla_\theta \mathcal{L}(\theta^k)$ used to update the parameters \\
$K$ & integer & Total number of training iterations \\

\end{longtable}
\endgroup

\section{Proofs}

\subsection{Proof of Theorem \ref{thm1}}
\label{app:proof1}

\begin{proof}
Let
\begin{align}
X_{\mathrm{int}}=\{(t^i_{\mathrm{int}},\mathbf{x}^i_{\mathrm{int}})\}_{i=1}^{N_{\mathrm{int}}}
\subset [0,T]\times\Omega
\end{align}
be the given finite set of interior collocation points. Define
\begin{align}
J_-&:=\{i\in\{1,\dots,N_{\mathrm{int}}\}: t^i_{\mathrm{int}}<t_0\},\\
J_+&:=\{i\in\{1,\dots,N_{\mathrm{int}}\}: t^i_{\mathrm{int}} \ge t_0\}.
\end{align}

We will construct a smooth time cutoff \(\alpha\in C^\infty(\mathbb{R})\) such that
\begin{equation}
\label{eq:alpha}
\alpha(t)=
\begin{cases}
1, & t=0,\\[2mm]
1, & \text{in a neighborhood of each } t^i_{\mathrm{int}} \text{ with } i \in J_-,\\[2mm]
0, & \text{in a neighborhood of each } t^i_{\mathrm{int}} \text{ with } i \in J_+,\\[2mm]
0, & \forall\, t\ge t_0.
\end{cases}
\end{equation}

Since the set \(\{t^i_{\mathrm{int}}\}_{i=1}^{N_{\mathrm{int}}}\) is finite, we can choose \(\delta>0\) such that
\begin{align}
[\,t^i_{\mathrm{int}} - \delta,\; t^i_{\mathrm{int}}+\delta\,]\subset (0,t_0)
\qquad \text{for all } i\in J_-.
\end{align}
Let
\begin{align}
K:=\bigcup_{i\in J_-}[\,t^i_{\mathrm{int}} -\delta,\; t^i_{\mathrm{int}} +\delta\,]\cup\{0\}.
\end{align}
Then \(K\) is a compact subset of \((-\infty,t_0)\).

Choose an open set \(U\subset \mathbb{R}\) such that
\begin{align}
K\subset U
\qquad\text{and}\qquad
\overline{U}\subset (-\infty,t_0).
\end{align}
Let \(\chi_U\) be the indicator function of \(U\), and let \(\rho\in C_c^\infty(\mathbb{R})\) be a standard mollifier satisfying
\begin{align}
\rho\ge 0,\qquad \operatorname{supp}\rho\subset (-1,1),\qquad \int_{\mathbb{R}}\rho(s)\,ds=1.
\end{align}
For \(\eta>0\), define
\begin{align}
\rho_\eta(s)&:=\frac{1}{\eta}\rho\!\left(\frac{s}{\eta}\right),\\
\alpha&:=\rho_\eta * \chi_U.
\end{align}
For sufficiently small \(\eta>0\), the function \(\alpha\in C^\infty(\mathbb{R})\) satisfies
\begin{align}
0\le \alpha\le 1,
\end{align}
and, because \(K\) lies a positive distance away from \(U^c\), we have \(\alpha=1\) on \(K\). Also, since \(\overline{U}\subset (-\infty,t_0)\), we have \(\alpha(t)=0\) for all \(t\ge t_0\). Hence \eqref{eq:alpha} holds.

Now define
\begin{align}
\mathbf{u}^\dagger(t,\mathbf{x}) := \alpha(t)\,\mathbf{u}^\ast(t,\mathbf{x}).
\end{align}
Since \(\alpha\in C^\infty(\mathbb{R})\) and \(\mathbf{u}^\ast\) is classical, it follows that \(\mathbf{u}^\dagger\) is classical on \([0,T]\times\Omega\).

We verify the three claims.

\medskip
\noindent\textbf{(i) Initial and boundary conditions.}
Since $\alpha(0) = 1$,
\begin{align}
\mathbf{u}^\dagger(0,\mathbf{x})
&=
\alpha(0)\,\mathbf{u}^\ast(0,\mathbf{x})
=
\mathbf{u}^\ast(0,\mathbf{x})
=
\mathbf{g}(\mathbf{x}).
\end{align}
Moreover, on \([0,T]\times\partial\Omega\),
\begin{align}
\mathbf{u}^\dagger(t,\mathbf{x})
=
\alpha(t)\,\mathbf{u}^\ast(t,\mathbf{x})
=
0,
\end{align}
because \(\mathbf{u}^\ast\) satisfies the homogeneous Dirichlet boundary condition.

\medskip
\noindent\textbf{(ii) Triviality for \(t\ge t_0\).}
By \eqref{eq:alpha},
\begin{align}
\mathbf{u}^\dagger(t,\mathbf{x})
=
\alpha(t)\,\mathbf{u}^\ast(t,\mathbf{x})
=
0,
\qquad \forall\, t\ge t_0,\ \mathbf{x}\in\Omega.
\end{align}

\medskip
\noindent\textbf{(iii) Vanishing empirical residual.}
Fix any collocation point \((t_j,\mathbf{x}_j)\in X_{\mathrm{int}}\).

If \(j\in J_-\), then by \eqref{eq:alpha}, \(\alpha\equiv 1\) in a neighborhood of \(t_j\). Hence \(\mathbf{u}^\dagger=\mathbf{u}^\ast\) in a neighborhood of \((t_j,\mathbf{x}_j)\), and therefore
\begin{align}
\mathcal{R}_{\mathrm{int}}[\mathbf{u}^\dagger](t_j,\mathbf{x}_j)
=
\mathcal{R}_{\mathrm{int}}[\mathbf{u}^\ast](t_j,\mathbf{x}_j)
=
0,
\end{align}
since \(\mathbf{u}^\ast\) is a classical solution of the homogeneous PDE.

If \(j\in J_+\), then by \eqref{eq:alpha}, \(\alpha\equiv 0\) in a neighborhood of \(t_j\). Hence \(\mathbf{u}^\dagger\equiv 0\) in a neighborhood of \((t_j,\mathbf{x}_j)\). Since the PDE is homogeneous (\(\mathbf{f}=0\)), the zero function has zero interior residual, and thus
\begin{align}
\mathcal{R}_{\mathrm{int}}[\mathbf{u}^\dagger](t_j,\mathbf{x}_j)=0.
\end{align}

Therefore,
\begin{align}
\mathcal{R}_{\mathrm{int}}[\mathbf{u}^\dagger](t_j,\mathbf{x}_j)=0,
\qquad j=1,\dots,N_{\mathrm{int}},
\end{align}
and so
\begin{align}
\mathcal{L}_{\mathrm{int}}(\mathbf{u}^\dagger;X_{\mathrm{int}})
&=
\frac{1}{N_{\mathrm{int}}}\sum_{j=1}^{N_{\mathrm{int}}}
\bigl|\mathcal{R}_{\mathrm{int}}[\mathbf{u}^\dagger](t_j,\mathbf{x}_j)\bigr|^2
=
0.
\end{align}

This proves the theorem.
\end{proof}

\subsection{Proof of Theorem \ref{thm:unstable}}
\label{app:proof2}

\begin{proof}

We first recall a standard construction of the cutoff \(\alpha_h\).

\medskip
\noindent
\textbf{Step 1: Standard smooth cutoff construction.}
Define
\begin{align}
\phi(s):=
\begin{cases}
e^{-1/s}, & s>0,\\[2mm]
0, & s\le 0.
\end{cases}
\end{align}
Then \(\phi\in C^\infty(\mathbb R)\), \(\phi(s)>0\) for \(s>0\), and all derivatives of \(\phi\) vanish at \(s=0\). Now set
\begin{align}
\eta(s):=
\frac{\phi(1-s)}{\phi(s)+\phi(1-s)}.
\end{align}
Then \(\eta\in C^\infty(\mathbb R)\) and satisfies
\begin{align}
\eta(s)=1 \quad\text{for } s\le 0,
\qquad
\eta(s)=0 \quad\text{for } s\ge 1,
\qquad
0\le \eta(s)\le 1.
\end{align}
Since \(\eta\) is smooth and constant outside \([0,1]\), all of its derivatives are bounded, and moreover
\begin{align}
\eta'(0)=\eta'(1)=0.
\end{align}
For the given \(t_0\) and \(h>0\), define
\begin{align}
\alpha_h(t):=\eta\!\left(\frac{t-t_0}{h}\right).
\end{align}
Then \(\alpha_h\in C^\infty(\mathbb R)\) and
\begin{align}
\alpha_h(t)=1 \quad\text{for } t\le t_0,
\qquad
\alpha_h(t)=0 \quad\text{for } t\ge t_0+h.
\end{align}
By the chain rule,
\begin{align}
\alpha_h^{(k)}(t)=h^{-k}\eta^{(k)}\!\left(\frac{t-t_0}{h}\right).
\end{align}
Recall that the transition layer is defined by
\begin{align}
I_h:=[t_0,t_0+h],
\end{align}
so in particular
\begin{align}
\operatorname{supp}(\alpha_h')\subset I_h,
\qquad
\operatorname{supp}(\alpha_h'')\subset I_h.
\end{align}

\medskip
\noindent
\textbf{Step 2: Expectation as a space-time average.}
Since the fresh collocation points
\((\tilde t_j,\tilde{\mathbf x}_j)\), \(j=1,\dots,M\),
are i.i.d.\ and uniformly distributed on \((0,T)\times\Omega\), for any measurable function \(F\) on \((0,T)\times\Omega\),
\begin{align}
\mathbb E\!\left[\frac1M\sum_{j=1}^M F(\tilde t_j,\tilde{\mathbf x}_j)\right]
=
\frac{1}{T|\Omega|}
\int_0^T \int_\Omega F(t,\mathbf x)\,d\mathbf x\,dt.
\end{align}
Therefore, for any sufficiently regular \(v\),
\begin{align}
\mathbb E\,\mathcal L_{\mathrm{int}}^{\mathrm{new}}(v)
=
\frac{1}{T|\Omega|}
\int_0^T \int_\Omega
\left|\mathcal R_{\mathrm{int}}[v](t,\mathbf x)\right|^2
\,d\mathbf x\,dt.
\end{align}

\medskip
\noindent
\textbf{Step 3: Residual of the spurious solution.}
Since \(\mathbf{u}^\ast\) is a classical solution of \eqref{eq:problem},
\begin{align}
\mathcal R_{\mathrm{int}}[\mathbf{u}^\ast]=0.
\end{align}
Because \(\mathcal D\) is linear and acts only on the spatial variable \(\mathbf x\), while \(\alpha_h\) depends only on \(t\),
\begin{align}
\mathcal D[\alpha_h \mathbf{u}^\ast]=\alpha_h\,\mathcal D[\mathbf{u}^\ast].
\end{align}
Hence
\begin{align}
\mathcal R_{\mathrm{int}}[\mathbf{u}^\dagger]
&=
\partial_t(\alpha_h \mathbf{u}^\ast)+\mathcal D[\alpha_h \mathbf{u}^\ast] \notag\\
&=
\alpha_h' \mathbf{u}^\ast+\alpha_h \partial_t \mathbf{u}^\ast+\alpha_h \mathcal D[\mathbf{u}^\ast] \notag\\
&=
\alpha_h' \mathbf{u}^\ast+\alpha_h\,\mathcal R_{\mathrm{int}}[\mathbf{u}^\ast] \notag\\
&=
\alpha_h' \mathbf{u}^\ast,
\end{align}
so \(\mathcal R_{\mathrm{int}}[\mathbf{u}^\dagger]\) is supported in \(I_h\times\Omega\). Let
\begin{align}
G(t):=\int_\Omega |\mathbf{u}^\ast(t,\mathbf x)|^2\,d\mathbf x.
\end{align}
Since \(\mathbf{u}^\ast\) is classical, \(G\) is continuous on \([0,T]\). Moreover, since \(\mathbf{u}^\ast\) is continuous on the compact set \([0,T]\times\overline\Omega\), there exists
\begin{align}
K:=\|\mathbf{u}^\ast\|_{L^\infty((0,T)\times\Omega)}<\infty,
\end{align}
and hence \(0\le G(t)\le |\Omega|K^2\) for all \(t\in[0,T]\). Using Step 2,
\begin{align}
\mathbb E\,\mathcal L_{\mathrm{int}}^{\mathrm{new}}(\mathbf{u}^\dagger)
=
\frac{1}{T|\Omega|}
\int_{I_h} |\alpha_h'(t)|^2\,G(t)\,dt.
\end{align}
Now
\begin{align}
\int_{I_h} |\alpha_h'(t)|^2\,dt
=
\int_{t_0}^{t_0+h}
h^{-2}\left|\eta'\!\left(\frac{t-t_0}{h}\right)\right|^2 dt
=
h^{-1}\int_0^1 |\eta'(s)|^2\,ds,
\end{align}
and therefore
\begin{align}
\mathbb E\,\mathcal L_{\mathrm{int}}^{\mathrm{new}}(\mathbf{u}^\dagger)
\le
\frac{K^2}{T}
\int_{I_h} |\alpha_h'(t)|^2\,dt
=
\frac{K^2}{T}
\left(\int_0^1 |\eta'(s)|^2\,ds\right)h^{-1},
\end{align}
which proves \(\mathbb E\,\mathcal L_{\mathrm{int}}^{\mathrm{new}}(\mathbf{u}^\dagger)=O(h^{-1})\).

\medskip
\noindent
\textbf{Step 4: Residual after one pseudo-time stepping update.}
By definition,
\begin{align}
\mathbf{u}^{\dagger,+}
=
\mathbf{u}^\dagger-\tau\,\mathcal R_{\mathrm{int}}[\mathbf{u}^\dagger]
=
\alpha_h \mathbf{u}^\ast-\tau \alpha_h' \mathbf{u}^\ast
=
(\alpha_h-\tau\alpha_h')\mathbf{u}^\ast.
\end{align}
Applying the same product rule as above gives
\begin{align}
\mathcal R_{\mathrm{int}}[\mathbf{u}^{\dagger,+}]
&=
\partial_t\!\bigl((\alpha_h-\tau\alpha_h')\mathbf{u}^\ast\bigr)
+\mathcal D\!\bigl((\alpha_h-\tau\alpha_h')\mathbf{u}^\ast\bigr) \notag\\
&=
(\alpha_h'-\tau\alpha_h'')\mathbf{u}^\ast
+(\alpha_h-\tau\alpha_h')\mathcal R_{\mathrm{int}}[\mathbf{u}^\ast] \notag\\
&=
(\alpha_h'-\tau\alpha_h'')\mathbf{u}^\ast.
\end{align}
Hence
\begin{align}
\mathbb E\,\mathcal L_{\mathrm{int}}^{\mathrm{new}}(\mathbf{u}^{\dagger,+})
=
\frac{1}{T|\Omega|}
\int_{I_h}
|\alpha_h'(t)-\tau\alpha_h''(t)|^2\,G(t)\,dt
\le
\frac{K^2}{T}
\int_{I_h}
|\alpha_h'(t)-\tau\alpha_h''(t)|^2\,dt.
\end{align}
Expanding the square,
\begin{align}
\int_{I_h} |\alpha_h'-\tau\alpha_h''|^2\,dt
=
\int_{I_h} |\alpha_h'|^2\,dt
-
2\tau \int_{I_h} \alpha_h'\alpha_h''\,dt
+
\tau^2 \int_{I_h} |\alpha_h''|^2\,dt.
\end{align}
The cross term vanishes:
\begin{align}
\int_{I_h} \alpha_h'(t)\alpha_h''(t)\,dt
=
\frac12\Bigl[(\alpha_h'(t))^2\Bigr]_{t_0}^{t_0+h}
=0,
\end{align}
because \(\alpha_h'(t_0)=\alpha_h'(t_0+h)=0\). Therefore
\begin{align}
\int_{I_h} |\alpha_h'-\tau\alpha_h''|^2\,dt
=
\int_{I_h} |\alpha_h'|^2\,dt
+
\tau^2 \int_{I_h} |\alpha_h''|^2\,dt.
\end{align}
Also,
\begin{align}
\int_{I_h} |\alpha_h''(t)|^2\,dt
=
\int_{t_0}^{t_0+h}
h^{-4}\left|\eta''\!\left(\frac{t-t_0}{h}\right)\right|^2 dt
=
h^{-3}\int_0^1 |\eta''(s)|^2\,ds.
\end{align}
Hence
\begin{align}
\mathbb E\,\mathcal L_{\mathrm{int}}^{\mathrm{new}}(\mathbf{u}^{\dagger,+})
\le
\frac{K^2}{T}
\left(
\int_0^1 |\eta'(s)|^2\,ds\cdot h^{-1}
+
\tau^2 \int_0^1 |\eta''(s)|^2\,ds\cdot h^{-3}
\right),
\end{align}
which implies
\begin{align}
\mathbb E\,\mathcal L_{\mathrm{int}}^{\mathrm{new}}(\mathbf{u}^{\dagger,+})
=
O\!\left(h^{-1}+\tau^2 h^{-3}\right).
\end{align}

\medskip
\noindent
\textbf{Step 5: Sharpness under the nondegeneracy assumption.}
Assume now that \(\mathbf{u}^\ast(t_0,\cdot)\not\equiv 0\), so that
\begin{align}
G(t_0)=\int_\Omega |\mathbf{u}^\ast(t_0,\mathbf x)|^2\,d\mathbf x>0.
\end{align}
Since \(G\) is continuous, there exist constants \(c_0>0\) and \(h_0>0\) such that
\begin{align}
G(t)\ge c_0
\qquad\text{for all } t\in[t_0,t_0+h],\quad 0<h\le h_0.
\end{align}
Using this lower bound,
\begin{align}
\mathbb E\,\mathcal L_{\mathrm{int}}^{\mathrm{new}}(\mathbf{u}^\dagger)
\ge
\frac{c_0}{T|\Omega|}
\int_{I_h} |\alpha_h'(t)|^2\,dt
=
\frac{c_0}{T|\Omega|}
\left(\int_0^1 |\eta'(s)|^2\,ds\right) h^{-1}.
\end{align}
Combined with the upper bound from Step 3, this yields
\begin{align}
\mathbb E\,\mathcal L_{\mathrm{int}}^{\mathrm{new}}(\mathbf{u}^\dagger)\asymp h^{-1}.
\end{align}
Similarly,
\begin{align}
\mathbb E\,\mathcal L_{\mathrm{int}}^{\mathrm{new}}(\mathbf{u}^{\dagger,+})
\ge
\frac{c_0}{T|\Omega|}
\int_{I_h} |\alpha_h'(t)-\tau\alpha_h''(t)|^2\,dt
=
\frac{c_0}{T|\Omega|}
\left[
\left(\int_0^1 |\eta'(s)|^2\,ds\right) h^{-1}
+
\tau^2
\left(\int_0^1 |\eta''(s)|^2\,ds\right) h^{-3}
\right],
\end{align}
where we used the identity from Step 4. Combining with the upper bound from Step 4, we conclude
\begin{align}
\mathbb E\,\mathcal L_{\mathrm{int}}^{\mathrm{new}}(\mathbf{u}^\dagger)\asymp h^{-1},
\qquad
\mathbb E\,\mathcal L_{\mathrm{int}}^{\mathrm{new}}(\mathbf{u}^{\dagger,+})
\asymp h^{-1}+\tau^2 h^{-3}.
\end{align}
This completes the proof.
\end{proof}

\section{Experimental Details}
\label{appendix:experiments}
In this section, we summarize the experimental setup used throughout the paper. We first describe the neural architectures and boundary-condition treatments considered in our study, followed by the optimizers and weighting strategies used during training. We then present the benchmark-specific hyperparameter settings and the numerical configurations used to generate the reference solutions for all PDE problems, and finally provide additional visualizations to complement the quantitative comparisons in the main text. 

\subsection{Architectures}
\label{appendix:arch}

This section presents the network architectures considered in our work.

\paragraph{PirateNet.} We also consider PirateNet \cite{wang2024piratenets}, which was proposed to improve the stability and efficiency of training deep PINNs. The architecture first maps the input coordinate \(\mathbf{x}\) into a higher-dimensional feature space using random Fourier features \cite{tancik2020fourier}:
\begin{align*}
    \Phi(\mathbf{x})=
    \begin{bmatrix}
        \cos(\mathbf{B}\mathbf{x}) \\
        \sin(\mathbf{B}\mathbf{x})
    \end{bmatrix},
\end{align*}
where \(\mathbf{B}\in\mathbb{R}^{m\times d}\) has entries sampled i.i.d. from \(\mathcal{N}(0,s^2)\) with user-specified \(s>0\). This embedding helps mitigate spectral bias by enriching the frequency content of the corresponding neural tangent kernel, thereby improving the approximation of high-frequency and multiscale solution components \cite{wang2021eigenvector}.

The embedded input is then passed through two dense layers that serve as gates:
\begin{align*}
    \mathbf{U} &= \sigma(\mathbf{W}_1 \Phi(\mathbf{x}) + \mathbf{b}_1), \qquad
    \mathbf{V} = \sigma(\mathbf{W}_2 \Phi(\mathbf{x}) + \mathbf{b}_2),
\end{align*}
where \(\sigma\) is a pointwise activation function. This gating mechanism mirrors that used in the modified MLP.

Let \(\mathbf{x}^{(1)}=\Phi(\mathbf{x})\), and let \(\mathbf{x}^{(l)}\) denote the input to the \(l\)-th residual block for \(1\le l\le L\). Each block is defined as
\begin{align}
    \mathbf{f}^{(l)}  &= \sigma\big(\mathbf{W}_1^{(l)} \mathbf{x}^{(l)} + \mathbf{b}_1^{(l)}\big), \label{eq:step1} \\
    \mathbf{z}_1^{(l)} &= \mathbf{f}^{(l)} \odot \mathbf{U} + \bigl(1-\mathbf{f}^{(l)}\bigr)\odot \mathbf{V}, \label{eq:gate1} \\
    \mathbf{g}^{(l)}  &= \sigma\big(\mathbf{W}_2^{(l)} \mathbf{z}_1^{(l)} + \mathbf{b}_2^{(l)}\big), \\
    \mathbf{z}_2^{(l)} &= \mathbf{g}^{(l)} \odot \mathbf{U} + \bigl(1-\mathbf{g}^{(l)}\bigr)\odot \mathbf{V}, \label{eq:gate2} \\
    \mathbf{h}^{(l)}  &= \sigma\big(\mathbf{W}_3^{(l)} \mathbf{z}_2^{(l)} + \mathbf{b}_3^{(l)}\big), \\
    \mathbf{x}^{(l+1)} &= \alpha^{(l)} \mathbf{h}^{(l)} + \bigl(1-\alpha^{(l)}\bigr)\mathbf{x}^{(l)}.
    \label{eq:skip}
\end{align}
Thus, each block contains three dense layers, two gating operations, and an adaptive residual connection. The trainable coefficient \(\alpha^{(l)}\) controls the degree of nonlinearity in the block: \(\alpha^{(l)}=0\) corresponds to an identity map, whereas \(\alpha^{(l)}=1\) yields a fully nonlinear transformation.

For a PirateNet with \(L\) residual blocks, the final output is
\begin{align}
    \mathbf{u}_{\theta} = \mathbf{W}^{(L+1)} \mathbf{x}^{(L)}.
\end{align}
In our implementation, we initialize \(\alpha^{(l)}=0\), so the network starts as a shallow model whose output is a linear combination of first-layer embeddings. This initialization alleviates optimization difficulties in deep PINN models, while allowing the effective depth to increase adaptively during training as the \(\alpha^{(l)}\) values evolve.

\paragraph{Exact imposition of periodic boundary conditions.} To impose periodic boundary conditions exactly, we follow the construction of \cite{dong2021method}, which enforces periodicity as a hard constraint and can improve both optimization and accuracy. Consider a one-dimensional periodic function with period \(P\), satisfying
\begin{align}
    u^{(l)}(a)=u^{(l)}(a+P), \qquad l=0,1,2,\dots.
    \label{eq:periodic_constraint}
\end{align}
We introduce the Fourier embedding
\begin{align}
    \mathbf{v}(x)=\bigl(\cos(\omega x),\,\sin(\omega x)\bigr),
    \label{eq:1D_Fourier}
\end{align}
where \(\omega=\frac{2\pi}{P}\). Any neural network of the form \(u_{\theta}(\mathbf{v}(x))\) then satisfies the periodic boundary condition by construction.

This idea extends directly to higher dimensions. For a two-dimensional domain with periods \(P_x\) and \(P_y\), the periodicity conditions are
\begin{align}
    \frac{\partial^l}{\partial x^l}u(a,y) &= \frac{\partial^l}{\partial x^l}u(a+P_x,y),
    \qquad y\in [b,b+P_y], \\
    \frac{\partial^l}{\partial y^l}u(x,b) &= \frac{\partial^l}{\partial y^l}u(x,b+P_y),
    \qquad x\in [a,a+P_x],
\end{align}
for \(l=0,1,2,\dots\). These constraints can be encoded through the embedding
\begin{align}
    \mathbf{v}(x,y)=
    \begin{bmatrix}
        \cos(\omega_x x), \sin(\omega_x x), \cos(\omega_y y), \sin(\omega_y y)
    \end{bmatrix},
\end{align}
where \(\omega_x=\frac{2\pi}{P_x}\) and \(\omega_y=\frac{2\pi}{P_y}\).

For time-dependent problems, we simply concatenate the time variable with the spatial embedding, for example
\(u_{\theta}([t,\mathbf{v}(x)])\) in one spatial dimension or
\(u_{\theta}([t,\mathbf{v}(x,y)])\) in two spatial dimensions.

\subsection{Optimizers}
\label{appendix:optim}

We consider the following optimizers in our experiments:
\begin{itemize}
    \item \textbf{Adam.} We use Adam \cite{kingma2014adam} with the standard hyperparameter setting \(\beta_1 = 0.9\) and \(\beta_2 = 0.999\). Owing to its robustness and computational efficiency, Adam has become a standard choice for training PINNs.
    
\item \textbf{SOAP \cite{vyas2024soap}.} SOAP is a recently proposed optimizer that integrates Shampoo-style second-order preconditioning with Adam-style adaptive updates in a rotated eigenbasis. In our experiments, we use \(\beta_1 = 0.9\) and \(\beta_2 = 0.999\), which provide the best overall performance across our benchmark problems. We update the covariance matrices every 2 iterations.
\end{itemize}

\paragraph{SOAP update.} For completeness, we briefly outline the SOAP update rule used during training. At iteration $k$, 
for a parameter matrix \(\mathbf{W}^k\) with gradient \(\mathbf{G}^k\), SOAP maintains Shampoo-style preconditioner statistics and periodically computes the corresponding left and right eigenspaces, denoted by \(\mathbf{Q}_L^k\) and \(\mathbf{Q}_R^k\), respectively. At iteration \(k\), the gradient is first projected into this rotated basis:
\[
\widetilde{\mathbf{G}}^k = (\mathbf{Q}_L^k)^\top \mathbf{G}^k \mathbf{Q}_R^k .
\]
SOAP then applies an Adam-style update in the rotated coordinates:
\[
\widetilde{\mathbf{m}}^k
=
\beta_1 \widetilde{\mathbf{m}}^{k-1}
+
(1-\beta_1)\widetilde{\mathbf{G}}^k,
\qquad
\widetilde{\mathbf{v}}^k
=
\beta_2 \widetilde{\mathbf{v}}^{k-1}
+
(1-\beta_2)\bigl(\widetilde{\mathbf{G}}^k \odot \widetilde{\mathbf{G}}^k\bigr),
\]
from which the preconditioned update is formed as
\[
\widetilde{\mathbf{U}}^k
=
\frac{\widetilde{\mathbf{m}}^k}{\sqrt{\widetilde{\mathbf{v}}^k}+\epsilon}.
\]
Finally, the update is mapped back to the original parameter space:
\[
\mathbf{U}^k
=
\mathbf{Q}_L^k \widetilde{\mathbf{U}}^k (\mathbf{Q}_R^k)^\top,
\qquad
\mathbf{W}^{k+1}
=
\mathbf{W}^k - \eta \mathbf{U}^k .
\]
In this sense, SOAP may be viewed as performing Adam in the eigenspace induced by Shampoo, while updating this basis only periodically according to a prescribed preconditioning frequency.

\subsection{Weighting schemes}
\label{appendix:weighitng}

\paragraph{Causal training.}
Recent work \cite{wang2024respecting} has shown that, when solving time-dependent PDEs, PINNs can violate temporal causality by reducing the residual at later times before accurately resolving the solution at earlier times. To alleviate this issue, we adopt a causality-aware training strategy that reweights the residual loss across time.

Specifically, we partition the temporal domain into \(M\) equal subintervals and denote by \(\mathcal{L}_{\text{int}}^i(\mathbf{\theta})\) the PDE residual loss over the \(i\)-th subinterval. The weighted residual loss is then defined as
\begin{align}
     \mathcal{L}_{\text{int}}
    =
    \frac{1}{M}\sum_{i=1}^{M} w_i  \mathcal{L}_{\text{int}}^i(\mathbf{\theta}).
\end{align}
The temporal weights are chosen as
\begin{align}
    \label{eq:temporal_update}
    w_i
    =
    \exp\left(
    -\epsilon \sum_{k=1}^{i-1}  \mathcal{L}_{\text{int}}^k(\mathbf{\theta})
    \right),
    \qquad i=2,3,\dots,M,
\end{align}
with \(w_1=1\). Substituting these weights into the residual loss yields
\begin{align}
     \mathcal{L}_{\text{int}}
    =
    \frac{1}{M}\sum_{i=1}^{M}
    \exp\left(
    -\epsilon \sum_{k=1}^{i-1}  \mathcal{L}_{\text{int}}^k(\mathbf{\theta})
    \right)
     \mathcal{L}_{\text{int}}^i(\mathbf{\theta}).
\end{align}
This construction assigns each time segment a weight that decays exponentially with the accumulated residual loss over all preceding segments. Consequently, the residual loss in the \(i\)-th segment is minimized only after the earlier residual losses \(\{ \mathcal{L}_{\text{int}}^k(\mathbf{\theta}) \}_{k=1}^{i-1}\) have been sufficiently reduced. In this way, the optimization encourages the PINN model to gradually learn the PDE solution following  the temporal ordering of the underlying dynamics.


\paragraph{Learning rate annealing.} 

Another major challenge in training PINNs is balancing the different loss components, which often have very different scales and can therefore produce highly unbalanced gradients. This imbalance may cause optimization to favor certain terms while neglecting others, ultimately leading to training failure.

To address this issue, we employ the self-adaptive learning rate annealing strategy proposed in \cite{wang2021understanding}, which automatically balances the weighted loss
\begin{align}
    \mathcal{L}(\mathbf{\theta}) =  \lambda_{ic} \mathcal{L}_{ic}(\mathbf{\theta}) + \lambda_{bc} \mathcal{L}_{bc}(\mathbf{\theta}) +  \lambda_\text{int}  \mathcal{L}_{\text{int}}, 
\end{align}
The global weights are dynamically computed to equalize the gradient norms of each loss component:
  \begin{align}
  \label{eq: lambda_ic_update}
     \hat{\lambda}_{ic} &=  \frac{  \|\nabla_{\theta}  \mathcal{L}_{ic}(\theta)\| +  \|\nabla_{\theta}  \mathcal{L}_{bc}(\theta)\| +  \|\nabla_{\theta}   \mathcal{L}_{\text{int}}(\theta)\|   } {\|\nabla_{\theta}  \mathcal{L}_{ic}(\theta)\|}, \\
     \label{eq: lambda_bc_update}
      \hat{\lambda}_{bc} &= \frac{  \|\nabla_{\theta}  \mathcal{L}_{ic}(\theta)\| +  \|\nabla_{\theta}  \mathcal{L}_{bc}(\theta)\| +  \|\nabla_{\theta}   \mathcal{L}_{\text{int}}(\theta)\|   } {\|\nabla_{\theta}  \mathcal{L}_{bc}(\theta)\|},  \\
      \label{eq: lambda_r_update}
       \hat{\lambda}_{\text{int}} &=  \frac{  \|\nabla_{\theta}  \mathcal{L}_{ic}(\theta)\| +  \|\nabla_{\theta}  \mathcal{L}_{bc}(\theta)\| +  \|\nabla_{\theta}   \mathcal{L}_{\text{int}}(\theta)\|   } {\|\nabla_{\theta}   \mathcal{L}_{\text{int}}(\theta)\|}, 
  \end{align}
where $\|\cdot\|$ denotes the $L^2$ norm.
Then we obtain
\begin{align}
  \| \hat{\lambda}_{ic} \nabla_\theta \mathcal{L}_{ic} (\theta) \| =   \| \hat{\lambda}_{bc} \nabla_\theta \mathcal{L}_{bc} (\theta) \| = \| \hat{\lambda}_{\text{int}} \nabla_\theta \mathcal{L}_{\text{int}} (\theta) \| = \| \nabla_\theta \mathcal{L}_{ic} (\theta) \| +  \| \nabla_\theta \mathcal{L}_{bc} (\theta) \|  +  \| \nabla_\theta  \mathcal{L}_{\text{int}} (\theta) \|. 
\end{align}
Thus, the weighted losses contribute gradients of comparable magnitude, which helps prevent the optimization from being dominated by any single term. In practice, the weights are updated using running averages of their previous values, which improves stability during stochastic gradient descent. Moreover, these updates are performed only at user-specified intervals, typically every 100--1000 iterations, so the additional computational cost is minimal.

Finally, we emphasize that when loss weighting is used together with pseudo-time stepping, the pseudo-time step should be updated \textbf{before} the loss weights are updated. This ordering is important because the pseudo-time-step update depends on the magnitude of the most recent residual-loss gradient. Reversing the order can noticeably degrade training performance.

\subsection{Hyper-parameters}
\label{appendix:training}
Table~\ref{tab:hyper-parameters} summarizes the hyperparameter settings used for all benchmark PDEs. It includes the backbone architecture, learning rate schedule, training setup, and weighting strategies used in our experiments. Most settings are shared across problems, with only a few benchmark-specific choices such as activation function, batch size, and causal-weighting parameters.

\begin{table}[h]
\renewcommand{\arraystretch}{1.2}
\centering
\caption{{\em Hyperparameter configurations for benchmark PDEs.} Hyperparameter settings used to reproduce our experimental results. The backbone architecture is PirateNet, where Depth indicates the number of adaptive residual blocks, and Width denotes the number of neurons per hidden layer. RFF represents Random Fourier Features.}
\label{tab:hyper-parameters}
\vspace{1mm}
\resizebox{\textwidth}{!}{
\begin{tabular}{l cccccccccc}
\toprule
\textbf{Parameter} & \textbf{Inviscid Burgers} & \textbf{AC} & \textbf{KdV} & \textbf{KS} & \textbf{GS} & \textbf{GL} & \textbf{LDC} & \textbf{BFS} & \textbf{KF} & \textbf{RT} \\
\midrule
\textbf{Architecture} & \multicolumn{10}{c}{PirateNet}  \\
\midrule
Depth & \multicolumn{10}{c}{3}  \\
Width & \multicolumn{10}{c}{256} \\
Activation & Tanh & Tanh & Tanh & Tanh & Swish & Swish & Tanh & Swish & Tanh & Swish \\
RFF scale & 2.0 & 2.0 & 2.0 & 2.0 & 2.0 & 2.0 & 10.0 & 2.0 & 2.0 & 2.0 \\
\midrule
\textbf{Learning rate schedule} & & & & & & & & & & \\
\midrule
Initial learning rate & \multicolumn{10}{c}{$10^{-3}$} \\
Decay rate & \multicolumn{10}{c}{0.9} \\
Decay steps & \multicolumn{10}{c}{$2 \times 10^3 $}  \\
Warmup steps & \multicolumn{10}{c}{$2 \times 10^{3}$} \\
\midrule
\textbf{Training} & & & & & & & & & & \\
\midrule
Iterations  & \multicolumn{10}{c}{$10^5$}  \\
Batch size & 4,096 & 4,096& 4,096 & 4,096& 8,192 & 8,192 & 4,096& 4,096& 8,192& 8,192\\
\midrule
\textbf{Weighting Scheme} & \multicolumn{10}{c}{Grad Norm} \\
\midrule
\textbf{Causal weighting} & & & & & & & & & & \\
\midrule
Tolerance & 1.0 & 1.0 & 1.0 & 1.0 & 1.0 & 1.0 & N / A & N / A & 1.0 & 1.0 \\
\# Chunks & 16 & 16 & 16 & 16 & 16 & 16 & N / A & N / A & 16 & 16 \\
\bottomrule
\end{tabular}
}
\end{table}

\subsection{Benchmarks}
\label{appendix:benchmarks}

Compared with the benchmark settings commonly used in the current PINN literature, where long-time dynamics are often handled through time-marching or temporal decomposition, our problem setup adopts slightly shorter final times for these physical systems. This choice is intentional, since our goal is not to rely on time-marching strategies. Although such strategies can improve long-time prediction accuracy, they also introduce substantial computational overhead and can obscure the intrinsic capability of the underlying method. Instead, we aim to assess how far in time PINNs, equipped with modern training techniques, can remain accurate without temporal decomposition. The parameter settings and numerical configurations used to generate the reference solutions for all benchmarks are summarized in Table~\ref{tab: data_gen}.

\begin{table}[h]
\renewcommand{\arraystretch}{1.2}
\centering
\caption{Parameter settings, numerical configurations, and corresponding software packages used to generate the reference solutions for the PDE benchmarks; see Chebfun~\cite{driscoll2014chebfun}, PyClaw~\cite{pyclaw-sisc}, SU2~\cite{Economon2016SU2}, and JAX-Fluids~\cite{Bezgin2023}.}
\label{tab: data_gen}
\vspace{1mm}
\resizebox{\textwidth}{!}{
\begin{tabular}{l|l c c}
\toprule
\textbf{PDE} & Parameter & Package & Resolution  \\
\midrule
\textbf{AC} & $\epsilon = 10^{-4}, a=5$ & Chebfun  & $200 \times 512$ \\
\textbf{Inviscid Burgers} & $-$ & PyClaw & $200 \times 512$\\
\textbf{KdV} & $\eta=1, \mu=0.022$ & Chebfun & $200 \times 512$\\
\textbf{KS} & $\alpha = 100/ 16, \beta=100 / 16^2, \gamma=100 / 16^4$ & Chebfun & $250 \times 512$\\
\textbf{GS} & $\epsilon_1=0.2, \epsilon_2=0.1, b_1=40, b_2=100, c_1=c_2=1{,}000$ & Chebfun & $100 \times 200 \times 200$\\
\textbf{GL} & $\epsilon=0.004, \mu=10, \gamma=10 + 15i$ & Chebfun & $100 \times 200 \times 200$\\
\textbf{LDC} & $\text{Re}{=}5{\times}10^3$ & JAX-Fluids   & $128 \times 128$\\
\textbf{BFS} & $\text{Re}{=}8{\times}10^2$ & SU2  & $485 \times 65$\\
\textbf{KF} & $\text{Re}{=}10^4$ & IncompressibleNavierStokes & $50 \times 512 \times 512$\\
\textbf{RT} & $\text{Ra}{=}10^6, \text{Pr}=0.71$ & IncompressibleNavierStokes & $40 \times 100 \times 200$\\
\bottomrule
\end{tabular}
}
\end{table}

\paragraph{Allen--Cahn equation.}
The Allen--Cahn equation is a prototypical reaction--diffusion model arising in phase-field descriptions of interface motion and phase separation. Because it involves both diffusive smoothing and nonlinear bistable dynamics, it often develops sharp transition layers and is therefore widely used as a benchmark for PINNs.

We consider the one-dimensional Allen--Cahn equation
\begin{align}
    u_t - \epsilon u_{xx} + \alpha \left(u^3 - u\right) = 0,
    \qquad (t,x)\in [0,T]\times \Omega,
\end{align}
where $\Omega=[-1,1]$. In the standard benchmark setting, we take
\begin{align}
    \epsilon = 10^{-4}, \qquad \alpha = 5,
\end{align}
with the initial condition
\begin{align}
    u(0,x)=x^2\cos(\pi x),
\end{align}
and periodic boundary conditions
\begin{align}
    u(t,-1)=u(t,1), \qquad u_x(t,-1)=u_x(t,1).
\end{align}
We aim to train PINNs to solve this problem accurately up to \(T=1\).

\paragraph{Korteweg--De Vries equation.}
The Korteweg--De Vries (KdV) equation is a prototypical nonlinear dispersive PDE arising in the study of shallow-water waves. It captures the interplay between nonlinear wave steepening and dispersive spreading, and is widely known for its solitary-wave behavior. As such, it provides a useful benchmark for evaluating the ability of PINNs to learn nonlinear wave dynamics.

We consider the one-dimensional KdV equation
\begin{align}
    u_t + \eta u u_x + \mu^2 u_{xxx} = 0,
    \qquad (t,x)\in [0,T]\times \Omega,
\end{align}
where \(\Omega=[-1,1]\), with initial condition
\begin{align}
    u(0,x)=\cos(\pi x),
\end{align}
and periodic boundary condition
\begin{align}
    u(t,-1)=u(t,1).
\end{align}
Following the classical setting, we take \(\eta=1\) and \(\mu=0.022\) \cite{zabusky1965interaction}. Under this evolution, the initial wave develops into a sequence of solitary-type waves. We aim to train PINNs to solve this problem accurately up to \(T=1\).

\paragraph{Inviscid Burgers' equation.}
The inviscid Burgers equation is one of the simplest nonlinear hyperbolic PDEs and serves as a classical prototype for nonlinear wave propagation and shock formation. Despite its simple form, the nonlinear convection term causes initially smooth solutions to steepen over time and eventually form an \emph{exact discontinuity} in finite time. This makes the problem substantially more challenging than the viscous Burgers equation more commonly studied in the PINN literature, where the viscosity term smooths the solution and prevents true shock discontinuities. 

In this work, we consider the one-dimensional inviscid Burgers equation
\begin{align}
    u_t + u u_x = 0,
    \qquad (t,x)\in [0,T]\times \Omega,
\end{align}
with \(\Omega=[-1,1]\). Following the standard benchmark setting, we impose the initial condition
\begin{align}
    u(0,x) = -\sin(\pi x),
\end{align}
together with periodic boundary conditions
\begin{align}
    u(t,-1)=u(t,1).
\end{align}
Under this evolution, the solution steepens and develops a shock, creating a challenging test case for PINN training. Our goal is to train PINNs to accurately solve this problem up to the final time \(T=1.5\).

\paragraph{Kuramoto-Sivashinsky equation.} 
The Kuramoto--Sivashinsky (KS) equation is a classical nonlinear PDE arising in pattern-forming dissipative systems such as flame fronts, thin film flows, and plasma instabilities. Its dynamics reflect a balance among nonlinear advection, destabilization, and higher-order dissipation, leading to complex multiscale spatiotemporal behavior even in one dimension. This makes it a challenging benchmark for PINNs.

The one-dimensional equation takes the form:
\begin{align*}
    &u_t+\alpha u u_x+\beta u_{x x}+\gamma u_{x x x x}=0, \quad t \in[0,T], x \in[0,2 \pi], \\
    & u(0, x)=u_0(x),
\end{align*}
where $u$ represents the height of a thin film or flame front. This equation arises in various physical contexts, including flame front propagation, thin film flows, and plasma instabilities.

In this example, we take  $T=0.35$, $\alpha=100 / 16, \beta=100 / 16^2, \gamma=100 / 16^4$ and $u_0(x)=\cos (x)(1+\sin (x))$.

\paragraph{Grey-Scott equation.}  The Gray--Scott equation is a classical reaction--diffusion system that models the interaction of two chemical species undergoing diffusion and nonlinear reactions. It is well known for producing rich pattern formation phenomena, including spots, stripes, and self-organized structures, and is therefore a standard benchmark for learning complex spatiotemporal dynamics.

The system is described by the following coupled PDEs:
\begin{align*}
    u_t &=\epsilon_1 \Delta u + b_1(1-u) - c_1 u v^2,  \quad t \in (0, 0.8)\,, \ (x, y) \in (-1, 1)^2\,, \\
    v_t &=\epsilon_2 \Delta v - b_2 v + c_2 u v^2\,, \quad t \in (0, 0.8)\,, \ (x, y) \in (-1, 1)^2\,,
\end{align*}
With periodic boundary conditions, the initial conditions are:
\begin{align*}
    &u_0(x, y) = 1 - \exp(-10 ((x + 0.05)^2 + (y + 0.02)^2))\,, \\
    &v_0(x, y) = 1 - \exp(-10 ((x - 0.05)^2 + (y - 0.02)^2))\,.
\end{align*}
where $u$ and $v$ represent activator and inhibitor concentrations respectively, $\varepsilon_1$ and $\varepsilon_2$ are diffusion coefficients, and $(b_1, b_2, c_1, c_2)$ control reaction kinetics. This system generates diverse spatial patterns including spots and stripes.

We set parameters $\epsilon_1=0.2$, $\epsilon_2=0.1$, $b_1=40$, $b_2=100$, and $c_1=c_2=1{,}000$, which generates characteristic pattern formations. Due to the similar behavior of $u$ and $v$, we report only the relative $L^2$ error of $u$ in Table \ref{tab:sota}.

The corresponding PDE residuals are defined as
\begin{align}
    \mathcal{R}_u[\mathbf{u}_\theta](t,x,y)
    &:=
    \partial_t u_\theta
    - \epsilon_1 \left(\partial_{xx}u_\theta + \partial_{yy}u_\theta\right)
    - b_1(1-u_\theta)
    + c_1 u_\theta v_\theta^2,\\
    \mathcal{R}_v[\mathbf{u}_\theta](t,x,y)
    &:=
    \partial_t v_\theta
    - \epsilon_2 \left(\partial_{xx}v_\theta + \partial_{yy}v_\theta\right)
    + b_2 v_\theta
    - c_2 u_\theta v_\theta^2 .
\end{align}
The PDE residual loss relaxed by pseudo-time stepping is given by
\begin{align}
    \mathcal{L}_{ru}^{\mathrm{pts}}(\theta^k;\theta^{k-1},X_{\mathrm{int}})
    &:=
    \frac{1}{N_{\mathrm{int}}}
    \sum_{i=1}^{N_{\mathrm{int}}}
    \left|
    \frac{
    u_{\theta^k}(t^i_{\mathrm{int}},x^i_{\mathrm{int}},y^i_{\mathrm{int}})
    -
    u_{\theta^{k-1}}(t^i_{\mathrm{int}},x^i_{\mathrm{int}},y^i_{\mathrm{int}})
    }{\tau_u}
    +
    \mathcal{R}_u[\mathbf{u}_{\theta^k}](t^i_{\mathrm{int}},x^i_{\mathrm{int}},y^i_{\mathrm{int}})
    \right|^2,\\
    \mathcal{L}_{rv}^{\mathrm{pts}}(\theta^k;\theta^{k-1},X_{\mathrm{int}})
    &:=
    \frac{1}{N_{\mathrm{int}}}
    \sum_{i=1}^{N_{\mathrm{int}}}
    \left|
    \frac{
    v_{\theta^k}(t^i_{\mathrm{int}},x^i_{\mathrm{int}},y^i_{\mathrm{int}})
    -
    v_{\theta^{k-1}}(t^i_{\mathrm{int}},x^i_{\mathrm{int}},y^i_{\mathrm{int}})
    }{\tau_v}
    +
    \mathcal{R}_v[\mathbf{u}_{\theta^k}](t^i_{\mathrm{int}},x^i_{\mathrm{int}},y^i_{\mathrm{int}})
    \right|^2.
\end{align}

\paragraph{Ginzburg-Landau equation.}  The Ginzburg--Landau equation is a classical model for pattern formation and weakly nonlinear dynamics near instability onset. It arises in many physical settings, including superconductivity, nonlinear waves, and oscillatory media, and serves as a standard benchmark for learning complex spatiotemporal dynamics.

The complex Ginzburg-Landau equation in 2D takes the form
\begin{align*}
    \frac{\partial A}{\partial t}= \epsilon \Delta A + \mu A - \gamma  A|A|^2\,, \quad t \in (0, 0.4)\,,
    \ (x, y) \in (-1, 1)^2\,,
\end{align*}
with periodic boundary conditions, an initial condition
\begin{align*}
    A_0(x, y) = (10y + 10 i  x) \exp\left(-0.01 (2500 x^2 + 2500 y^2)\right)\,,
\end{align*}
where  $A$ is the complex amplitude representing the envelope of oscillations, $\epsilon$ represents the diffusion coefficient, $\mu$ is the linear growth rate, and $\gamma$ controls the nonlinear saturation. For this example, we set $\epsilon=0.004$, $\mu = 10$ and $\gamma= 10 + 15i$.

By denoting $A = u + i v$, we can decompose the equation into real and imaginary components, resulting in the following system of PDEs,
\begin{align*}
     \frac{\partial u}{\partial t} &= \epsilon \Delta u + \mu (u - (u-1.5 v) (u^2 + v^2))\,, \\
      \frac{\partial v}{\partial t} &= \epsilon \Delta v + \mu (v - (v + 1.5 u) (u^2 + v^2))\,.
\end{align*}

Given the coupled dynamics of $u$ and $v$, we present the relative $L^2$ error of $u$ in Table \ref{tab:sota}. The PDE residual loss relaxed by pseudo-time stepping is defined analogously to that for the Gray--Scott equation.

\paragraph{Lid-driven Cavity.} We study the incompressible Navier-Stokes equations in non-dimensional form for a two-dimensional domain:
\begin{align*}
    \mathbf{u} \cdot \nabla \mathbf{u}+\nabla p-\frac{1}{R e} \Delta \mathbf{u}&=0\,, \quad  (x,y) \in (0,1)^2\,, \\
    \nabla \cdot \mathbf{u}&=0\,, \quad  (x,y) \in (0,1)^2\,,
\end{align*}
where $\mathbf{u} = (u,v)$ represents the steady-state velocity field, $p$ is the pressure field, and $Re$ is the Reynolds number which characterizes the ratio of inertial to viscous forces.  This system models the equilibrium state of the flow, which is driven by the top boundary moving at a constant velocity while the other walls are stationary, leading to the formation of characteristic vortical structures whose complexity increases with the Reynolds number.

To ensure continuity at the corner boundaries, we implement a smoothed top-lid boundary condition:
\begin{align}
& u(x, y)=1-\frac{\cosh \left(C_0(x-0.5)\right)}{\cosh \left(0.5 C_0\right)}\,, \quad v(x, y)=0\,,
\end{align}
where $x \in [0, 1], y=1, C_0 = 50$. For the other three walls, we enforce a no-slip boundary condition. Our goal is to obtain the velocity and pressure field corresponding to a Reynolds number of $5{,}000$.  The accuracy of our method is evaluated using the velocity magnitude $\sqrt{u^2 + v^2}$, with results presented in Table \ref{tab:sota}.

\paragraph{Backward facing step flow.}
Backward-facing step flow is a standard benchmark for separated internal flows and is widely used to study the formation of recirculation and reattachment caused by a sudden expansion in channel geometry. Despite its simple configuration, it exhibits nontrivial flow structures, making it a useful test for whether PINNs can capture localized gradients and separation behavior in incompressible viscous flows. 

We consider the steady two-dimensional laminar incompressible flow over a backward-facing step. Let $\mathbf{u}=(u,v)$ denote the velocity field and $p$ the pressure. The governing equations are the steady incompressible Navier--Stokes equations
\begin{align}
\nabla \cdot \mathbf{u} &= 0,\\
(\mathbf{u}\cdot\nabla)\mathbf{u} + \nabla p - \nu \Delta \mathbf{u} &= 0,
\end{align}
in the fluid domain $\Omega$ downstream of the step. The channel has downstream length $15H$, where $H=1$ is the channel height after expansion. The left boundary is divided into two parts: the upper half is the inlet, and the lower half is the vertical step face. At the inlet, we prescribe a fully developed parabolic profile in the positive $x$-direction,
\begin{align}
u(0,y) = 24\,y(0.5-y), \qquad v(0,y)=0, \qquad 0\le y \le 0.5.
\end{align}
The fluid density is fixed to $\rho=1$, and the viscosity is chosen as $\nu=0.0025$ for $\mathrm{Re}=400$ and $\nu=0.00125$ for $\mathrm{Re}=800$. No-slip boundary conditions are imposed on the top and bottom walls as well as on the step face,
\begin{align}
\mathbf{u} = 0 \qquad \text{on walls}.
\end{align}
At the outlet, a reference pressure condition is imposed,
\begin{align}
p = 0.
\end{align}
We report the relative  $L^2$ error of $u$ in Table \ref{tab:sota}.  The PDE residual loss relaxed by pseudo-time stepping is defined analogously to that for the lid-driven cavity flow.

\paragraph{Kolmogorov flow.} 
Kolmogorov flow is a classical benchmark for studying forced incompressible turbulence and nonlinear energy transfer across spatial scales. Owing to its simple geometry and well-controlled sinusoidal forcing, it is widely used to evaluate whether numerical methods can capture complex fluid dynamics and long-range spatiotemporal interactions.

We study the two-dimensional Kolmogorov flow governed by the incompressible Navier-Stokes equations:
\begin{align*}
\mathbf{u}_t + \mathbf{u} \cdot \nabla \mathbf{u} & =- \nabla p + \frac{1}{R e} \Delta \mathbf{u} + \mathbf{f}, \\
\nabla \cdot \mathbf{u} & =0,
\end{align*}
on the unit square domain $(x,y) \in [0, 1]^2$.  

Here $\mathbf{u} = (u,v)$ represents the time-varying velocity field, and $\mathbf{f}$ denotes the external forcing term that maintains the flow structure. The system evolves from a random initial state and develops characteristic flow patterns, where energy transfers between different spatial scales through nonlinear interactions and viscous dissipation.

For our study, the system is driven by a sinusoidal forcing $\mathbf{f} =(2 \sin(4 \pi y), 0)$. The numerical experiment initializes with a random initial condition and evolves until $T=0.8$. The model's performance is quantified by the relative $L^2$ error of vorticity (Table \ref{tab:sota}).  The PDE residual loss relaxed by pseudo-time stepping is defined analogously to that for the lid-driven cavity flow.

\paragraph{Rayleigh--Taylor instability.}
Rayleigh--Taylor instability is a classical buoyancy-driven flow instability that occurs when a heavier fluid lies above a lighter one. Small perturbations grow over time and develop into complex plume-like structures, making this problem a useful benchmark for testing whether PINNs can capture strongly nonlinear flow dynamics. 

We consider a coupled flow--temperature system in the rectangular domain \((x,y)\in[0,1]\times[0,2]\):
\begin{align}
    \mathbf{u}_t + \mathbf{u}\cdot\nabla \mathbf{u} &= -\nabla p + \sqrt{\frac{Pr}{Ra}}\,\Delta \mathbf{u} + T\mathbf{e}_y, \\
    \nabla \cdot \mathbf{u} &= 0, \\
    T_t + \nabla \cdot (\mathbf{u}T) &= \frac{1}{\sqrt{Pr\,Ra}}\, \Delta T.
\end{align}
Here \(T\) denotes the temperature field, which acts as a density proxy under the Boussinesq approximation, \(Pr\) is the Prandtl number, and \(Ra\) is the Rayleigh number. In our experiments, we set \(Pr=0.71\) and \(Ra=10^6\). Periodic boundary conditions are imposed in the horizontal direction for both \(\mathbf{u}\) and \(T\), while Dirichlet conditions \(\mathbf{u}=T=0\) are enforced on the top and bottom boundaries. We evaluate model accuracy using the relative \(L^2\) error of the temperature field, with results reported in Table~\ref{tab:sota}. 

We define the network output as
\[
\mathbf{u}_\theta(t,x,y)
=
\bigl(
u_\theta(t,x,y),\,
v_\theta(t,x,y),\,
p_\theta(t,x,y),\,
T_\theta(t,x,y)
\bigr).
\]
The corresponding PDE residuals are defined as
\begin{align}
    \mathcal{R}_u[\mathbf{u}_\theta](t,x,y)
    &:=
    \partial_t u_\theta
    + u_\theta \partial_x u_\theta
    + v_\theta \partial_y u_\theta
    + \partial_x p_\theta
    - \sqrt{\frac{Pr}{Ra}}
    \left(
    \partial_{xx}u_\theta+\partial_{yy}u_\theta
    \right),\\
    \mathcal{R}_v[\mathbf{u}_\theta](t,x,y)
    &:=
    \partial_t v_\theta
    + u_\theta \partial_x v_\theta
    + v_\theta \partial_y v_\theta
    + \partial_y p_\theta
    - \sqrt{\frac{Pr}{Ra}}
    \left(
    \partial_{xx}v_\theta+\partial_{yy}v_\theta
    \right)
    - T_\theta,\\
    \mathcal{R}_c[\mathbf{u}_\theta](t,x,y)
    &:=
    \partial_x u_\theta + \partial_y v_\theta,\\
    \mathcal{R}_T[\mathbf{u}_\theta](t,x,y)
    &:=
    \partial_t T_\theta
    + \partial_x(u_\theta T_\theta)
    + \partial_y(v_\theta T_\theta)
    - \frac{1}{\sqrt{Pr\,Ra}}\,\partial_{tt}T_\theta .
\end{align}

Let
\[
X_{\mathrm{int}}
=
\{(t^i_{\mathrm{int}},x^i_{\mathrm{int}},y^i_{\mathrm{int}})\}_{i=1}^{N_{\mathrm{int}}}
\subset (0,0.8)\times(0,1)\times(0,2)
\]
denote the set of interior collocation points. Then the PDE residual loss relaxed by pseudo-time stepping is given by
\begin{align}
    \mathcal{L}_{r}^{\mathrm{pts}}(\theta^k;\theta^{k-1},X_{\mathrm{int}})
    &:=
    \lambda_{ru}\mathcal{L}_{ru}^{\mathrm{pts}}
    +\lambda_{rv}\mathcal{L}_{rv}^{\mathrm{pts}}
    +\lambda_{rc}\mathcal{L}_{rc}^{\mathrm{pts}}
    +\lambda_{rT}\mathcal{L}_{rT}^{\mathrm{pts}},
\end{align}
where
\begin{align}
    \mathcal{L}_{ru}^{\mathrm{pts}}(\theta^k;\theta^{k-1},X_{\mathrm{int}})
    &:=
    \frac{1}{N_{\mathrm{int}}}
    \sum_{i=1}^{N_{\mathrm{int}}}
    \left|
    \frac{
    u_{\theta^k}(t^i_{\mathrm{int}},x^i_{\mathrm{int}},y^i_{\mathrm{int}})
    -
    u_{\theta^{k-1}}(t^i_{\mathrm{int}},x^i_{\mathrm{int}},y^i_{\mathrm{int}})
    }{\tau_u}
    +
    \mathcal{R}_u[\mathbf{u}_{\theta^k}](t^i_{\mathrm{int}},x^i_{\mathrm{int}},y^i_{\mathrm{int}})
    \right|^2,\\
    \mathcal{L}_{rv}^{\mathrm{pts}}(\theta^k;\theta^{k-1},X_{\mathrm{int}})
    &:=
    \frac{1}{N_{\mathrm{int}}}
    \sum_{i=1}^{N_{\mathrm{int}}}
    \left|
    \frac{
    v_{\theta^k}(t^i_{\mathrm{int}},x^i_{\mathrm{int}},y^i_{\mathrm{int}})
    -
    v_{\theta^{k-1}}(t^i_{\mathrm{int}},x^i_{\mathrm{int}},y^i_{\mathrm{int}})
    }{\tau_v}
    +
    \mathcal{R}_v[\mathbf{u}_{\theta^k}](t^i_{\mathrm{int}},x^i_{\mathrm{int}},y^i_{\mathrm{int}})
    \right|^2,\\
    \mathcal{L}_{rc}^{\mathrm{pts}}(\theta^k;\theta^{k-1},X_{\mathrm{int}})
    &:=
    \frac{1}{N_{\mathrm{int}}}
    \sum_{i=1}^{N_{\mathrm{int}}}
    \left|
    \frac{
    p_{\theta^k}(t^i_{\mathrm{int}},x^i_{\mathrm{int}},y^i_{\mathrm{int}})
    -
    p_{\theta^{k-1}}(t^i_{\mathrm{int}},x^i_{\mathrm{int}},y^i_{\mathrm{int}})
    }{\tau_p}
    +
    \mathcal{R}_c[\mathbf{u}_{\theta^k}](t^i_{\mathrm{int}},x^i_{\mathrm{int}},y^i_{\mathrm{int}})
    \right|^2,\\
    \mathcal{L}_{rT}^{\mathrm{pts}}(\theta^k;\theta^{k-1},X_{\mathrm{int}})
    &:=
    \frac{1}{N_{\mathrm{int}}}
    \sum_{i=1}^{N_{\mathrm{int}}}
    \left|
    \frac{
    T_{\theta^k}(t^i_{\mathrm{int}},x^i_{\mathrm{int}},y^i_{\mathrm{int}})
    -
    T_{\theta^{k-1}}(t^i_{\mathrm{int}},x^i_{\mathrm{int}},y^i_{\mathrm{int}})
    }{\tau_T}
    +
    \mathcal{R}_T[\mathbf{u}_{\theta^k}](t^i_{\mathrm{int}},x^i_{\mathrm{int}},y^i_{\mathrm{int}})
    \right|^2.
\end{align}

\clearpage
\subsection{Additional visualizations}

\begin{figure}[h]
    \centering
    \includegraphics[width=1.0\linewidth]{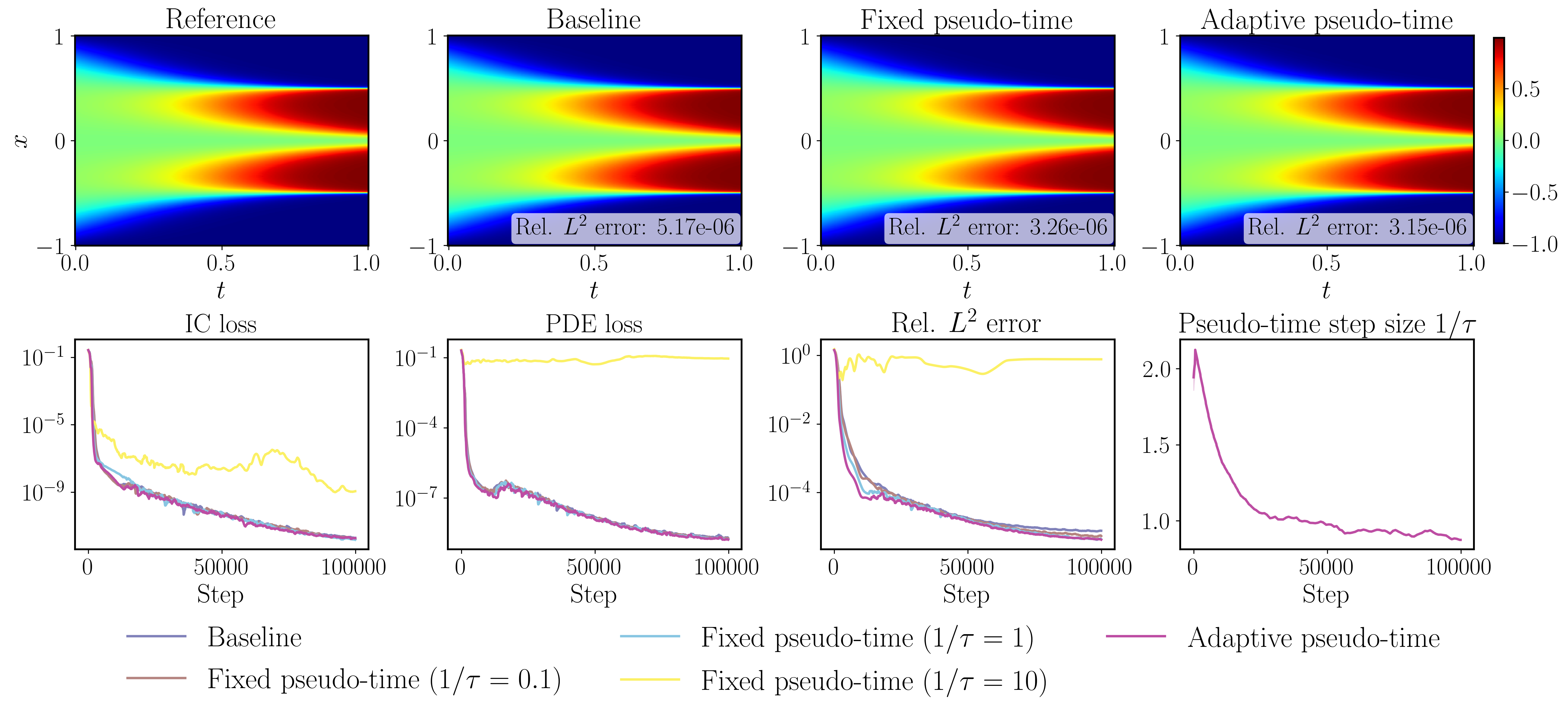}
\caption{{\em Allen--Cahn equation.} Comparison of the reference solution, the baseline PINN, and the baseline PINN with fixed and adaptive pseudo-time stepping. Top: predicted solutions. Bottom: training loss, relative \(L^2\) error, and pseudo-time step size histories.}
    \label{fig:allen_cahn_results}
\end{figure}

\begin{figure}[h]
    \centering
    \includegraphics[width=1.0\linewidth]{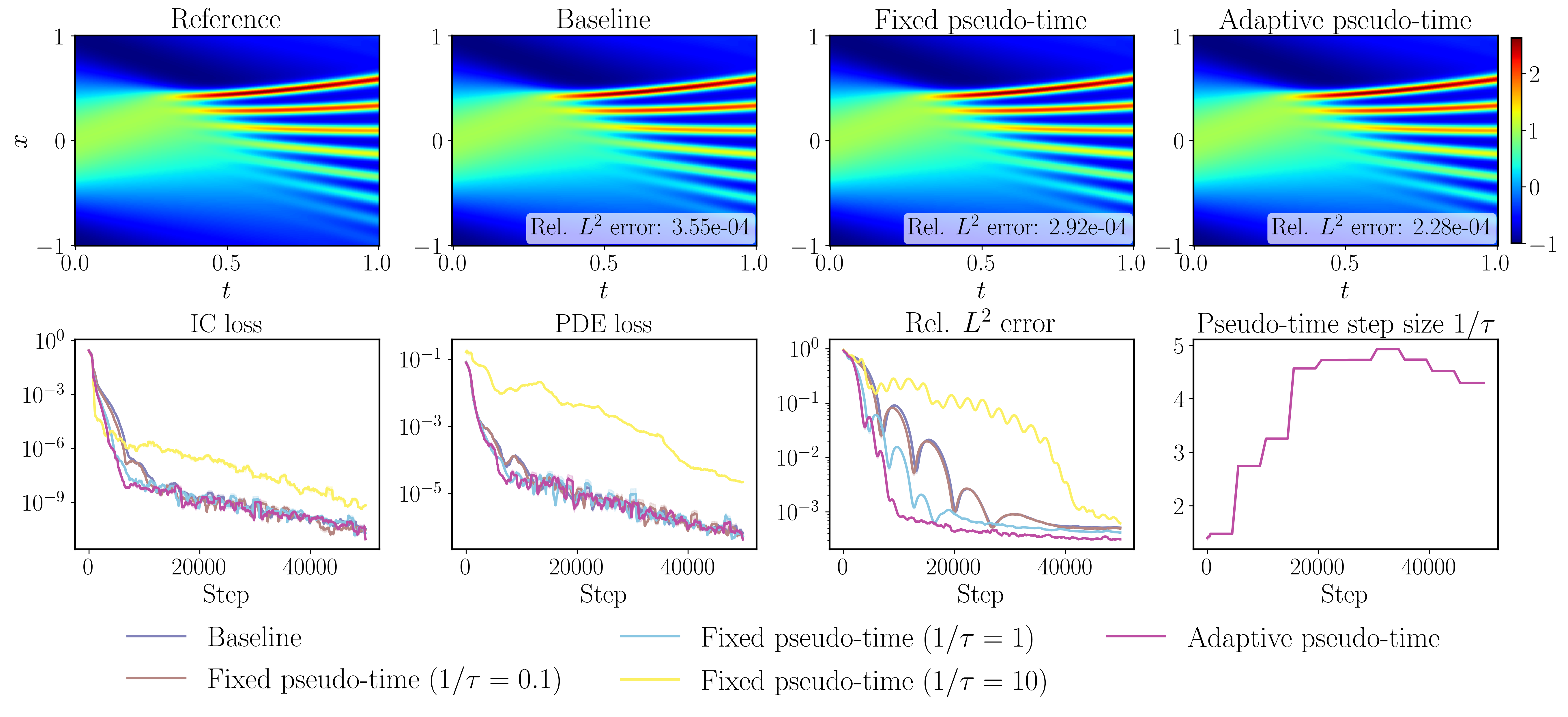}
\caption{{\em Korteweg–De Vries equation.} Comparison of the reference solution, the baseline PINN, and the baseline PINN with fixed and adaptive pseudo-time stepping. Top: predicted solutions. Bottom: training loss, relative \(L^2\) error, and pseudo-time step size histories.}
    \label{fig:kdv_results}
\end{figure}

\begin{figure}[h]
    \centering
    \includegraphics[width=1.0\linewidth]{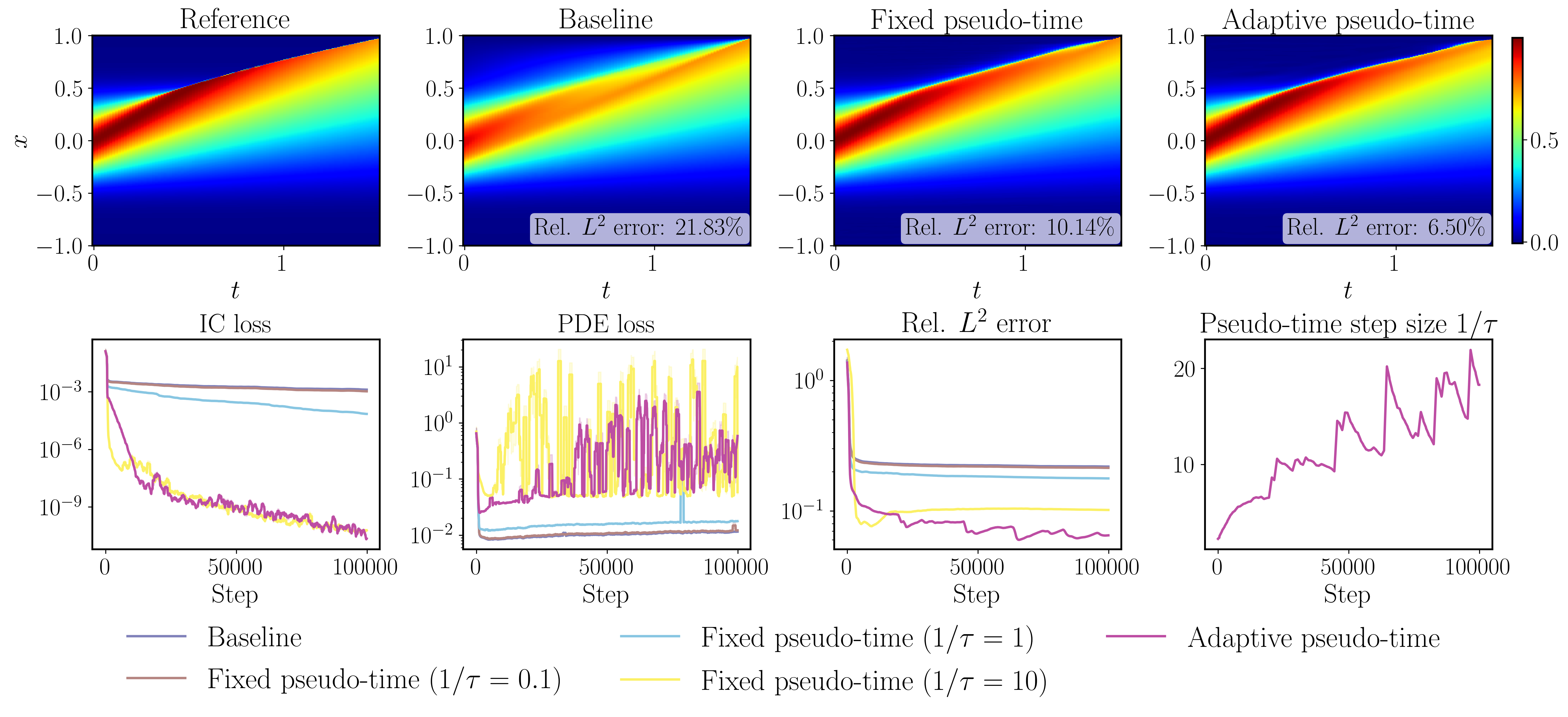}
\caption{{\em Inviscid Burgers' equation.} Comparison of the reference solution, the baseline PINN, and the baseline PINN with fixed and adaptive pseudo-time stepping. Top: predicted solutions. Bottom: training loss, relative \(L^2\) error, and pseudo-time step size histories.}
    \label{fig:inviscid_burgers_results}
\end{figure}

\begin{figure}[h]
    \centering
    \includegraphics[width=1.0\linewidth]{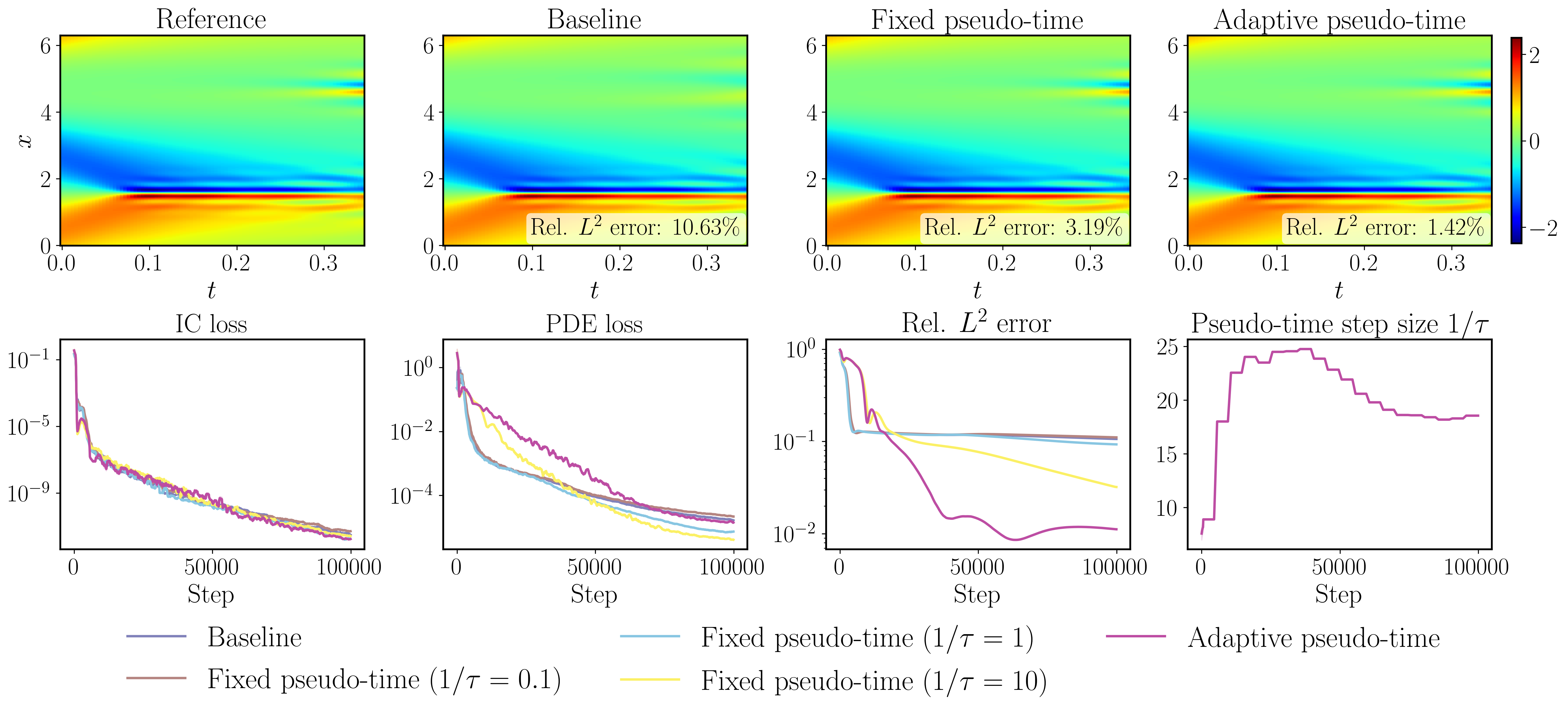}
\caption{{\em Kuramoto–Sivashinsky equation.} Comparison of the reference solution, the baseline PINN, and the baseline PINN with fixed and adaptive pseudo-time stepping. Top: predicted solutions. Bottom: training loss, relative \(L^2\) error, and pseudo-time step size histories.}
    \label{fig:ks_results}
\end{figure}

\begin{figure}[h]
    \centering

    \begin{subfigure}{1.0\linewidth}
        \centering
        \includegraphics[width=\linewidth]{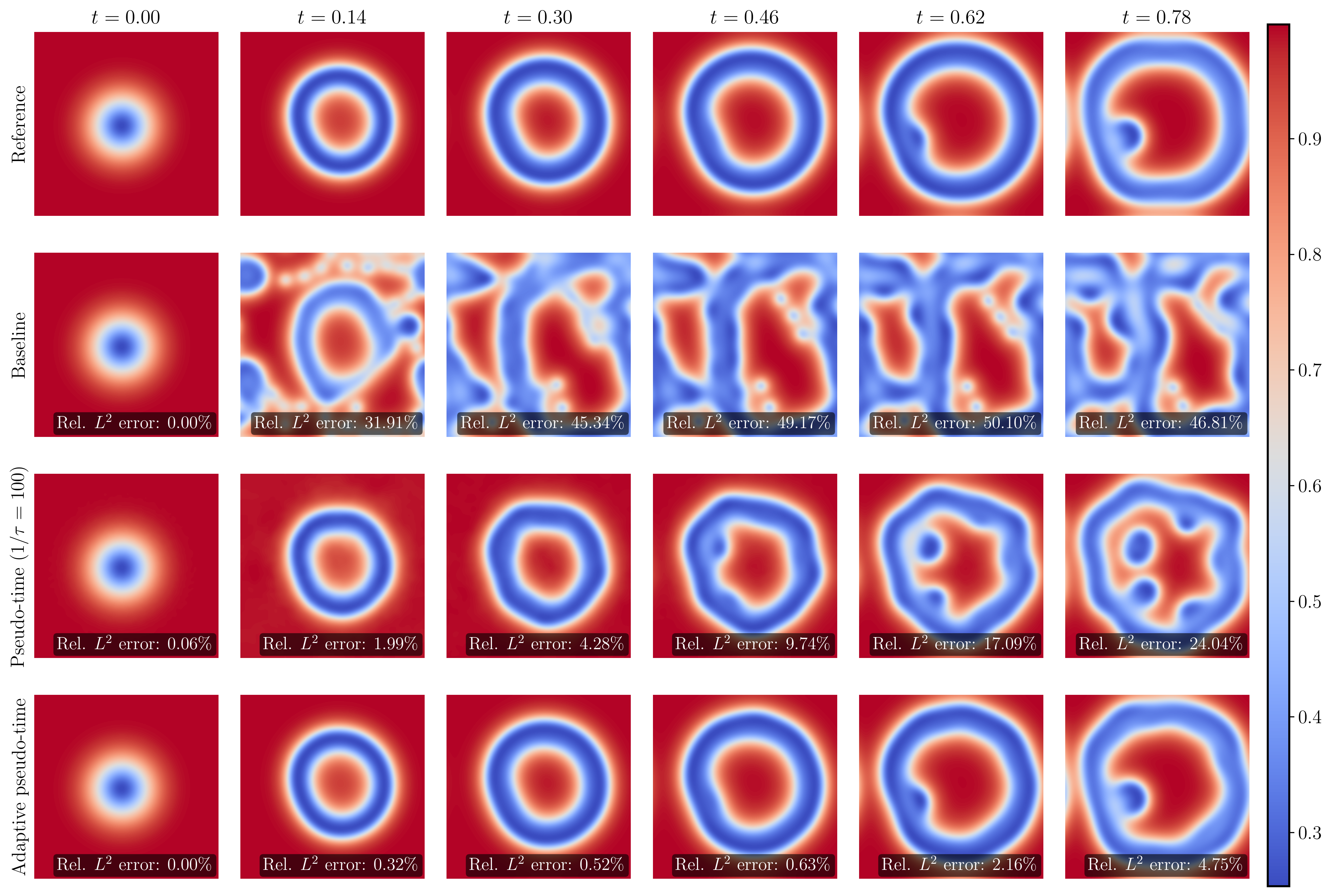}
    \end{subfigure}

    \vspace{0.1em}

    \begin{subfigure}{1.0\linewidth}
        \centering
        \includegraphics[width=\linewidth]{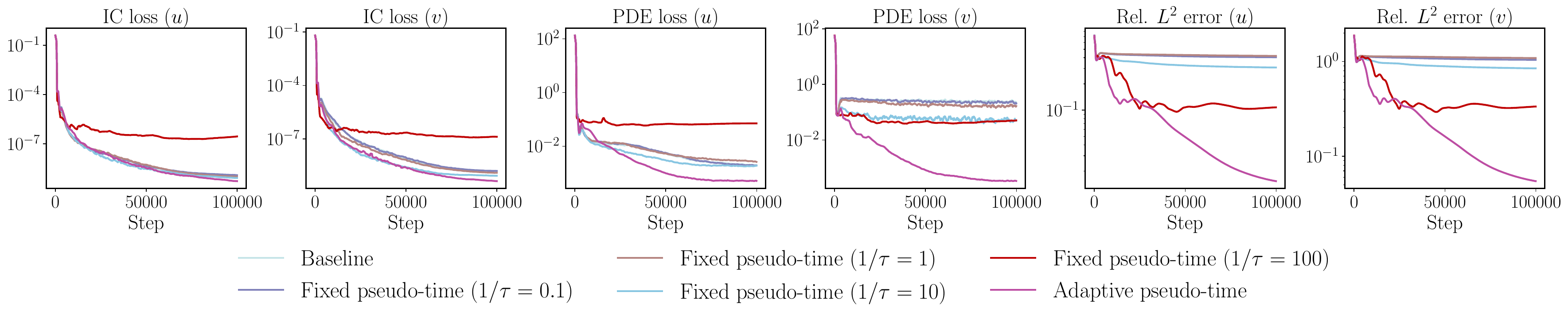}
    \end{subfigure}

        \caption{{\em Grey-Scott equations}. Top: Comparison of the reference solution, the baseline PINN, and the baseline PINN with fixed and adaptive pseudo-time stepping. Bottom: training loss and relative \(L^2\) error histories.}
    \label{fig:gs_results}
\end{figure}

\begin{figure}[h]
    \centering

    \begin{subfigure}[t]{1.0\linewidth}
        \centering
        \includegraphics[width=\linewidth]{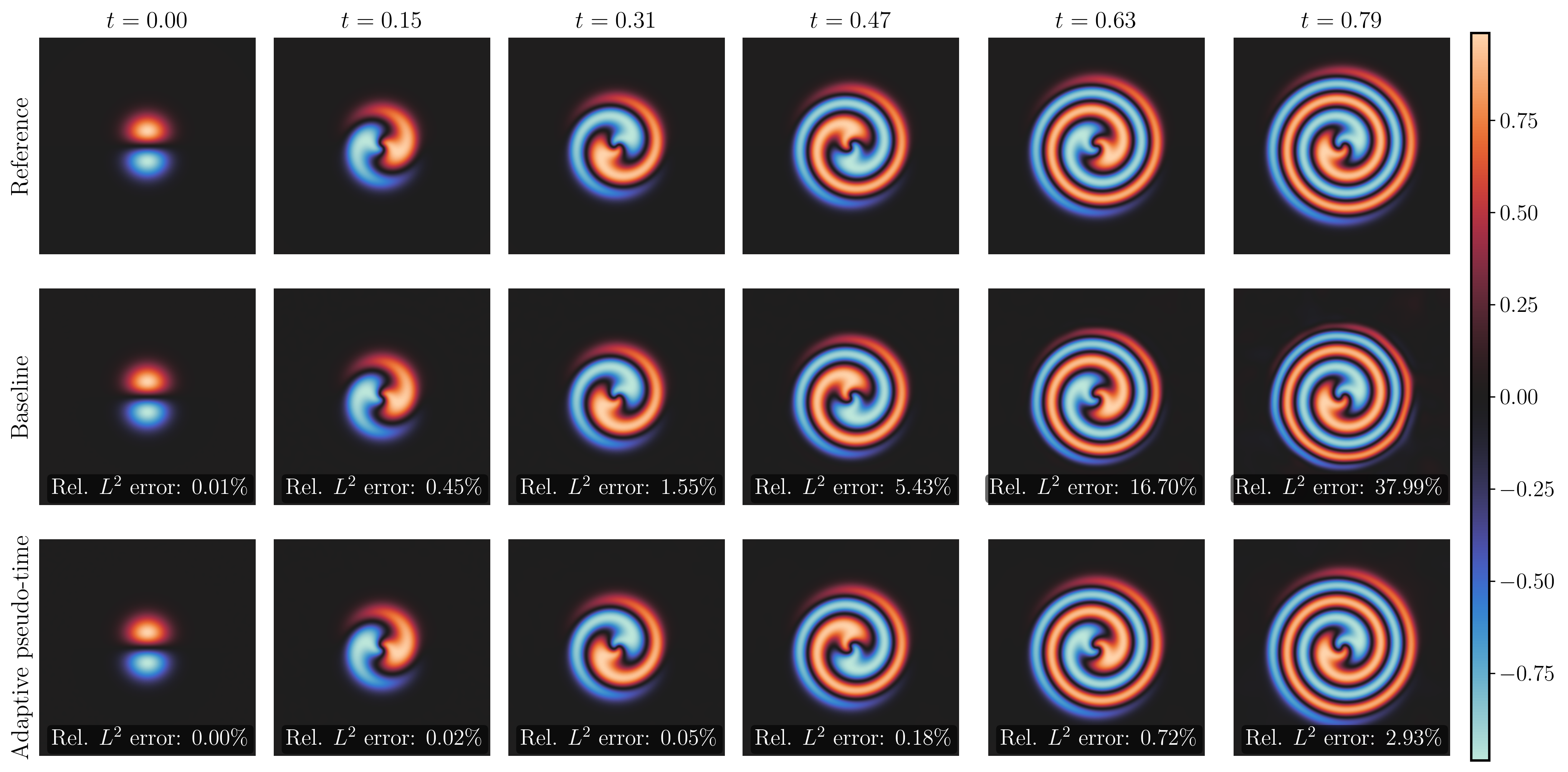}
    \end{subfigure}

    \vspace{0.1em}

    \begin{subfigure}[t]{1.0\linewidth}
        \centering
        \includegraphics[width=\linewidth]{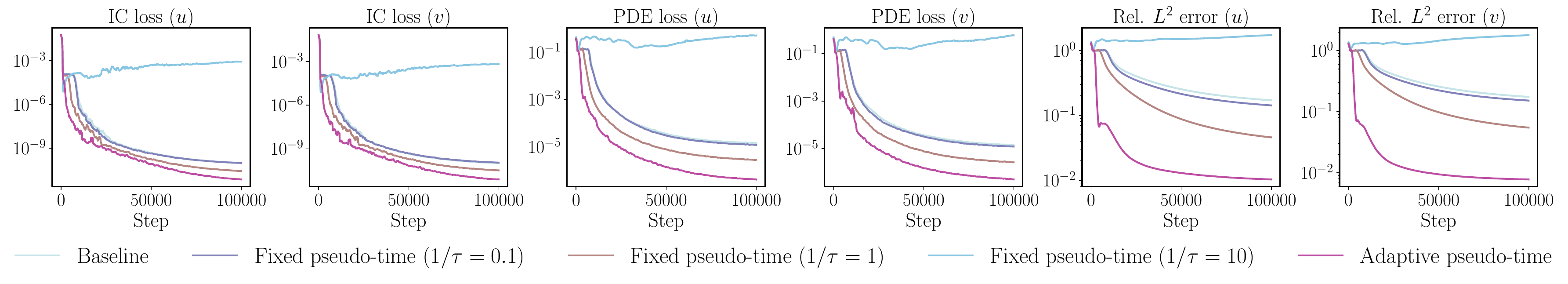}
    \end{subfigure}
    \caption{{\em Ginzburg–Landau equations}. Top: Comparison of the reference solution, the baseline PINN, and the baseline PINN with adaptive pseudo-time stepping. Bottom: training loss and relative \(L^2\) error histories.}
    \label{fig:gl_results}
\end{figure}

\begin{figure}[h]
    \centering
    \begin{subfigure}{\linewidth}
        \centering
        \includegraphics[width=\linewidth]{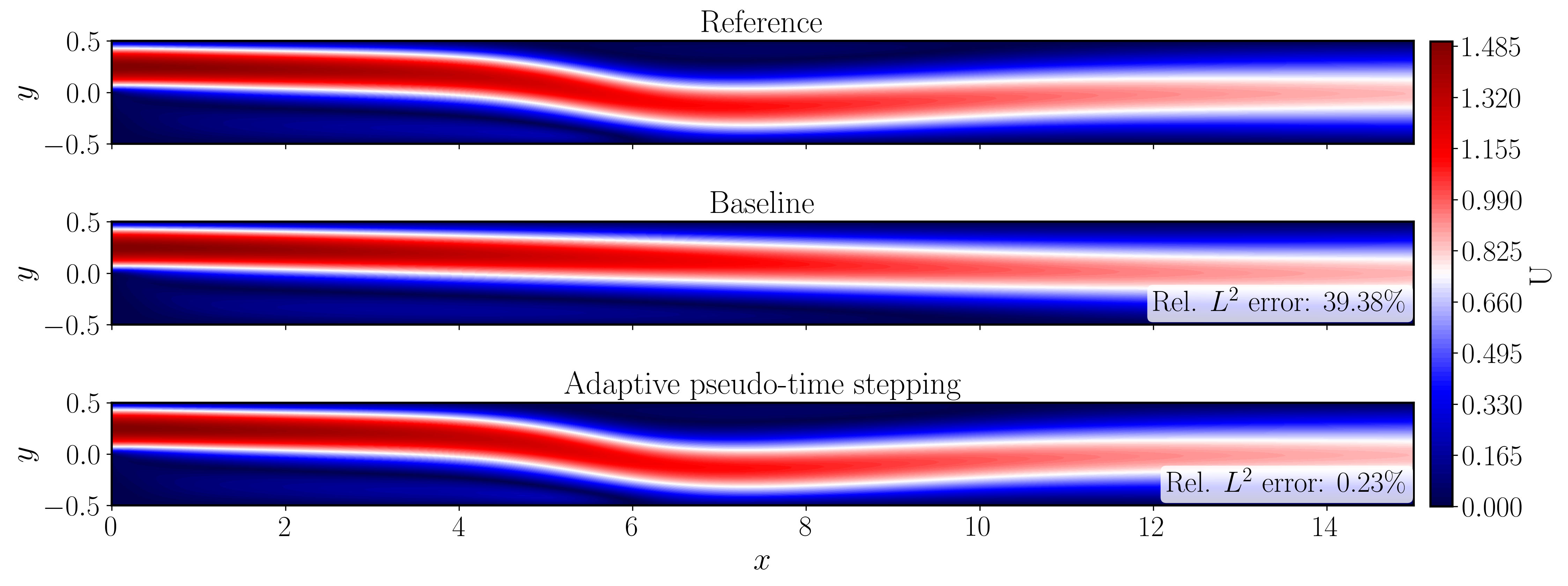}
    \end{subfigure}

    \begin{subfigure}{\linewidth}
        \centering
        \includegraphics[width=\linewidth]{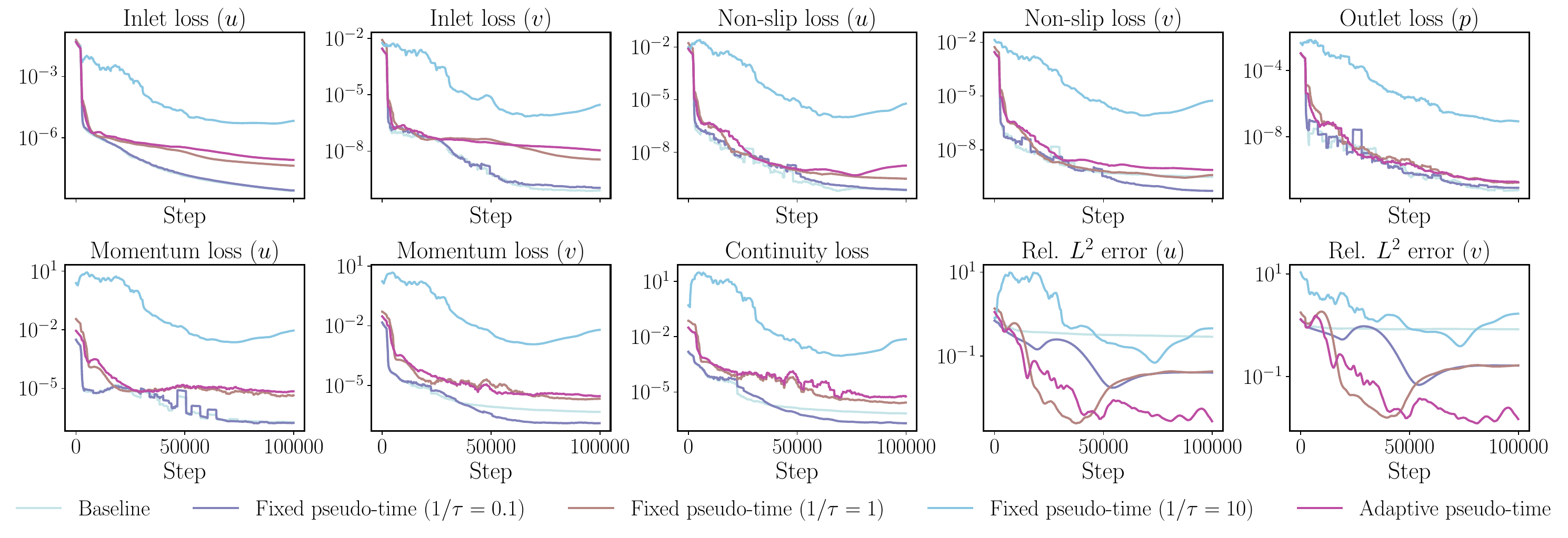}

    \end{subfigure}
\caption{{\em Backward-facing step flow}. Top: Comparison of the reference velocity field $U = \sqrt{u^2 + v^2}$, the baseline PINN, and the baseline PINN with adaptive pseudo-time stepping. Bottom: training loss and relative \(L^2\) error histories.}
    \label{fig:bfs}
\end{figure}

\begin{figure}[h]
    \centering

    \begin{subfigure}[t]{1.0\linewidth}
        \centering
        \includegraphics[width=\linewidth]{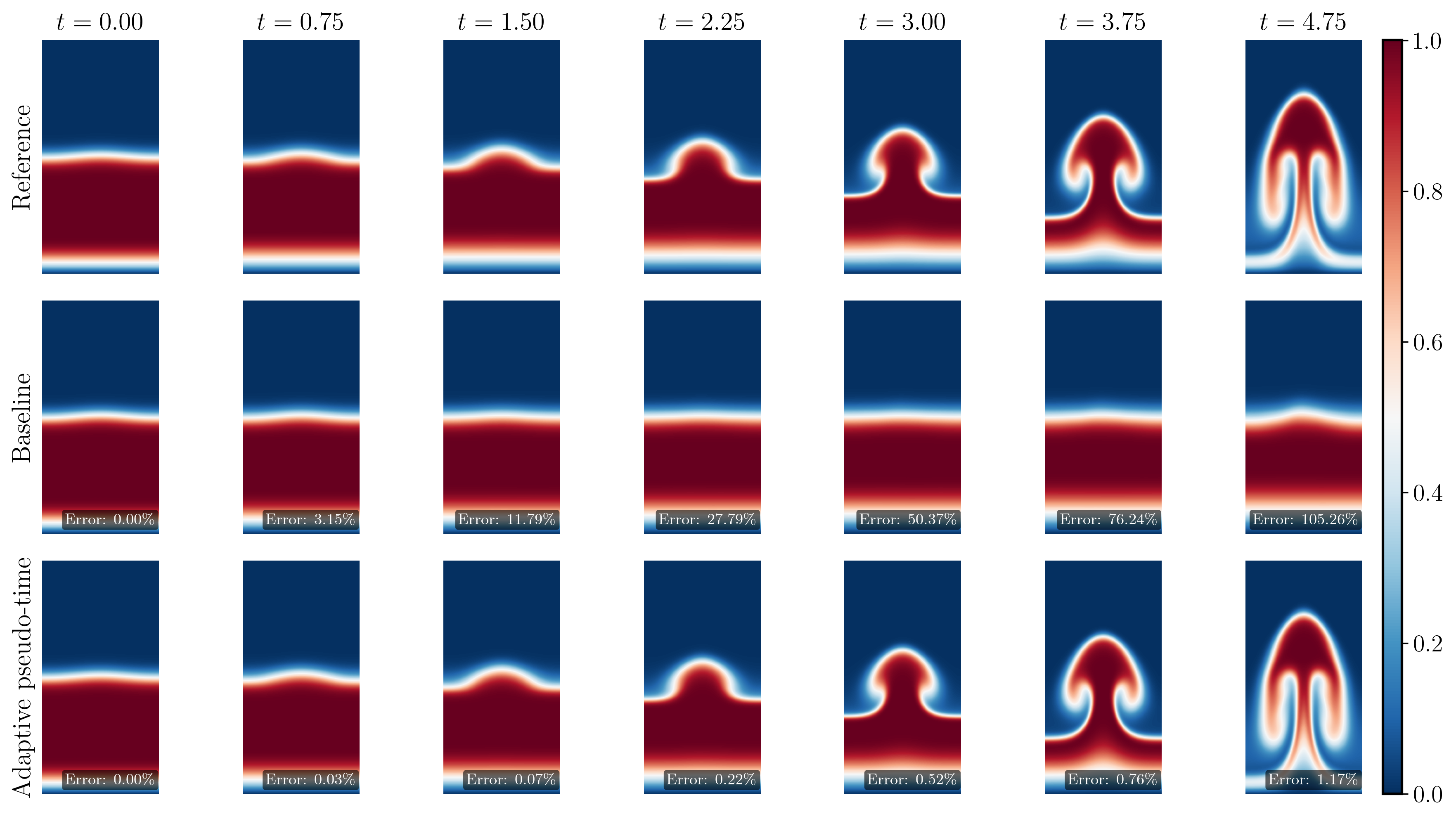}
    \end{subfigure}

    \vspace{0.1em}

    \begin{subfigure}[t]{1.0\linewidth}
        \centering
        \includegraphics[width=\linewidth]{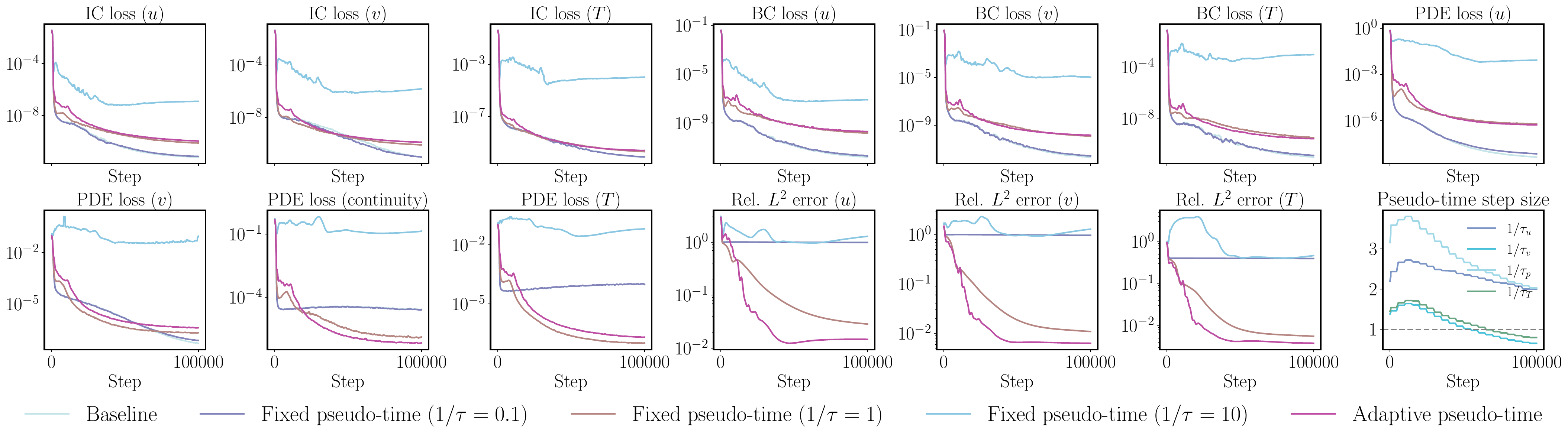}
    \end{subfigure}

    \caption{{\em Rayleigh-Taylor instability}. Top: Comparison of the reference temperature field, the baseline PINN, and the baseline PINN with adaptive pseudo-time stepping. Bottom: training loss and relative \(L^2\) error histories.}
    \label{fig:rt_results}
\end{figure}